\documentclass[11pt,a4paper,reqno]{amsart}
\pdfoutput=1
\usepackage{times}
\usepackage[margin=1.0in,tmargin=0.9in,bmargin=0.80in]{geometry}
\usepackage[dvipsnames,table]{xcolor}
\usepackage[english]{babel} 

\definecolor{bred}{rgb}{0.8,0,0}
\definecolor{Gray}{gray}{0.85}

\usepackage{amsmath,amssymb,graphicx,epsfig}
\usepackage{mathtools}
\usepackage{mathrsfs}
\usepackage{enumerate,amsthm,epstopdf,mathabx}
\usepackage[shortlabels]{enumitem}
\setlist[enumerate]{label={\upshape(\roman*)}}
\usepackage[flushleft]{threeparttable}
\usepackage{multirow}
\usepackage{longtable}
\usepackage[utf8]{inputenc} 
\usepackage[T1]{fontenc}    
\usepackage{bbm}
\usepackage{xargs}
\usepackage{footnote}
\makesavenoteenv{tabular}
\usepackage{latexsym}
\usepackage{wasysym}
\usepackage{float}
\usepackage{hyperref}       
\usepackage{url}            
\usepackage{booktabs}       
\usepackage{amsfonts}       
\usepackage{nicefrac}       
\usepackage{microtype}      
\usepackage{cleveref}
\usepackage{dsfont}
\usepackage{footmisc}
\usepackage{array}
\usepackage{subcaption}
\usepackage[linesnumbered,ruled,vlined]{algorithm2e}
\usepackage[section]{placeins}
\usepackage[labelfont=bf]{caption}
\usepackage[textsize=footnotesize]{todonotes}
\hypersetup{colorlinks,linkcolor={blue},citecolor={bred},urlcolor={blue}}
\usepackage[numbers,sort]{natbib}

\SetCommentSty{mycommfont}

\SetKwInput{KwInput}{Input}                
\SetKwInput{KwOutput}{Output}              

\DeclareMathOperator{\re}{Re}
\DeclareMathOperator{\im}{Im}
\DeclareMathOperator{\img}{Img}
\DeclareMathOperator{\sgn}{sgn}

\DeclareMathOperator{\id}{id}

\DeclareMathOperator{\supp}{supp}

\usepackage{stackengine}
\newcommand\xrowht[2][0]{\addstackgap[.5\dimexpr#2\relax]{\vphantom{#1}}}

\newtheorem{theorem}{Theorem}[section]
\newtheorem{proposition}[theorem]{Proposition}
\newtheorem{lemma}[theorem]{Lemma}

\newtheorem{remark}[theorem]{Remark}
\newtheorem{definition}[theorem]{Definition}
\newtheorem{example}[theorem]{Example}
\newtheorem{assumption}[theorem]{Assumption}

\newcolumntype{L}[1]{>{\raggedright\let\newline\\\arraybackslash\hspace{0pt}}m{#1}}
\newcolumntype{C}[1]{>{\centering\let\newline\\\arraybackslash\hspace{0pt}}m{#1}}
\newcolumntype{R}[1]{>{\raggedleft\let\newline\\\arraybackslash\hspace{0pt}}m{#1}}

\usepackage{scalerel,stackengine}
\stackMath
\newcommand\reallywidehat[1]{%
	\savestack{\tmpbox}{\stretchto{%
			\scaleto{%
				\scalerel*[\widthof{\ensuremath{#1}}]{\kern.1pt\mathchar"0362\kern.1pt}%
				{\rule{0ex}{\textheight}}
			}{\textheight}%
		}{2.4ex}}%
	\stackon[-6.9pt]{#1}{\tmpbox}%
}
\newcommand\reallywidetilde[1]{%
	\savestack{\tmpbox}{\stretchto{%
			\scaleto{%
				\scalerel*[\widthof{\ensuremath{#1}}]{\kern.1pt\mathchar"0366\kern.1pt}%
				{\rule{0ex}{\textheight}}
			}{\textheight}%
		}{2.4ex}}%
	\stackon[-6.9pt]{#1}{\tmpbox}%
}

\makeatletter
\newcommand{\labeltext}[3][]{%
	\@bsphack%
	\csname phantomsection\endcsname
	\def\tst{#1}%
	\def\labelmarkup{\emph}
	\def\refmarkup{}%
	\ifx\tst\empty\def\@currentlabel{\refmarkup{#2}}{\label{#3}}%
	\else\def\@currentlabel{\refmarkup{#1}}{\label{#3}}\fi%
	\@esphack%
	\labelmarkup{#2}
}
\makeatother

\allowdisplaybreaks

\begin{document}

\title[]{Universal approximation results for neural networks with non-polynomial activation function over non-compact domains}

\author[]{Ariel Neufeld}

\address{Nanyang Technological University, Division of Mathematical Sciences, 21 Nanyang Link, Singapore}
\email{ariel.neufeld@ntu.edu.sg}

\author[]{Philipp Schmocker}

\address{Nanyang Technological University, Division of Mathematical Sciences, 21 Nanyang Link, Singapore}
\email{philippt001@e.ntu.edu.sg}

\date{\today}
\thanks{Financial support by the MOE AcRF Tier 2 Grant   \textit{MOE-T2EP20222-0013} and the Nanyang Assistant Professorship Grant (NAP Grant) \emph{Machine Learning based Algorithms in Finance and Insurance} is gratefully acknowledged.}

\begin{abstract}
	This paper extends the universal approximation property of single-hidden-layer feedforward neural networks beyond compact domains, which is of particular interest for the approximation within weighted $C^k$-spaces and weighted Sobolev spaces over unbounded domains. More precisely, by assuming that the activation function is non-polynomial, we establish universal approximation results within function spaces defined over non-compact subsets of a Euclidean space, including $L^p$-spaces, weighted $C^k$-spaces, and weighted Sobolev spaces, where the latter two include the approximation of the (weak) derivatives. Moreover, we provide some dimension-independent rates for approximating a function with sufficiently regular and integrable Fourier transform by neural networks with non-polynomial activation function.
\end{abstract}

\maketitle

\vspace{-0.7cm}

\section{Introduction}
\label{SecIntro}

Inspired by the functionality of human brains, (artificial) neural networks have been discovered in the seminal work of McCulloch and Pitts~\cite{mcculloch43}, and can be described in mathematical terms as concatenation of affine and non-affine functions. Nowadays, neural networks are successfully applied in the fields of image classification \cite{krizhevsky12}, speech recognition \cite{hinton12}, and computer games \cite{silver16}, and provide as a supervised machine learning technique an algorithmic approach for the quest of artificial intelligence (see \cite{turing50,mitchell97}). 

One of the fundamental reasons for the success of neural networks is their universal approximation property. In so-called universal approximation theorems (UATs), this property establishes the denseness of the set of (single-hidden-layer) neural networks within a given function space. First, Cybenko~\cite{cybenko89}, Funahashi~\cite{funahashi89}, and Hornik et al.~\cite{hornik89} showed independently that neural networks with a continuous sigmoidal activation function are dense in the space of continuous functions over a compact subset of a Euclidean space. Subsequently, the universal approximation property was extended to broader classes of activation functions, e.g., by Hornik et al.~\cite{hornik91} to continuous, bounded, and non-constant activation functions, and by Mhaskar and Micchelli~\cite{mhaskar92}, Leshno et.~al~\cite{leshno93}, Chen and Chen~\cite{chen95}, and Pinkus~\cite{pinkus99} to non-polynomial activation functions. Moreover, Kidger and Lyons~\cite{kidger20} has established a UAT for deep narrow neural networks with non-affine activation function which is differentiable at at least one point with non-zero derivative. In addition, different UATs have been established within more general function spaces, e.g., Cybenko~\cite{cybenko89}, Mhaskar and Micchelli~\cite{mhaskar92}, and Leshno~\cite{leshno93} within $L^p$-spaces with compactly supported measure, and Hornik~\cite{hornik91} within $L^p$-spaces with finite measure. In addition, \cite{hsw90,hornik91} included the approximation of the derivatives, showing UATs within $C^k$-spaces over compact domains and within Sobolev spaces with compactly supported measures.

In this paper, we extend the universal approximation property of (single-hidden-layer) neural networks to function spaces over non-compact domains. This is of particular interest for the approximation in weighted $C^k$-spaces and weighted Sobolev spaces over unbounded domains, which, to the best of our knowledge, has not yet been addressed in the literature. To be as general as possible, we consider function spaces obtained as completion of the space of bounded and $k$-times continuously differentiable functions with bounded derivatives over a possibly non-compact domain with respect to a weighted norm. In this setting, we then follow Cybenko's approach~\cite{cybenko89} to obtain the universal approximation property, where we combine the classical Hahn-Banach separation argument with a Riesz representation theorem on weighted spaces \cite{doersek10} and Korevaar's distributional extension~\cite{korevaar65} of Wiener's Tauberian theorem~\cite{wiener32}. Our approach also extends the weighted UAT in \cite{cuchiero23} proven for neural networks between infinite dimensional spaces, by now also including the approximation of the derivatives. In the same way it also generalizes the weighted UAT in \cite{vannuland24}, which observes that bounded sigmoidal activation functions with distinct limits at positive and negative infinity are more powerful than those with identical limits. Moreover, in the context of dynamical systems, \cite{grigoryeva19,acciaio22} considered weighted approximation in time.

Furthermore, we prove dimension-independent rates to approximate a given function by a single-hidden-layer neural network in a (weighted) Sobolev space, which relates the approximation error to the network size. First, Barron~\cite{barron92,barron93} obtained the approximation rates for neural networks with sigmoidal activation within $L^2$-spaces, which was extended by Darken et al.~\cite{darken93,darken97} to $L^p$-spaces, in Mhaskar and Micchelli~\cite{mhaskar94,mhaskar95} to non-sigmoidal activation functions, and in Siegel and Xu~\cite{siegel20} to $W^{k,2}$-Sobolev spaces over bounded domains. For the particular case of shallow neural networks with ReLU activation, better approximation rates have been obtained, e.g., in \cite{siegel22,mao23,yang24}. On the other hand, \cite{papon22} derived approximation rates for deep neural networks with non-affine activation function which is differentiable at at least one point with non-zero derivative, coinciding with the best known rates for deep ReLU neural networks in \cite{zhang24} up to constants. For a more detailed review of the literature on approximation rates, we refer to \cite{pinkus99,kurkova12,bgkp17,siegel20}. In this paper, we extend the approximation rates for single-hidden-layer neural networks to weighted Sobolev spaces over unbounded domains, where we apply the (distributional dual) ridgelet transform of Sonoda and Murata~\cite{sonoda17} and use the concept of Rademacher averages.

\subsection{Outline}

In Section~\ref{SecNN}, we introduce (single-hidden-layer) neural networks and extend their universal approximation property beyond compact domains. In Section~\ref{SecApproxRates}, we provide some dimension-independent rates for neural networks in a weighted Sobolev space. Finally, Section~\ref{SecProofs} contains all proofs.

\subsection{Notation}
\label{SubsecNotation}

As usual, $\mathbb{N} := \lbrace 1, 2, 3, ... \rbrace$ and $\mathbb{N}_0 := \mathbb{N} \cup \lbrace 0 \rbrace$ denote the sets of natural numbers, whereas $\mathbb{Z}$ represents the set of integers. Moreover, $\mathbb{R}$ and $\mathbb{C}$ denote the sets of real and complex numbers, respectively, where $\mathbf{i} := \sqrt{-1} \in \mathbb{C}$ represents the imaginary unit. In addition, for any $r \in \mathbb{R}$, we define $\lceil r \rceil := \min\left\lbrace k \in \mathbb{Z}: k \geq r \right\rbrace$. Furthermore, for any $z \in \mathbb{C}$, we denote its real and imaginary part by $\re(z)$ and $\im(z)$, respectively, whereas its complex conjugate is defined as $\overline{z} := \re(z) - \im(z) \mathbf{i}$. Moreover, for any $m \in \mathbb{N}$, we denote by $\mathbb{R}^m$ (and $\mathbb{C}^m$) the $m$-dimensional (complex) Euclidean space, which is equipped with the Euclidean norm $\Vert u \Vert = \sqrt{\sum_{i=1}^m \vert u_i \vert^2}$.

In addition, for $U \subseteq \mathbb{R}^m$, we denote by $\mathcal{B}(U)$ the $\sigma$-algebra of Borel-measurable subsets of $U$. Moreover, for $U \in \mathcal{B}(\mathbb{R}^m)$, we denote by $\mathcal{L}(U)$ the $\sigma$-algebra of Lebesgue-measurable subsets of $U$, while $du: \mathcal{L}(U) \rightarrow [0,\infty]$ represents the Lebesgue measure on $U$. Then, a property is said to hold true almost everywhere (shortly a.e.)~if it holds everywhere true except on a set of Lebesgue measure zero.

Furthermore, for every fixed $m,d \in \mathbb{N}$, $k \in \mathbb{N}_0$, and $U \subseteq \mathbb{R}^m$ (open, if $k \geq 1$), we denote by $C^k(U;\mathbb{R}^d)$ the vector space of $k$-times continuously differentiable functions $f: U \rightarrow \mathbb{R}^d$ such that for every $\alpha \in \mathbb{N}^m_{0,k} := \left\lbrace \alpha = (\alpha_1,...,\alpha_m) \in \mathbb{N}_0^m: \vert\alpha\vert := \alpha_1 + ... + \alpha_m \leq k \right\rbrace$ the partial derivative $U \ni u \mapsto \partial_\alpha f(u) := \frac{\partial^{\vert\alpha\vert} f}{\partial u_1^{\alpha_1} \cdots \partial u_m^{\alpha_m}}(u) \in \mathbb{R}^d$ is continuous. If $m = 1$, we write $f^{(j)} := \frac{\partial^j f}{\partial u^j}: U \rightarrow \mathbb{R}^d$, $j = 0,...,k$. Note that for $k \geq 1$ the domain $U \subseteq \mathbb{R}^m$ must be open for the definition of partial derivatives.

Moreover, we denote by $C^k_b(U;\mathbb{R}^d)$ the vector space of bounded functions $f \in C^k(U;\mathbb{R}^d)$ with bounded partial derivatives $\partial_\alpha f: U \rightarrow \mathbb{R}^d$, for all $\alpha \in \mathbb{N}^m_{0,k}$. Then, the norm
\begin{equation*}
	\Vert f \Vert_{C^k_b(U;\mathbb{R}^d)} := \max_{\alpha \in \mathbb{N}^m_{0,k}} \sup_{u \in U} \Vert \partial_\alpha f(u) \Vert.
\end{equation*}
turns $(C^k_b(U;\mathbb{R}^d),\Vert \cdot \Vert_{C^k_b(U;\mathbb{R}^d)})$ into a Banach space. For $k = 0$ and $U \subset \mathbb{R}^m$ compact, we obtain the Banach space $(C^0(U;\mathbb{R}^d),\Vert \cdot \Vert_{C^0(U;\mathbb{R}^d)})$ of continuous functions equipped with the supremum norm.

In addition, for $\gamma \in [0,\infty)$, we denote by $C^k_{pol,\gamma}(U;\mathbb{R}^d)$ the vector space of functions $f \in C^k(U;\mathbb{R}^d)$ with
\begin{equation*}
	\Vert f \Vert_{C^k_{pol,\gamma}(U;\mathbb{R}^d)} := \max_{\alpha \in \mathbb{N}^m_{0,k}} \sup_{u \in U} \frac{\Vert \partial_\alpha f(u) \Vert}{(1 + \Vert u \Vert)^\gamma} < \infty.
\end{equation*}
Furthermore, we introduce the \emph{weighed $C^k$-space} $\overline{C^k_b(U;\mathbb{R}^d)}^\gamma$ defined as the closure of $C^k_b(U;\mathbb{R}^d)$ with respect to $\Vert \cdot \Vert_{C^k_{pol,\gamma}(U;\mathbb{R}^d)}$. Then, $(\overline{C^k_b(U;\mathbb{R}^d)}^\gamma,\Vert \cdot \Vert_{C^k_{pol,\gamma}(U;\mathbb{R}^d)})$ is by definition a Banach space. If $U \subseteq \mathbb{R}^m$ is bounded, then $\overline{C^k_b(U;\mathbb{R}^d)}^\gamma = C^k_b(U;\mathbb{R}^d)$. Otherwise, $f \in \overline{C^k_b(U;\mathbb{R}^d)}^\gamma$ if and only if $f \in C^k(U;\mathbb{R}^d)$ and $\lim_{r \rightarrow \infty} \max_{\alpha \in \mathbb{N}^m_{0,k}} \sup_{u \in U, \, \Vert u \Vert \geq r} \frac{\Vert \partial_\alpha f(u) \Vert}{(1+\Vert u \Vert)^\gamma} = 0$ (see Lemma~\ref{LemmaCkw}).

Moreover, we abbreviate $C^k(U) := C^k(U;\mathbb{R})$, $C^k_b(U) := C^k_b(U;\mathbb{R})$ and define the complex-valued function spaces $C^k(U;\mathbb{C}^d) \cong C^k(U;\mathbb{R}^{2d})$, $C^k_b(U;\mathbb{C}^d) \cong C^k_b(U;\mathbb{R}^{2d})$, etc.~by identifying $\mathbb{C}^d \cong \mathbb{R}^{2d}$.

Furthermore, we define the Fourier transform of any Lebesgue-integrable function $f: \mathbb{R}^m \rightarrow \mathbb{C}^d$ as
\begin{equation}
	\label{EqDefFourier}
	\mathbb{R}^m \ni \xi \quad \mapsto \quad \widehat{f}(\xi) := \int_{\mathbb{R}^m} e^{-i\xi^\top u} f(u) du \in \mathbb{C}^d,
\end{equation}
see \cite[p.~247]{folland92}. Then, by using \cite[Proposition~1.2.2]{hytoenen16}, it follows that
\begin{equation}
	\label{EqFourierBound}
	\sup_{\xi \in \mathbb{R}^m} \left\Vert \widehat{f}(\xi) \right\Vert = \sup_{\xi \in \mathbb{R}^m} \left\Vert \int_{\mathbb{R}^m} e^{-i\xi^\top u} f(u) du \right\Vert \leq \int_{\mathbb{R}^m} \Vert f(u) \Vert du = \Vert f \Vert_{L^1(\mathbb{R}^m,\mathcal{L}(\mathbb{R}^m),du;\mathbb{R}^d)}.
\end{equation}
In addition, we denote by $\mathcal{S}'(\mathbb{R}^m;\mathbb{C})$ the dual space of the Schwartz space $\mathcal{S}(\mathbb{R}^m;\mathbb{C})$ consisting of smooth functions $f: \mathbb{R}^m \rightarrow \mathbb{C}$ such that the seminorms $\max_{\alpha \in \mathbb{N}^m_{0,n}} \sup_{u \in \mathbb{R}^m} \left( 1 + \Vert u \Vert^2 \right)^n \left\vert \partial_\alpha f(u) \right\vert < \infty$, $n \in \mathbb{N}_0$, are finite. Then, the Fourier transform of any tempered distribution $T \in \mathcal{S}'(\mathbb{R}^m;\mathbb{C})$ is defined as $\widehat{T}(g) := T(\widehat{g})$, for $g \in \mathcal{S}(\mathbb{R}^m;\mathbb{C})$ (see \cite[Equation~9.28]{folland92}).

\section{Universal approximation of neural networks} 
\label{SecNN}

Inspired by the functionality of a human brain, neural networks were introduced in \cite{mcculloch43} and are nowadays applied as machine learning technique in various research areas (see \cite{mitchell97}). In mathematical terms, a neural network can be described as a concatenation of affine and non-affine functions.

\begin{definition}
	For $\rho \in C^0(\mathbb{R})$, a function $\varphi: \mathbb{R}^m \rightarrow \mathbb{R}^d$ is called a \emph{(single-hidden-layer feed-forward) neural network} if it is of the form
	\begin{equation}
		\label{EqDefNN}
		\mathbb{R}^m \ni u \quad \mapsto \quad \varphi(u) = \sum_{n=1}^N y_n \rho\left( a_n^\top u - b_n \right) \in \mathbb{R}^d
	\end{equation}
	with respect to some $N \in \mathbb{N}$ denoting the \emph{number of neurons}, where $a_1,...,a_N \in \mathbb{R}^m$, $b_1,...,b_N \in \mathbb{R}$, and $y_1,...,y_N \in \mathbb{R}^d$ represent the \emph{weight vectors}, \emph{biases}, and \emph{linear readouts}, respectively.
\end{definition}

\begin{definition}
	For $U \subseteq \mathbb{R}^m$ and $\rho \in C^0(\mathbb{R})$, we denote by $\mathcal{NN}^\rho_{U,d}$ the set of all neural networks of the form \eqref{EqDefNN} restricted to $U$ with corresponding activation function $\rho \in C^0(\mathbb{R})$.
\end{definition}

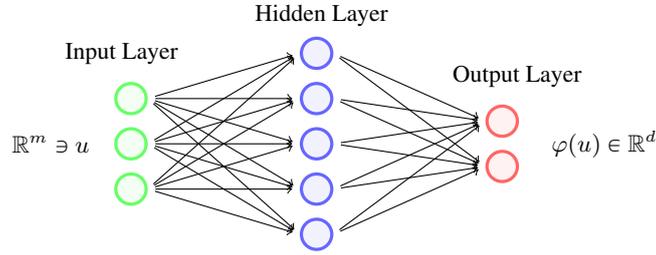
\begin{figure}[ht]
	\centering
	\begin{tikzpicture}[
		inputnode/.style={circle, draw=green!60, fill=green!5, very thick, minimum size=4mm},
		hiddennode/.style={circle, draw=blue!60, fill=blue!5, very thick, minimum size=4mm},
		outputnode/.style={circle, draw=red!60, fill=red!5, very thick, minimum size=4mm},
		node distance=6mm
		]
		\node[inputnode] (x1) {};
		\node[inputnode] (x2) [below of = x1] {};
		\node[inputnode] (x3) [below of = x2] {};
		\node[hiddennode] (y2) [right = 2cm of x1] {};
		\node[hiddennode] (y1) [above of = y2] {};
		\node[hiddennode] (y3) [below of = y2] {};
		\node[hiddennode] (y4) [below of = y3] {};
		\node[hiddennode] (y5) [below of = y4] {};
		\node[outputnode] (z1) [right = 2cm of y2, yshift=-3.0mm] {};
		\node[outputnode] (z2) [below of = z1] {};
		
		\draw[shorten >=0.1cm,shorten <=0.1cm, ->] (x1.east) -- (y1.west);	
		\draw[shorten >=0.1cm,shorten <=0.1cm, ->] (x1.east) -- (y2.west);
		\draw[shorten >=0.1cm,shorten <=0.1cm, ->] (x1.east) -- (y3.west);
		\draw[shorten >=0.1cm,shorten <=0.1cm, ->] (x1.east) -- (y4.west);
		\draw[shorten >=0.1cm,shorten <=0.1cm, ->] (x1.east) -- (y5.west);
		\draw[shorten >=0.1cm,shorten <=0.1cm, ->] (x2.east) -- (y1.west);	
		\draw[shorten >=0.1cm,shorten <=0.1cm, ->] (x2.east) -- (y2.west);
		\draw[shorten >=0.1cm,shorten <=0.1cm, ->] (x2.east) -- (y3.west);
		\draw[shorten >=0.1cm,shorten <=0.1cm, ->] (x2.east) -- (y4.west);
		\draw[shorten >=0.1cm,shorten <=0.1cm, ->] (x2.east) -- (y5.west);
		\draw[shorten >=0.1cm,shorten <=0.1cm, ->] (x3.east) -- (y1.west);	
		\draw[shorten >=0.1cm,shorten <=0.1cm, ->] (x3.east) -- (y2.west);
		\draw[shorten >=0.1cm,shorten <=0.1cm, ->] (x3.east) -- (y3.west);
		\draw[shorten >=0.1cm,shorten <=0.1cm, ->] (x3.east) -- (y4.west);
		\draw[shorten >=0.1cm,shorten <=0.1cm, ->] (x3.east) -- (y5.west);
		\draw[shorten >=0.1cm,shorten <=0.1cm, ->] (y1.east) -- (z1.west);
		\draw[shorten >=0.1cm,shorten <=0.1cm, ->] (y2.east) -- (z1.west);
		\draw[shorten >=0.1cm,shorten <=0.1cm, ->] (y3.east) -- (z1.west);
		\draw[shorten >=0.1cm,shorten <=0.1cm, ->] (y4.east) -- (z1.west);
		\draw[shorten >=0.1cm,shorten <=0.1cm, ->] (y5.east) -- (z1.west);
		\draw[shorten >=0.1cm,shorten <=0.1cm, ->] (y1.east) -- (z2.west);
		\draw[shorten >=0.1cm,shorten <=0.1cm, ->] (y2.east) -- (z2.west);
		\draw[shorten >=0.1cm,shorten <=0.1cm, ->] (y3.east) -- (z2.west);
		\draw[shorten >=0.1cm,shorten <=0.1cm, ->] (y4.east) -- (z2.west);
		\draw[shorten >=0.1cm,shorten <=0.1cm, ->] (y5.east) -- (z2.west);
		
		\draw[] (-1,0.6) node[anchor=west] {\footnotesize Input Layer};
		\draw[] (1.5,1.1) node[anchor=west] {\footnotesize Hidden Layer};
		\draw[] (4.1,0.3) node[anchor=west] {\footnotesize Output Layer};
		\draw[] (-1.7,-0.6) node[anchor=west] {\footnotesize $\mathbb{R}^m \ni u$};
		\draw[] (5.4,-0.6) node[anchor=west] {\footnotesize $\varphi(u) \in \mathbb{R}^d$};
	\end{tikzpicture}
	\caption{A neural network $\varphi: \mathbb{R}^m \rightarrow \mathbb{R}^d$ with $m = 3$, $d = 2$, and $N = 5$.}
	\label{FigNN}
\end{figure}

In this paper, we restrict ourselves to single-hidden-layer feed-forward neural networks of the form \eqref{EqDefNN} and simply refer to them as neural networks. 

\subsection{Universal approximation}

Neural networks admit the so-called universal approximation property, which establishes the denseness of the set of neural networks in a given function space with respect to some suitable topology. For example, every continuous function can be approximated arbitrarily well on a compact subset of a Euclidean space (see, e.g., \cite{cybenko89,hornik91,pinkus99} and the references therein). 

In order to generalize the approximation property of neural networks beyond the space of continuous functions on compacta, we now introduce the following type of function spaces. For this purpose, we fix the input dimension $m \in \mathbb{N}$ and the output dimension $d \in \mathbb{N}$ throughout the rest of this paper.

\begin{assumption}
	\label{AssApprox}
	For $k \in \mathbb{N}_0$, $U \subseteq \mathbb{R}^m$ (open, if $k \geq 1$), and $\gamma \in (0,\infty)$, let $(X,\Vert \cdot \Vert_X)$ be a Banach space consisting of functions $f: U \rightarrow \mathbb{R}^d$ such that the restriction map
	\begin{equation}
		\label{EqAssApprox1}
		(C^k_b(\mathbb{R}^m;\mathbb{R}^d),\Vert \cdot \Vert_{C^k_{pol,\gamma}(\mathbb{R}^m;\mathbb{R}^d)}) \ni f \quad \mapsto \quad f\vert_U \in (X,\Vert \cdot \Vert_X)
	\end{equation}
	is a continuous dense embedding.
\end{assumption}

Note that for $k \geq 1$ the domain $U \subseteq \mathbb{R}^m$ must be open to compute partial derivatives of a function.

\begin{remark}
	\label{RemApprox}
	The restriction map \eqref{EqAssApprox1} is a continuous dense embedding if and only if it is continuous and its image $\left\lbrace f\vert_U: f \in C^k_b(\mathbb{R}^m;\mathbb{R}^d) \right\rbrace$ is dense in $X$. Since $\overline{C^k_b(\mathbb{R}^m;\mathbb{R}^d)}^\gamma$ is defined as the closure of $C^k_b(\mathbb{R}^m;\mathbb{R}^d)$ with respect to $\Vert \cdot \Vert_{C^k_{pol,\gamma}(\mathbb{R}^m;\mathbb{R}^d)}$, the restriction map \eqref{EqAssApprox1} is a continuous dense embedding if and only if $(\overline{C^k_b(\mathbb{R}^m;\mathbb{R}^d)}^\gamma,\Vert \cdot \Vert_{C^k_{pol,\gamma}(\mathbb{R}^m;\mathbb{R}^d)}) \ni f \mapsto f\vert_U \in (X,\Vert \cdot \Vert_X)$ is a continuous dense embedding.
\end{remark}

The continuous dense embedding in \eqref{EqAssApprox1} ensures that the set of neural networks $\mathcal{NN}^\rho_{U,d} \subseteq X$ is well-defined in the function space $(X,\Vert \cdot \Vert_X)$, for all activation functions $\rho \in \overline{C^k_b(\mathbb{R})}^\gamma$.

\begin{lemma}
	\label{LemmaProp}
	Let $(X,\Vert \cdot \Vert_X)$ satisfy Assumption~\ref{AssApprox} with $k \in \mathbb{N}_0$, $U \subseteq \mathbb{R}^m$ (open, if $k \geq 1$), and $\gamma \in (0,\infty)$. Moreover, let $\rho \in \overline{C^k_b(\mathbb{R})}^\gamma$. Then, we have $\mathcal{NN}^\rho_{\mathbb{R}^m,d} \subseteq \overline{C^k_b(\mathbb{R}^m;\mathbb{R}^d)}^\gamma$ and $\mathcal{NN}^\rho_{U,d} \subseteq X$.
\end{lemma}

Let us give some examples of function spaces satisfying Assumption~\ref{AssApprox}, which includes in particular $C^k_b$-spaces, weighted $C^k$-spaces, $L^p$-spaces, and (weighted) $W^{k,p}$-Sobolev spaces.

\begin{example}
	\label{ExApprox}
	The following function spaces $(X,\Vert \cdot \Vert_X)$ are Banach spaces satisfying Assumption~\ref{AssApprox} with $k \in \mathbb{N}_0$, $U \subseteq \mathbb{R}^m$ (open, if $k \geq 1$), and $\gamma \in (0,\infty)$:
	\vspace{-0.3cm}
	\begin{center}
		\begin{tabular}{lll|l}
			& Function space & Banach space $X$ & Additional necessary conditions \\
			\hline\xrowht{5pt}
			
			\labeltext{(a)}{ExApproxCkb} & $C^k_b$-space\footnote{\label{FootnoteNotation}For the definition, we refer to Section~\ref{SubsecNotation}.} & $C^k_b(U;\mathbb{R}^d)$ & $U \subset \mathbb{R}^m$ is bounded. \\
			\hline\xrowht{15pt}
			
			\labeltext{(b)}{ExApproxCkw} & Weighted $C^k$-space\footref{FootnoteNotation} & $\overline{C^k_b(U;\mathbb{R}^d)}^\gamma$ & none \\
			\hline\xrowht{5pt}
			
			\labeltext{(c)}{ExApproxLp} & $L^p$-space\footnote{\label{FootnoteLp}For $U \subseteq \mathbb{R}^m$, $p \in [1,\infty)$, and Borel measure $\mu: \mathcal{B}(U) \rightarrow [0,\infty]$, we denote by $L^p(U,\mathcal{B}(U),\mu;\mathbb{R}^d)$ the $L^p$-space of (equivalence classes of) $\mathcal{B}(U)/\mathcal{B}(\mathbb{R}^d)$-measurable $f: U \rightarrow \mathbb{R}^d$ with $\Vert f \Vert_{L^p(U,\mathcal{B}(U),\mu;\mathbb{R}^d)} := \big( \int_U \Vert f(u) \Vert^p \mu(du) \big)^{1/p} < \infty$.}, $k = 0$ & $L^p(U,\mathcal{B}(U),\mu;\mathbb{R}^d)$ & $\int_U (1+\Vert u \Vert)^{\gamma p} \mu(du) < \infty$. \\
			\hline\xrowht{5pt}
			
			\multirow{2}{*}{\labeltext{(d)}{ExApproxWkp}} & \multirow{2}{*}{Sobolev space\footnote{For $k \in \mathbb{N}$, $U \subseteq \mathbb{R}^m$ open, and $p \in [1,\infty)$, we denote by $W^{k,p}(U,\mathcal{L}(U),du;\mathbb{R}^d)$ the Sobolev space of (equivalence classes of) $k$-times weakly differentiable $f: U \rightarrow \mathbb{R}^d$ with $\partial_\alpha f \in L^p(U,\mathcal{L}(U),du;\mathbb{R}^d)$, where $\Vert f \Vert_{W^{k,p}(U,\mathcal{L}(U),du;\mathbb{R}^d)} := \big( \sum_{\alpha \in \mathbb{N}^m_{0,k}} \int_U \Vert \partial_\alpha f(u) \Vert^p du \big)^{1/p} < \infty$ (see \cite[Chapter~3]{adams75}).}, $k \geq 1$} & \multirow{2}{*}{$W^{k,p}(U,\mathcal{L}(U),du;\mathbb{R}^d)$} & $U \subset \mathbb{R}^m$ is bounded and \\
			& & & $U \subset \mathbb{R}^m$ has the segment property\footnote{\label{FootnoteSegmentProp}An open subset $U \subseteq \mathbb{R}^m$ has the \emph{segment property} if for every $u \in \partial U := \overline{U} \setminus U$ there exists an open neighborhood $V \subseteq \mathbb{R}^m$ of $u \in \partial U$ and some $y \in \mathbb{R}^m \setminus \lbrace 0 \rbrace$ such that for every $z \in \overline{U} \cap V$ and $t \in (0,1)$ we have $z + t y \in U$ (see \cite[p.~54]{adams75}).}. \\
			\hline\xrowht{5pt}
			
			\multirow{4}{*}{\labeltext{(e)}{ExApproxWkpw}} & & \multirow{4}{*}{$W^{k,p}(U,\mathcal{L}(U),w;\mathbb{R}^d)$} & $U \subseteq \mathbb{R}^m$ has the segment property\footref{FootnoteSegmentProp}, \\
			& Weighted Sobolev space\footnote{\label{FootnoteWkpw}For $k \in \mathbb{N}$, $U \subseteq \mathbb{R}^m$ open, $p \in [1,\infty)$, and a weight $w: U \rightarrow [0,\infty)$, we denote by $W^{k,p}(U,\mathcal{L}(U),w;\mathbb{R}^d)$ the weighted Sobolev space of (equivalence classes of) $k$-times weakly differentiable $f: U \rightarrow \mathbb{R}^d$ with $\partial_\alpha f \in L^p(U,\mathcal{L}(U),w(u)du;\mathbb{R}^d)$, where $\Vert f \Vert_{W^{k,p}(U,\mathcal{L}(U),w;\mathbb{R}^d)} := \big( \sum_{\alpha \in \mathbb{N}^m_{0,k}} \int_U \Vert \partial_\alpha f(u) \Vert^p w(u) du \big)^{1/p} < \infty$ (see \cite[p.~5]{kufner80}). Hereby, the weight $w: U \rightarrow [0,\infty)$ is always assumed to be $\mathcal{L}(U)/\mathcal{B}(\mathbb{R})$-measurable and a.e.~strictly positive.}, & & $w: U \rightarrow [0,\infty)$ is bounded, \\
			& $k \geq 1$ & & $\inf_{u \in B} w(u) > 0$ for all bounded $B \subseteq U$, \\
			& & & and $\int_U (1+\Vert u \Vert)^{\gamma p} w(u) du < \infty$. \\
			\hline\xrowht{5pt}
		\end{tabular}
	\end{center}
	\vspace{-0.4cm}
	For $L^p$-spaces with the Lebesgue measure $\mu := du$ in \ref{ExApproxLp}, the domain $U \subseteq \mathbb{R}^m$ needs to be bounded.
\end{example}

Moreover, we assume that the activation function $\rho \in \overline{C^k_b(\mathbb{R})}^\gamma$ is non-polynomial, i.e.~algebraically not a polynomial over $\mathbb{R}$. Since $\rho \in \overline{C^k_b(\mathbb{R})}^\gamma$ induces the tempered distribution $\big( g \mapsto T_\rho(g) := \int_{\mathbb{R}} \rho(s) g(s) ds \big) \in \mathcal{S}'(\mathbb{R};\mathbb{C})$ (see \cite[Equation~9.26]{folland92}), this is equivalent to the condition that the Fourier transform $\widehat{T_\rho} \in \mathcal{S}'(\mathbb{R};\mathbb{C})$ is supported\footnote{The support of $T \in \mathcal{S}'(\mathbb{R}^m;\mathbb{C})$ is defined as the complement of the largest open set $U \subseteq \mathbb{R}^m$ on which $T \in \mathcal{S}'(\mathbb{R}^m;\mathbb{C})$ vanishes, i.e.~$T(g) = 0$ for all smooth functions $g: U \rightarrow \mathbb{C}$ with compact support contained in $U$.} at a non-zero point (see \cite[Examples~7.16]{rudin91} and Example~\ref{ExAdm}).

In order to obtain the universal approximation property of neural networks in function space satisfying Assumption~\ref{AssApprox}, we now combine the classical Hahn-Banach separation argument with a Riesz representation theorem on weighted spaces (see \cite[Theorem~2.4]{doersek10}) and follow Korevaar's distributional extension~\cite{korevaar65} of Wiener's Tauberian theorem~\cite{wiener32}. The proof can be found in Section~\ref{AppThmUAT}.

\begin{theorem}
	\label{ThmUAT}
	Let $(X,\Vert \cdot \Vert_X)$ be a function space satisfying Assumption~\ref{AssApprox} with $k \in \mathbb{N}_0$, $U \subseteq \mathbb{R}^m$ (open, if $k \geq 1$), and $\gamma \in (0,\infty)$. Moreover, let $\rho \in \overline{C^k_b(\mathbb{R})}^\gamma$ be non-polynomial. Then, $\mathcal{NN}^\rho_{U,d}$ is dense in $X$.
\end{theorem}

\begin{remark}
	Theorem~\ref{ThmUAT} extends the following universal approximation thereoms (UATs).
	\begin{enumerate}
		\item Within $C^0$-spaces on compacta: Cybenko~\cite{cybenko89} and Hornik et al.~\cite{hornik89} with sigmoidal activation; Hornik~\cite{hornik91} with continuous, bounded, and non-constant activation; Mhaskar and Micchelli~\cite{mhaskar92}, Leshno et al.~\cite{leshno93}, Chen and Chen~\cite{chen95}, and Pinkus~\cite{pinkus99} with non-polynomial activation.
		\item Within $L^p$-spaces with compactly supported and/or finite measure: Hornik et al.~\cite{hornik89} with sigmoidal activation; Hornik~\cite{hornik91} with bounded and non-constant activation; Mhaskar and Micchelli~\cite{mhaskar92} and Leshno et al.~\cite{leshno93} with non-polynomial activation.
		\item Within $C^k$-spaces over compact domains or Sobolev spaces with compactly supported measure: Hornik et al.~\cite{hsw90} with $k$-finite activation; Hornik~\cite{hornik91} with bounded and non-constant activation.
		\item Within function spaces over non-compact domains: Cuchiero et al.~\cite{cuchiero23} within weighted $C^0$-spaces without derivatives; van Nuland~\cite{vannuland24} within the space of continuous functions vanishing at infinity, which is comparable to our situation of $U \ni u \mapsto \frac{f(u)}{(1+\Vert u \Vert)^\gamma} \in \mathbb{R}^d$ vanishing at infinity.
	\end{enumerate}
\end{remark}

Hence, the main contribution of Theorem~\ref{ThmUAT} is the universal approximation property of neural networks within weighted $C^k$-spaces and weighted Sobolev spaces over unbounded domains.

\section{Approximation rates}
\label{SecApproxRates}

In this section, we provide some rates to approximate a given function by a (single-hidden-layer) neural network within the weighted Sobolev space $W^{k,p}(U,\mathcal{L}(U),w;\mathbb{R}^d)$, with $k \in \mathbb{N}_0$, $p \in [1,\infty)$, and $U \subseteq \mathbb{R}^m$ (open, if $k \geq 1$), where we set $W^{0,p}(U,\mathcal{L}(U),w;\mathbb{R}^d) := L^p(U,\mathcal{L}(U),w(u)du;\mathbb{R}^d)$.

To this end, we apply the reconstruction formula in \cite[Theorem~5.6]{sonoda17} to obtain an integral representation of the function to be approximated (see Proposition~\ref{PropIntRepr} below). For that, we consider pairs $(\psi,\rho) \in \mathcal{S}_0(\mathbb{R};\mathbb{C}) \times C^k_{pol,\gamma}(\mathbb{R})$ of a ridgelet function\footnote{$\mathcal{S}_0(\mathbb{R};\mathbb{C}) \subseteq \mathcal{S}(\mathbb{R};\mathbb{C})$ denotes the vector subspace of $\psi \in \mathcal{S}(\mathbb{R};\mathbb{C})$ with $\widehat{\psi}^{(j)}(0) = \int_{\mathbb{R}} u^j \psi(u) du = 0$ for all $j \in \mathbb{N}_0$.} $\psi \in \mathcal{S}_0(\mathbb{R};\mathbb{C})$ and an activation function $\rho \in C^k_{pol,\gamma}(\mathbb{R})$ satisfying the following, which is a special case of \cite[Definition~5.1]{sonoda17} (see also \cite{candes98}).

\begin{definition}
	\label{DefAdm}
	For $k \in \mathbb{N}_0$ and $\gamma \in [0,\infty)$, a pair $(\psi,\rho) \in \mathcal{S}_0(\mathbb{R};\mathbb{C}) \times C^k_{pol,\gamma}(\mathbb{R})$ is called \emph{$m$-admissible} if $\widehat{T_\rho} \in \mathcal{S}'(\mathbb{R};\mathbb{C})$ coincides\footnote{This means that $\widehat{T_\rho}(g) = \int_{\mathbb{R} \setminus \lbrace 0 \rbrace} f_{\widehat{T_\rho}}(\xi) g(\xi) d\xi$ for all smooth functions $g: \mathbb{R} \setminus \lbrace 0 \rbrace \rightarrow \mathbb{C}$ with compact support in $\mathbb{R} \setminus \lbrace 0 \rbrace$.} on $\mathbb{R} \setminus \lbrace 0 \rbrace$ with a function $f_{\widehat{T_\rho}} \in L^1_{loc}(\mathbb{R} \setminus \lbrace 0 \rbrace;\mathbb{C})$ such that
	\begin{equation}
		\label{EqDefAdm1}
		C^{(\psi,\rho)}_m := (2\pi)^{m-1} \int_{\mathbb{R} \setminus \lbrace 0 \rbrace} \frac{\overline{\widehat{\psi}(\xi)} f_{\widehat{T_\rho}}(\xi)}{\vert \xi \vert^m} d\xi \in \mathbb{C} \setminus \lbrace 0 \rbrace.
	\end{equation}
\end{definition}

\begin{remark}
	If $(\psi,\rho) \in \mathcal{S}_0(\mathbb{R};\mathbb{C}) \times C^k_{pol,\gamma}(\mathbb{R})$ is $m$-admissible, then $\rho \in C^k_{pol,\gamma}(\mathbb{R})$ has to be non-polynomial. Indeed, otherwise the support of $\widehat{T_\rho} \in \mathcal{S}'(\mathbb{R};\mathbb{C})$ is contained in $\left\lbrace 0 \right\rbrace \subset \mathbb{R}$ (see \cite[Examples~7.16]{rudin91}), which implies that \eqref{EqDefAdm1} vanishes for any choice of $\psi \in \mathcal{S}_0(\mathbb{R};\mathbb{C})$.
\end{remark}

Together with some suitable $\psi \in \mathcal{S}_0(\mathbb{R};\mathbb{C})$, most common activation functions satisfy Definition~\ref{DefAdm}.

\begin{example}
	\label{ExAdm}
	For $k \in \mathbb{N}_0$ and $\gamma \in [0,\infty)$, let $\psi \in \mathcal{S}_0(\mathbb{R};\mathbb{C})$ be such that $\widehat{\psi} \in C^\infty_c(\mathbb{R})$ is non-negative with $\supp(\widehat{\psi}) = [\zeta_1,\zeta_2]$ for some $0 < \zeta_1 < \zeta_2 < \infty$. Then, for every $m \in \mathbb{N}$ and every activation function $\rho \in C^k_{pol,\gamma}(\mathbb{R})$ listed in the table below the pair $(\psi,\rho) \in \mathcal{S}_0(\mathbb{R};\mathbb{C}) \times C^k_{pol,\gamma}(\mathbb{R})$ is $m$-admissible.
	\begin{flushleft}
		\begin{tabular}{lll|ll|l}
			& & $\rho \in C^k_{pol,\gamma}(\mathbb{R})$ & $k \in \mathbb{N}_0$ & $\gamma \in [0,\infty)$ & $f_{\widehat{T_\rho}} \in L^1_{loc}(\mathbb{R} \setminus \lbrace 0 \rbrace;\mathbb{C})$ \\
			\hline
			\labeltext{(a)}{ExAdm1} & Sigmoid function & $\rho(s) := \frac{1}{1+\exp(-s)}$ & $k \in \mathbb{N}_0$ & $\gamma \geq 0$ & $f_{\widehat{T_\rho}}(\xi) = \frac{-\mathbf{i} \pi}{\sinh(\pi \xi)}$ \\
			\hline
			\labeltext{(b)}{ExAdm2} & Tangens hyperbolicus & $\rho(s) := \tanh(s)$ & $k \in \mathbb{N}_0$ & $\gamma \geq 0$ & $f_{\widehat{T_\rho}}(\xi) = \frac{-\mathbf{i} \pi}{\sinh(\pi \xi/2)}$ \\
			\hline
			\labeltext{(c)}{ExAdm3} & Softplus function & $\rho(s) := \ln\left( 1+\exp(s) \right)$ & $k \in \mathbb{N}_0$ & $\gamma \geq 1$ & $f_{\widehat{T_\rho}}(\xi) = \frac{-\pi}{\xi \sinh(\pi \xi)}$ \\
			\hline
			\labeltext{(d)}{ExAdm4} & ReLU function & $\rho(s) := \max(s,0)$ & $k = 0$ & $\gamma \geq 1$ & $f_{\widehat{T_\rho}}(\xi) = -\frac{1}{\xi^2}$ \\
			\hline
		\end{tabular}
	\end{flushleft}
	Moreover, there exists $C_{\psi,\rho} > 0$ (independent of $m,d \in \mathbb{N}$) such that $\big\vert C^{(\psi,\rho)}_m \big\vert \geq C_{\psi,\rho} (2\pi/\zeta_2)^m$.
\end{example}

Next, we follow \cite{candes98,sonoda17} and define for every $\psi \in \mathcal{S}_0(\mathbb{R};\mathbb{C})$ the (multi-dimensional) \emph{ridgelet transform} of any function $g \in L^1(\mathbb{R}^m,\mathcal{L}(\mathbb{R}^m),du;\mathbb{R}^d)$ as
\begin{equation}
	\label{EqDefRidgelet}
	\mathbb{R}^m \times \mathbb{R} \ni (a,b) \quad \mapsto \quad (\mathfrak{R}_\psi g)(a,b) := \int_{\mathbb{R}^m} \psi\left( a^\top u - b \right) g(u) \Vert a \Vert du \in \mathbb{C}^d.
\end{equation}
Then, we can apply the reconstruction formula in \cite[Theorem~5.6]{sonoda17} componentwise to obtain an integral representation. The proof can be found in Section~\ref{AppIntRepr}.

\begin{proposition}
	\label{PropIntRepr}
	For $k \in \mathbb{N}_0$ and $\gamma \in [0,\infty)$, let $(\psi,\rho) \in \mathcal{S}_0(\mathbb{R};\mathbb{C}) \times C^k_{pol,\gamma}(\mathbb{R})$ be $m$-admissible and let $g \in L^1(\mathbb{R}^m,\mathcal{L}(\mathbb{R}^m),du;\mathbb{R}^d)$ with $\widehat{g} \in L^1(\mathbb{R}^m,\mathcal{L}(\mathbb{R}^m),du;\mathbb{C}^d)$. Then, for a.e. $u \in \mathbb{R}^m$, it holds that
	\begin{equation*}
		\int_{\mathbb{R}^m} \int_{\mathbb{R}} (\mathfrak{R}_\psi g)(a,b) \rho\left( a^\top u - b \right) db da = C^{(\psi,\rho)}_m g(u).
	\end{equation*}
\end{proposition}

In addition, we introduce the following Barron spaces which are inspired by the works \cite{barron93,klusowski16,e22}.

\begin{definition}
	\label{DefRidgeletBarronSpace}
	For $k \in \mathbb{N}_0$, $\gamma \in [0,\infty)$, and $\psi \in \mathcal{S}_0(\mathbb{R};\mathbb{C})$, we define the ridgelet-Barron space $\mathbb{B}^{k,\gamma}_\psi(U;\mathbb{R}^d)$ as vector space of $\mathcal{L}(U)/\mathcal{B}(\mathbb{R}^d)$-measurable functions $f: U \rightarrow \mathbb{R}^d$ such that
	\begin{equation*}
		\Vert f \Vert_{\mathbb{B}^{k,\gamma}_\psi(U;\mathbb{R}^d)} := \inf_g \left( \int_{\mathbb{R}^m} \int_{\mathbb{R}} \left( 1 + \Vert a \Vert^2 \right)^{\gamma+k+\frac{m+1}{2}} \left( 1 + \vert b \vert^2 \right)^{\gamma+1} \Vert (\mathfrak{R}_\psi g)(a,b) \Vert^2 db da \right)^\frac{1}{2} < \infty,
	\end{equation*}
	where the infimum is taken over all $g \in L^1(\mathbb{R}^m,\mathcal{L}(\mathbb{R}^m),du;\mathbb{R}^d)$ with $\widehat{g} \in L^1(\mathbb{R}^m,\mathcal{L}(\mathbb{R}^m),du;\mathbb{C}^d)$ and $g = f$ a.e.~on $U$.
\end{definition}

Now, we present the dimension-independent approximation rates, whose proof is given in Section~\ref{AppApproxRates}.

\begin{theorem}
	\label{ThmApproxRates}
	For $k \in \mathbb{N}_0$, $p \in [1,\infty)$, $U \subseteq \mathbb{R}^m$ (open, if $k \geq 1$), and $\gamma \in [0,\infty)$, let $w: U \rightarrow [0,\infty)$ be a weight such that
	\vspace{-0.2cm}
	\begin{equation}
		\label{EqThmApproxRates1}
		C^{(\gamma,p)}_{U,w} := \left( \int_U (1+\Vert u \Vert)^{\gamma p} w(u) du \right)^\frac{1}{p} < \infty.
	\end{equation}	
	Moreover, let $(\psi,\rho) \in \mathcal{S}_0(\mathbb{R};\mathbb{C}) \times C^k_{pol,\gamma}(\mathbb{R})$ be $m$-admissible. Then, there exists a constant $C_p > 0$ (depending only on $p \in [1,\infty)$) such that for every $f \in W^{k,p}(U,\mathcal{L}(U),w;\mathbb{R}^d) \cap \mathbb{B}^{k,\gamma}_\psi(U;\mathbb{R}^d)$ and every $N \in \mathbb{N}$ there exists a neural network $\varphi_N \in \mathcal{NN}^\rho_{U,d}$ having $N$ neurons satisfying\footnote{Hereby, $\Gamma$ denotes the Gamma function (see \cite[Section~6.1]{abramowitz70}).}
	\vspace{-0.1cm}
	\begin{equation}
		\label{EqThmApproxRates2}
		\Vert f - \varphi_N \Vert_{W^{k,p}(U,\mathcal{L}(U),w;\mathbb{R}^d)} \leq C_p \Vert \rho \Vert_{C^k_{pol,\gamma}(\mathbb{R})} \frac{C^{(\gamma,p)}_{U,w} m^\frac{k}{p} \pi^\frac{m+1}{4}}{\left\vert C^{(\psi,\rho)}_m \right\vert \Gamma\left( \frac{m+1}{2} \right)^\frac{1}{2}} \frac{\Vert f \Vert_{\mathbb{B}^{k,\gamma}_\psi(U;\mathbb{R}^d)}}{N^{1-\frac{1}{\min(2,p)}}}.
	\end{equation}
\end{theorem}

Theorem~\ref{ThmApproxRates} provides us with an upper bound on the number of neurons $N \in \mathbb{N}$ that are needed for a neural network to approximate a given function. Let us compare Theorem~\ref{ThmApproxRates} with the literature.

\begin{remark}
	Theorem~\ref{ThmApproxRates} generalizes the following approximation rates in the literature.
	\begin{enumerate}
		\item The rate $\mathcal{O}\big( 1/N^{1/2} \big)$ of Barron \cite{barron93} for approximating functions $f: \mathbb{R}^m \rightarrow \mathbb{R}$ with sufficiently integrable Fourier transform by neural networks (with sigmoidal activation) within $L^2(B_r(0),\mathcal{B}(B_r(0)),\mu)$, where $B_r(0) := \lbrace u \in \mathbb{R}^m: \Vert u \Vert \leq r \rbrace$ and $\mu$ is a probability measure.
		\item The rate $\mathcal{O}\big( 1/N^{1-1/\min(2,p)} \big)$ of Darken et al.~\cite{darken93} for approximating functions $f: \mathbb{R}^m \rightarrow \mathbb{R}$ being in the convex closure of $\mathcal{NN}^\rho_{U,1}$ by neural networks (with sigmoidal activation $\rho$) within $L^p(U,\mathcal{B}(U),\mu)$, where $p \in [1,\infty)$, $U \subseteq \mathbb{R}^m$, and $(U,\mathcal{B}(U),\mu)$ is a finite measure space.
		\item The rate $\mathcal{O}\big( 1/N^{1/2} \big)$ of Siegel and Xu \cite{siegel20} for approximating functions $f: \mathbb{R}^m \rightarrow \mathbb{R}$ with sufficiently integrable Fourier transform by neural networks (with linear combination of polynomially decaying activation functions) in $W^{k,2}(U,\mathcal{L}(U),du)$, where $U \subseteq \mathbb{R}^m$ is open and bounded.
	\end{enumerate}
	Hence, the contribution of Theorem~\ref{ThmApproxRates} are approximation rates for neural networks within weighted Sobolev spaces over unbounded domains.
\end{remark}

Next, we give a sufficient condition for a function $f: U \rightarrow \mathbb{R}^d$ to belong to the space $\mathbb{B}^{k,\gamma}_\psi(U;\mathbb{R}^d)$. The proof of the remaining results of this section can be found in Section~\ref{AppConstants}.

\begin{proposition}
	\label{PropConst}
	Let $k \in \mathbb{N}_0$, $U \subseteq \mathbb{R}^m$ (open, if $k \geq 1$), $\gamma \in [0,\infty)$, and let $\psi \in \mathcal{S}_0(\mathbb{R};\mathbb{C})$ such that $\zeta_1 := \inf\big\lbrace \vert \zeta \vert : \zeta \in \supp(\widehat{\psi}) \big\rbrace > 0$. Then, there exists a constant $C_1 > 0$ (independent of $m,d \in \mathbb{N}$) such that for any $f \in L^1(\mathbb{R}^m,\mathcal{L}(\mathbb{R}^m),du;\mathbb{R}^d)$ with $(\lceil\gamma\rceil+2)$-times differentiable Fourier transform, we have
	\begin{equation}
		\label{EqPropConst1}
		\Vert f \Vert_{\mathbb{B}^{k,\gamma}_\psi(U;\mathbb{R}^d)} \leq \frac{C_1}{\zeta_1^\frac{m}{2}} \sum_{\beta \in \mathbb{N}^m_{0,\lceil\gamma\rceil+2}} \left( \int_{\mathbb{R}^m} \big\Vert \partial_\beta \widehat{f}(\xi) \big\Vert^2 \left( 1 + \Vert \xi/\zeta_1 \Vert^2 \right)^\frac{4\lceil\gamma\rceil+2k+m+5}{2} d\xi \right)^\frac{1}{2}.
	\end{equation}
	In particular, if the right-hand side of \eqref{EqPropConst1} is finite, then $f \in \mathbb{B}^{k,\gamma}_\psi(U;\mathbb{R}^d)$. By using \cite[Theorem~7.8.(c)]{folland92} and the definition of Sobolev spaces via Fourier transform (see \cite[Definition~6.9]{grubb09}), this is the case if $\left( u \mapsto u^\beta f(u) \right) \in W^{4\lceil\gamma\rceil+2k+m+5,2}(\mathbb{R}^m,\mathcal{L}(\mathbb{R}^m),du;\mathbb{R}^d)$ for all $\beta \in \mathbb{N}^m_{0,\lceil\gamma\rceil+2}$, where $u^\beta := \prod_{l=1}^m u_l^{\beta_l}$.
\end{proposition}

In addition, we analyze the situation when neural networks overcome the curse of dimensionality in the sense that the computational costs (here measured as the number of neurons $N \in \mathbb{N}$) grow polynomially in both the dimensions $m,d \in \mathbb{N}$ and the reciprocal of a pre-specified error tolerance $\varepsilon > 0$. To this end, we estimate the constant $C^{(\gamma,p)}_{U,w}$, while a lower bound for $\big\vert C^{(\psi,\rho)}_m \big\vert$ is given below Example~\ref{ExAdm}.

\begin{lemma}
	\label{LemmaWeight}
	Let $k \in \mathbb{N}_0$, $p \in [1,\infty)$, $U \subseteq \mathbb{R}^m$ (open, if $k \geq 1$), $\gamma \in [0,\infty)$, and let $U \ni u := (u_1,...,u_m)^\top \mapsto w(u) := \prod_{l=1}^m w_0(u_l) \in [0,\infty)$ be a weight, where $w_0: \mathbb{R} \rightarrow [0,\infty)$ satisfies $\int_{\mathbb{R}} w_0(s) ds = 1$ and $C^{(\gamma,p)}_{\mathbb{R},w_0} := \big( \int_{\mathbb{R}} (1+\vert s \vert)^{\gamma p} w_0(s) ds \big)^{1/p} < \infty$. Then, $C^{(\gamma,p)}_{U,w} \leq C^{(\gamma,p)}_{\mathbb{R},w_0} m^{\gamma+1/p}$.
\end{lemma}

\begin{proposition}
	\label{PropCOD}
	For $k \in \mathbb{N}_0$, $p \in (1,\infty)$, $U \subseteq \mathbb{R}^m$ (open, if $k \geq 1$), and $\gamma \in [0,\infty)$, let $w: U \rightarrow [0,\infty)$ be a weight as in Lemma~\ref{LemmaWeight}. Moreover, let $(\psi,\rho) \in \mathcal{S}_0(\mathbb{R};\mathbb{C}) \times C^k_{pol,\gamma}(\mathbb{R})$ be a pair as in Example~\ref{ExAdm}. In addition, let $f \in W^{k,p}(U,\mathcal{L}(U),w;\mathbb{R}^d)$ satisfy the conditions of Proposition~\ref{PropConst} such that the right-hand side of \eqref{EqPropConst1} satisfies $\mathcal{O}\left( m^s (2/\zeta_2)^m (m+1)^{m/2} \right)$ for some $s \in \mathbb{N}_0$. Then, there exist some constants $C_2,C_3 > 0$ such that for every $m,d \in \mathbb{N}$ and every $\varepsilon > 0$ there exists a neural network $\varphi_N \in \mathcal{NN}^\rho_{U,d}$ with $N = \Big\lceil C_2 m^{C_3} \varepsilon^{-\frac{\min(2,p)}{\min(2,p)-1}} \Big\rceil$ neurons satisfying $\Vert f - \varphi_N \Vert_{W^{k,p}(U,\mathcal{L}(U),w;\mathbb{R}^d)} \leq \varepsilon$.
\end{proposition}

\section{Proofs}
\label{SecProofs}

\subsection{Proof of results in Section~\ref{SecIntro}.}
\label{SecProof1}

In this section, we show an equivalent characterization for functions in the weighted $C^k$-space $(\overline{C^k_b(U;\mathbb{R}^d)}^\gamma,\Vert \cdot \Vert_{C^k_{pol,\gamma}(U;\mathbb{R}^d)})$ introduced in Section~\ref{SubsecNotation}, where $k \in \mathbb{N}_0$, $U \subseteq \mathbb{R}^m$ (open, if $k \geq 1$), and $\gamma \in (0,\infty)$. This generalizes the results in \cite[Theorem~2.7]{doersek10} and \cite[Lemma~2.7]{cuchiero23} to differentiable functions defined on an open subset of a Euclidean space $\mathbb{R}^m$. 

In the following, we denote the factorial of a multi-index $\alpha := (\alpha_1,...,\alpha_m) \in \mathbb{N}^m_0$ by $\alpha! := \prod_{l=1}^m \alpha_l!$. Moreover, for any $r \geq 0$ and $u_0 \in \mathbb{R}^m$, we define $B_r(u_0) := \left\lbrace u \in \mathbb{R}^m: \Vert u-u_0 \Vert < r \right\rbrace$ and $\overline{B_r(u_0)} := \left\lbrace u \in \mathbb{R}^m: \Vert u-u_0 \Vert \leq r \right\rbrace$ as the open and closed ball with radius $r > 0$ around $u_0 \in \mathbb{R}^m$, respectively.

\begin{lemma}
	\label{LemmaCkw}
	Let $k \in \mathbb{N}_0$, $U \subseteq \mathbb{R}^m$ (open, if $k \geq 1$), and $\gamma \in (0,\infty)$. Then, the following holds true:
	\begin{enumerate}
		\item\label{LemmaCkw1} If $U \subseteq \mathbb{R}^m$ is bounded, then $\overline{C^k_b(U;\mathbb{R}^d)}^\gamma = C^k_b(U;\mathbb{R}^d)$.
		\item\label{LemmaCkw2} If $U \subseteq \mathbb{R}^m$ is unbounded, then $f \in \overline{C^k_b(U;\mathbb{R}^d)}^\gamma$ if and only if $f \in C^k(U;\mathbb{R}^d)$ and
		\begin{equation}
			\label{EqLemmaCkw1}
			\lim_{r \rightarrow \infty} \max_{\alpha \in \mathbb{N}^m_{0,k}} \sup_{u \in U \setminus B_r(0)} \frac{\Vert \partial_\alpha f(u) \Vert}{(1+\Vert u \Vert)^\gamma} = 0.
		\end{equation}
	\end{enumerate}
\end{lemma}
\begin{proof}
	The conclusion in \ref{LemmaCkw1} follows from the definition of $(\overline{C^k_b(U;\mathbb{R}^d)}^\gamma,\Vert \cdot \Vert_{C^k_{pol,\gamma}(U;\mathbb{R}^d)})$. Now, for necessity in \ref{LemmaCkw2}, fix some $f \in \overline{C^k_b(U;\mathbb{R}^d)}^\gamma$. Then, by definition of $\overline{C^k_b(U;\mathbb{R}^d)}^\gamma$, there exists a sequence $(g_n)_{n \in \mathbb{N}} \subseteq C^k_b(U;\mathbb{R}^d)$ with $\lim_{n \rightarrow \infty} \Vert f - g_n \Vert_{C^k_{pol,\gamma}(U;\mathbb{R}^d)} = 0$, which implies for every fixed $r > 0$ that
	\begin{equation*}
		\begin{aligned}
			\lim_{n \rightarrow \infty} \max_{\alpha \in \mathbb{N}^m_{0,k}} \sup_{u \in U \cap B_r(0)} \Vert \partial_\alpha f(u) - \partial_\alpha g_n(u) \Vert & \leq (1+r)^\gamma \lim_{n \rightarrow \infty} \max_{\alpha \in \mathbb{N}^m_{0,k}} \sup_{u \in U \cap B_r(0)} \frac{\Vert \partial_\alpha f(u) - \partial_\alpha g_n(u) \Vert}{(1+\Vert u \Vert)^\gamma} \\
			& \leq (1+r)^\gamma \lim_{n \rightarrow \infty} \Vert f - g_n \Vert_{C^k_{pol,\gamma}(U;\mathbb{R}^d)} = 0.
		\end{aligned}
	\end{equation*}
	This together with the Fundamental Theorem of Calculus shows that $f\vert_{U \cap B_r(0)}: U \cap B_r(0) \rightarrow \mathbb{R}^d$ is $k$-times differentiable since for every fixed $\alpha \in \mathbb{N}^m_{0,k}$ the partial derivative $\partial_\alpha f\vert_{U \cap B_r(0)}: U \cap B_r(0) \rightarrow \mathbb{R}^d$ is continuous as uniform limit of continuous functions. Hence, by using that $U$ is locally compact, it follows from \cite[Lemma~46.3+46.4]{munkres14} that $\partial_\alpha f: U \rightarrow \mathbb{R}^d$ is continuous everywhere on $U$. Since this holds true for every $\alpha \in \mathbb{N}^m_{0,k}$, we apply again the Fundamental Theorem of Calculus to conclude that $f \in C^k(U;\mathbb{R}^d)$. Moreover, in order to show \eqref{EqLemmaCkw1}, we fix some $\varepsilon > 0$ and choose some $n \in \mathbb{N}$ large enough such that $\Vert f - g_n \Vert_{C^k_{pol,\gamma}(U;\mathbb{R}^d)} < \varepsilon/2$. Moreover, we choose $r > 0$ sufficiently large such that $(1+r)^\gamma > 2 \varepsilon^{-1} \Vert g_n \Vert_{C^k_b(U;\mathbb{R}^d)}$ holds true, which implies that
	\begin{equation*}
		\begin{aligned}
			\max_{\alpha \in \mathbb{N}^m_{0,k}} \sup_{u \in U \setminus B_r(0)} \frac{\Vert \partial_\alpha f(u) \Vert}{(1+\Vert u \Vert)^\gamma} & \leq \max_{\alpha \in \mathbb{N}^m_{0,k}} \sup_{u \in U} \frac{\Vert \partial_\alpha f(u) - \partial_\alpha g_n(u) \Vert}{(1+\Vert u \Vert)^\gamma} + \max_{\alpha \in \mathbb{N}^m_{0,k}} \sup_{u \in U \setminus B_r(0)} \frac{\Vert \partial_\alpha g_n(u) \Vert}{(1+\Vert u \Vert)^\gamma} \\
			& < \frac{\varepsilon}{2} + \frac{\Vert g_n \Vert_{C^k_b(U;\mathbb{R}^d)}}{(1+r)^\gamma} < \frac{\varepsilon}{2} + \frac{\varepsilon}{2} = \varepsilon.
		\end{aligned}
	\end{equation*}
	Since $\varepsilon > 0$ was chosen arbitrarily, we obtain \eqref{EqLemmaCkw1}.
	
	For sufficiency in \ref{LemmaCkw2}, let $f \in C^k(U;\mathbb{R}^d)$ such that \eqref{EqLemmaCkw1} holds true and fix some $\varepsilon > 0$. Moreover, we choose some $h \in C^\infty_c(\mathbb{R}^m)$ such that $h(u) = 1$ for all $u \in \overline{B_1(0)}$, $h(u) = 0$ for all $u \in \mathbb{R}^m \setminus B_2(0)$, and that there exists a constant $C_h > 0$ such that for every $\alpha \in \mathbb{N}^m_{0,k}$ and $u \in \mathbb{R}^m$ it holds that $\vert \partial_\alpha h(u) \vert \leq C_h$. In addition, by using \eqref{EqLemmaCkw1}, there exists some $r > 1$ such that
	\begin{equation}
		\label{EqLemmaCkwProof1}
		\max_{\alpha \in \mathbb{N}^m_{0,k}} \sup_{u \in U \setminus B_r(0)} \frac{\Vert \partial_\alpha f(u) \Vert}{(1+\Vert u \Vert)^\gamma} < \frac{\varepsilon}{1+ 2^k C_h}.
	\end{equation}
	From this, we define the functions $\mathbb{R}^m \ni u \mapsto h_r(u) := h(u/r) \in \mathbb{R}$ and $U \ni u \mapsto g(u) := h_r(u) f(u) \in \mathbb{R}^d$, which both have bounded support. Furthermore, note that by the binomial theorem, we have for every $\alpha \in \mathbb{N}^m_0$ that
	\begin{equation}
		\label{EqLemmaCkwProof2}
		\sum_{\beta_1, \beta_2 \in \mathbb{N}^m_0 \atop \beta_1 + \beta_2 = \alpha} \frac{\alpha!}{\beta_1! \beta_2!} = \sum_{\beta \in \mathbb{N}^m_0 \atop \forall l: \beta_l \leq \alpha_l} \prod_{l=1}^m \frac{\alpha_l!}{\beta_l! (\alpha_l-\beta_l)!} \leq \prod_{l=1}^m \sum_{\beta_l=0}^{\alpha_l} \frac{\alpha_l!}{\beta_l! (\alpha_l-\beta_l)!} = \prod_{l=1}^m 2^{\alpha_l} \leq 2^{\vert \alpha \vert}.
	\end{equation}
	Then, by using the Leibniz product rule together with the triangle inequality, the inequality \eqref{EqLemmaCkwProof2}, that $\vert \partial_\alpha h_r(u) \vert = \left\vert \partial_\alpha h(u/r) \right\vert r^{-\vert \alpha \vert} \leq C_h$ for any $\alpha \in \mathbb{N}^m_{0,k}$ and $u \in \mathbb{R}^m$, and again the inequality \eqref{EqLemmaCkwProof2}, it follows for every $\alpha \in \mathbb{N}^m_{0,k}$ and $u \in U$ that
	\begin{equation}
		\label{EqLemmaCkwProof3}
		\Vert \partial_\alpha g(u) \Vert \leq \sum_{\beta_1, \beta_2 \in \mathbb{N}^m_0 \atop \beta_1 + \beta_2 = \alpha} \frac{\alpha!}{\beta_1! \beta_2!} \left\vert \partial_{\beta_1} h_r(u) \right\vert \left\Vert \partial_{\beta_2} f(u) \right\Vert \leq 2^k C_h \max_{\beta_2 \in \mathbb{N}^m_{0,k}} \left\Vert \partial_{\beta_2} f(u) \right\Vert.
	\end{equation}
	Hence, by using that $\partial_\alpha g(u) = \partial_\alpha (h_r(u) f(u))  = \partial_\alpha f(u)$ for any $\alpha \in \mathbb{N}^m_{0,k}$ and $u \in U \cap B_r(0)$ (as $h_r(u) = 1$ for any $u \in B_r(0)$), and the inequalities \eqref{EqLemmaCkwProof3} and \eqref{EqLemmaCkwProof1}, the function $g \in C^k_b(U;\mathbb{R}^d)$ satisfies
	\begin{equation*}
		\begin{aligned}
			\Vert f - g \Vert_{C^k_{pol,\gamma}(U;\mathbb{R}^d)} & = \max_{\alpha \in \mathbb{N}^m_{0,k}} \sup_{u \in U} \frac{\Vert \partial_\alpha f(u) - \partial_\alpha g(u) \Vert}{(1+\Vert u \Vert)^\gamma} \\
			& \leq \max_{\alpha \in \mathbb{N}^m_{0,k}} \sup_{u \in U \cap B_r(0)} \frac{\Vert \partial_\alpha f(u) - \partial_\alpha g(u) \Vert}{(1+\Vert u \Vert)^\gamma} + \max_{\alpha \in \mathbb{N}^m_{0,k}} \sup_{u \in U \setminus B_r(0)} \frac{\Vert \partial_\alpha f(u) \Vert}{(1+\Vert u \Vert)^\gamma} \\
			& \quad\quad + \max_{\alpha \in \mathbb{N}^m_{0,k}} \sup_{u \in U \setminus B_r(0)} \frac{\Vert \partial_\alpha g(u) \Vert}{(1+\Vert u \Vert)^\gamma} \\
			& \leq \max_{\alpha \in \mathbb{N}^m_{0,k}} \sup_{u \in U \setminus B_r(0)} \frac{\Vert \partial_\alpha f(u) \Vert}{(1+\Vert u \Vert)^\gamma} + 2^k C_h \max_{\alpha \in \mathbb{N}^m_{0,k}} \sup_{u \in U \setminus B_r(0)} \frac{\Vert \partial_\alpha f(u) \Vert}{(1+\Vert u \Vert)^\gamma} \\
			& < \frac{\varepsilon}{1+ 2^k C_h} + 2^k C_h \frac{\varepsilon}{1+ 2^k C_h} = \varepsilon.
		\end{aligned}
	\end{equation*}
	Since $\varepsilon > 0$ was chosen arbitrarily, it follows that $f \in \overline{C^k(U;\mathbb{R}^d)}^\gamma$.		
\end{proof}

\subsection{Proof of results in Section~\ref{SecNN}.}
\label{SecProof2}

\subsubsection{Proof of Lemma~\ref{LemmaProp}.}
\label{AppLemmaProp}

\begin{proof}[Proof of Lemma~\ref{LemmaProp}]
	Fix some $\rho \in \overline{C^k_b(\mathbb{R})}^\gamma$, $y \in \mathbb{R}^d$, $a \in \mathbb{R}^m$, and $b \in \mathbb{R}$, and define the constant $C_{y,a,b} := 1 + \max_{\alpha \in \mathbb{N}^m_{0,k}} \left\Vert y a^\alpha \right\Vert (1+\Vert a \Vert + \vert b \vert)^\gamma > 0$, where $a^\alpha := \prod_{l=1}^m a_l^{\alpha_l}$ for $a := (a_1,...,a_m)^\top \in \mathbb{R}^m$ and $\alpha := (\alpha_1,...,\alpha_m) \in \mathbb{N}^m_{0,k}$. Then, by using the definition of $\overline{C^k_b(\mathbb{R})}^\gamma$, there exists some $\widetilde{\rho} \in C^k_b(\mathbb{R})$ such that
	\begin{equation*}
		\Vert \rho - \widetilde{\rho} \Vert_{C^k_{pol,\gamma}(\mathbb{R})} := \max_{j=0,...,k} \sup_{s \in \mathbb{R}} \frac{\left\vert \rho^{(j)}(s) - \widetilde{\rho}^{(j)}(s) \right\vert}{(1+\vert s \vert)^\gamma} < \frac{\varepsilon}{C_{y,a,b}}.
	\end{equation*}
	Hence, by using the inequality $1+\left\vert a^\top u - b \right\vert \leq 1+\Vert a \Vert \Vert u \Vert + \vert b \vert \leq (1+\Vert a \Vert + \vert b \vert) (1 + \Vert u \Vert)$, it follows for the function $y \widetilde{\rho}\left( a^\top \cdot - b \right) := \left( u \mapsto y \widetilde{\rho}\left( a^\top u - b \right) \right) \in C^k_b(\mathbb{R}^m;\mathbb{R}^d)$ that
	\begin{equation*}
		\begin{aligned}
			& \left\Vert y \rho\left( a^\top \cdot - b \right) - y \widetilde{\rho}\left( a^\top \cdot - b \right) \right\Vert_{C^k_{pol,\gamma}(\mathbb{R}^m;\mathbb{R}^d)} = \max_{\alpha \in \mathbb{N}^m_{0,k}} \sup_{u \in \mathbb{R}^m} \frac{\left\Vert y \rho^{(\vert \alpha \vert)}\left( a^\top u - b \right) a^\alpha - y \widetilde{\rho}^{(\vert \alpha \vert)}\left( a^\top u - b \right) a^\alpha \right\Vert}{(1+\Vert u \Vert)^\gamma} \\
			& \quad\quad \leq \left( \max_{\alpha \in \mathbb{N}^m_{0,k}} \left\Vert y a^\alpha \right\Vert (1+\Vert a \Vert + \vert b \vert) \right) \max_{\alpha \in \mathbb{N}^m_{0,k}} \sup_{u \in \mathbb{R}^m} \frac{\left\Vert \rho^{(\vert\alpha\vert)}\left( a^\top u - b \right) - \widetilde{\rho}^{(\vert\alpha\vert)}\left( a^\top u - b \right) \right\Vert}{\left( 1 + \left\vert a^\top u - b \right\vert \right)^\gamma} \\
			& \quad\quad \leq C_{y,a,b} \max_{j=0,...,k} \sup_{s \in \mathbb{R}} \frac{\left\vert \rho^{(j)}(s) - \widetilde{\rho}^{(j)}(s) \right\vert}{(1+\vert s \vert)^\gamma} < C_{y,a,b} \frac{\varepsilon}{C_{y,a,b}} = \varepsilon.
		\end{aligned}
	\end{equation*}
	Since $\varepsilon > 0$ was chosen arbitrarily and $y \widetilde{\rho}\left( a^\top \cdot - b \right) \in C^k_b(\mathbb{R}^m;\mathbb{R}^d)$, it follows that $y \rho\left( a^\top \cdot - b \right) \in \overline{C^k_b(\mathbb{R}^m;\mathbb{R}^d)}^\gamma$. Thus, by using that $\mathcal{NN}^\rho_{\mathbb{R}^m,d}$ is defined as vector space consisting of functions of the form $\mathbb{R}^m \ni u \mapsto y \rho\left( a^\top u - b \right) \in \mathbb{R}^d$, with $y \in \mathbb{R}^d$, $a \in \mathbb{R}^m$, and $b \in \mathbb{R}$, the triangle inequality implies that $\mathcal{NN}^\rho_{\mathbb{R}^m,d} \subseteq \overline{C^k_b(\mathbb{R}^m;\mathbb{R}^d)}^\gamma$. Finally, by using that $(X,\Vert \cdot \Vert_X)$ satisfies Assumption~\ref{AssApprox}, i.e.~that the restriction map \eqref{EqAssApprox1} is a continuous embedding, it follows that $\mathcal{NN}^\rho_{U,d} \subseteq X$.
\end{proof}

\subsubsection{Proof of Theorem~\ref{ThmUAT}.}
\label{AppThmUAT}

In this section, we provide the proof of Theorem~\ref{ThmUAT}, i.e.~the universal approximation property of neural networks $\mathcal{NN}^\rho_{U,d}$ within any function space $(X,\Vert \cdot \Vert_X)$ satisfying Assumption~\ref{AssApprox} with $k \in \mathbb{N}_0$, $U \subseteq \mathbb{R}^m$ (open, if $k \geq 1$), and $\gamma \in (0,\infty)$, where $\rho \in \overline{C^k_b(\mathbb{R})}^\gamma$.

The idea of the proof is the following. By contradiction, we assume that $\mathcal{NN}^\rho_{U,d} \subseteq X$ is not dense in $X$. Then, by using the Hahn-Banach theorem (as in \cite[Theorem~1]{cybenko89}), there exists a non-zero continuous linear functional $l: X \rightarrow \mathbb{R}$ which vanishes on the vector subspace $\mathcal{NN}^\rho_{U,d} \subseteq X$. Moreover, by using the continuous embedding in \eqref{EqAssApprox1}, we can express $l: X \rightarrow \mathbb{R}$ on the dense subspace $\overline{C^k_b(\mathbb{R}^m;\mathbb{R}^d)}^\gamma$ with finite signed Radon measures, which relies on the Riesz representation theorem in \cite[Theorem~2.4]{doersek10}. Subsequently, we use the distributional extension of Wiener's Tauberian theorem in \cite{korevaar65} and that $\rho \in \overline{C^k_b(\mathbb{R})}^\gamma$ is non-polynomial to conclude that $l: X \rightarrow \mathbb{R}$ vanishes everywhere on $X$. This however contradicts the initial assumption that $l: X \rightarrow \mathbb{R}$ is non-zero. Hence, $\mathcal{NN}^\rho_{U,d}$ must be dense in $X$.

In order to prove Theorem~\ref{ThmUAT} as outlined above, we now first generalize the Riesz representation theorem in \cite[Theorem~2.7]{doersek10} to this vector-valued case with derivatives. Hereby, we define $\mathcal{M}_\gamma(\mathbb{R}^m)$ as the vector space of finite signed Radon measures $\eta: \mathcal{B}(\mathbb{R}^m) \rightarrow \mathbb{R}$ with $\int_{\mathbb{R}^m} (1+\Vert u \Vert)^\gamma \vert \eta \vert(du) < \infty$, where $\vert \eta \vert: \mathcal{B}(\mathbb{R}^m) \rightarrow [0,\infty)$ denotes the corresponding total variation measure. Moreover, we denote by $Z^*$ the dual space of a Banach space $(Z,\Vert \cdot \Vert_Z)$ which consists of continuous linear functionals $l: Z \rightarrow \mathbb{R}$ and is equipped with the norm $\Vert l \Vert_{Z^*} := \sup_{z \in Z, \, \Vert z \Vert_Z \leq 1} \vert l(z) \vert$.

\begin{proposition}[Riesz representation]
	\label{PropRiesz}
	For $k \in \mathbb{N}_0$ and $\gamma \in (0,\infty)$, let $l: \overline{C^k_b(\mathbb{R}^m;\mathbb{R}^d)}^\gamma \rightarrow \mathbb{R}$ be a continuous linear functional. Then, there exist some signed Radon measures $(\eta_{\alpha,i})_{\alpha \in \mathbb{N}^m_{0,k}, \, i = 1,...,d} \subseteq \mathcal{M}_\gamma(\mathbb{R}^m)$ such that for every $f = (f_1,...,f_d)^\top \in \overline{C^k_b(\mathbb{R}^m;\mathbb{R}^d)}^\gamma$ it holds that
	\begin{equation*}
		l(f) = \sum_{\alpha \in \mathbb{N}^m_{0,k}} \sum_{i=1}^d \int_{\mathbb{R}^m} \partial_\alpha f_i(u) \eta_{\alpha,i}(du).
	\end{equation*}
\end{proposition}
\begin{proof}
	First, we show the conclusion for $k = 0$ and $d = 1$. Indeed, by defining $\mathbb{R}^m \ni u \mapsto \psi(u) := (1+\Vert u \Vert)^\gamma \in (0,\infty)$, the tuple $(\mathbb{R}^m,\psi)$ is a weighted space in the sense of \cite[p.~5]{doersek10}. Hence, the conclusion follows from \cite[Theorem~2.4]{doersek10}.
	
	Now, for the general case of $k \geq 1$ and $d \geq 2$, we fix a continuous linear functional $l: \overline{C^k_b(\mathbb{R}^m;\mathbb{R}^d)}^\gamma \rightarrow \mathbb{R}$ and define the number $M := \vert \mathbb{N}^m_{0,k} \vert \cdot d$ as well as the map
	\begin{equation*}
		\overline{C^k_b(\mathbb{R}^m;\mathbb{R}^d)}^\gamma \ni f \quad \mapsto \quad \Xi(f) := (\partial_\alpha f_i)_{\alpha \in \mathbb{N}^m_{0,k}, \, i = 1,...,d}^\top \in \overline{C^0_b(\mathbb{R}^m;\mathbb{R}^M)}^\gamma.	
	\end{equation*}
	Moreover, we denote by $\img(\Xi) := \big\lbrace \Xi(f): f \in \overline{C^k_b(\mathbb{R}^m;\mathbb{R}^d)}^\gamma \big\rbrace \subseteq  \overline{C^0_b(\mathbb{R}^m;\mathbb{R}^M)}^\gamma$ the image vector subspace. Then, by using that $\Xi: \overline{C^k_b(\mathbb{R}^m;\mathbb{R}^d)}^\gamma \rightarrow \img(\Xi)$ is by definition bijective, there exists an inverse map $\Xi^{-1}: \img(\Xi) \rightarrow 	\overline{C^k_b(\mathbb{R}^m;\mathbb{R}^d)}^\gamma$. Moreover, we conclude for every $f \in 	\overline{C^k_b(\mathbb{R}^m;\mathbb{R}^d)}^\gamma$ that
	\begin{equation*}
		\begin{aligned}
			\left\Vert \Xi^{-1}((\partial_\alpha f_i)_{\alpha \in \mathbb{N}^m_{0,k}, \, i=1,...,d}) \right\Vert_{C^0_{pol,\gamma}(\mathbb{R}^m;\mathbb{R}^d)} & = \Vert f \Vert_{C^k_{pol,\gamma}(\mathbb{R}^m;\mathbb{R}^d)} = \max_{\alpha \in \mathbb{N}^m_{0,k}} \sup_{u \in \mathbb{R}^m} \frac{\Vert \partial_\alpha f(u) \Vert}{(1+\Vert u \Vert)^\gamma} \\
			& = \sup_{u \in \mathbb{R}^m} \max_{\alpha \in \mathbb{N}^m_{0,k}} \frac{\Vert \partial_\alpha f(u) \Vert}{(1+\Vert u \Vert)^\gamma} \leq \sup_{u \in \mathbb{R}^m} \frac{\Vert (\partial_\alpha f_i)_{\alpha \in \mathbb{N}^m_{0,k}, \, i=1,...,d} \Vert}{(1+\Vert u \Vert)^\gamma} \\
			& = \Vert (\partial_\alpha f_i)_{\alpha \in \mathbb{N}^m_{0,k}, \, i=1,...,d} \Vert_{C^0_{pol,\gamma}(\mathbb{R}^m;\mathbb{R}^M)},
		\end{aligned}
	\end{equation*}
	which shows that $\Xi^{-1}: \img(\Xi) \rightarrow \overline{C^k_b(\mathbb{R}^m;\mathbb{R}^d)}^\gamma$ is continuous. Hence, the concatenation $l \circ \Xi^{-1}: \img(\Xi) \rightarrow \mathbb{R}$ is a continuous linear functional on $\img(\Xi)$, which can be extended by using the Hahn-Banach theorem to a continuous linear functional $l_0: \overline{C^0_b(\mathbb{R}^m;\mathbb{R}^M)}^\gamma \rightarrow \mathbb{R}$ such that for every $f \in \overline{C^k_b(\mathbb{R}^m;\mathbb{R}^d)}^\gamma$ it holds that
	\begin{equation}
		\label{EqPropRieszProof2}
		l_0((\partial_\alpha f_i)_{\alpha \in \mathbb{N}^m_{0,k}, \, i=1,...,d}) = \left( l \circ \Xi^{-1} \right)((\partial_\alpha f_i)_{\alpha \in \mathbb{N}^m_{0,k}, \, i=1,...,d}) = l(f).
	\end{equation}
	Now, for every fixed $\alpha \in \mathbb{N}^m_{0,k}$ and $i = 1,...,d$, we define the linear map $\overline{C^0_b(\mathbb{R}^m)}^\gamma \ni g \mapsto l_{\alpha,i}(g) := l_0(g e_{\alpha,i}) \in \mathbb{R}$, where $e_{\alpha,i} \in \mathbb{R}^M := \mathbb{R}^{\vert \mathbb{N}^m_{0,k} \vert \cdot d} \cong \mathbb{R}^{\vert \mathbb{N}^m_{0,k} \vert} \times \mathbb{R}^d$ denotes the $(\alpha,i)$-th unit vector of $\mathbb{R}^M := \mathbb{R}^{\vert \mathbb{N}^m_{0,k} \vert \cdot d} \cong \mathbb{R}^{\vert \mathbb{N}^m_{0,k} \vert} \times \mathbb{R}^d$. Then, for every $g \in \overline{C^0_b(\mathbb{R}^m)}^\gamma$, it follows with $Z := \overline{C^0_b(\mathbb{R}^m;\mathbb{R}^M)}^\gamma$ that
	\begin{equation*}
		\begin{aligned}
			\left\vert l_{\alpha,i}(g) \right\vert & = \left\vert l_0(g e_{\alpha,i}) \right\vert \leq \Vert l_0 \Vert_{Z^*} \Vert g e_{\alpha,i} \Vert_{C^0_{pol,\gamma}(\mathbb{R}^m;\mathbb{R}^M)} = \Vert l_0 \Vert_{Z^*} \Vert g \Vert_{C^0_{pol,\gamma}(\mathbb{R}^m)},
		\end{aligned}
	\end{equation*}
	which shows that $l_{\alpha,i}: \overline{C^0_b(\mathbb{R}^m)}^\gamma \rightarrow \mathbb{R}$ is a continuous linear functional. Hence, by using \eqref{EqPropRieszProof2} and by applying for every $\alpha \in \mathbb{N}^m_{0,k}$ and $i = 1,...,d$ the case with $k = 0$ and $d = 1$, there exist some Radon measures $(\eta_{\alpha,i})_{\alpha \in \mathbb{N}^m_{0,k}, i=1,...,d} \in \mathcal{M}_\gamma(\mathbb{R}^m)$ such that for every $f \in \overline{C^k_b(\mathbb{R}^m;\mathbb{R}^d)}^\gamma$ it holds that
	\begin{equation*}
		\begin{aligned}
			l(f) & = \left( l \circ \Xi^{-1} \right)((\partial_\alpha f_i)_{\alpha \in \mathbb{N}^m_{0,k}, \, i=1,...,d}) \\
			& = l_0((\partial_\alpha f_i)_{\alpha \in \mathbb{N}^m_{0,k}, \, i=1,...,d}) \\
			& = \sum_{\alpha \in \mathbb{N}^m_{0,k}} \sum_{i=1}^d l_{\alpha,i}(\partial_\alpha f_i e_{\alpha,i}) \\
			& = \sum_{\alpha \in \mathbb{N}^m_{0,k}} \sum_{i=1}^d \int_{\mathbb{R}^m} \partial_\alpha f_i(u) \eta_{\alpha,i}(du),
		\end{aligned}
	\end{equation*}
	which completes the proof.
\end{proof}

Next, we show that every non-polynomial activation function $\rho \in \overline{C^k_b(\mathbb{R})}^\gamma$ is discriminatory in the sense of \cite[p.~306]{cybenko89}. To this end, we generalize the proof of \cite[Theorem~1]{chen95} from compactly supported signed Radon measures to measures in $\mathcal{M}_\gamma(\mathbb{R}^m)$. Hereby, we follow the distributional extension of Wiener's Tauberian theorem in \cite[Theorem~A]{korevaar65}.

\begin{proposition}
	\label{PropDiscr}
	For $\gamma \in (0,\infty)$, let $\eta \in \mathcal{M}_\gamma(\mathbb{R}^m)$ be a signed Radon measure and assume that $\rho \in \overline{C^0_b(\mathbb{R})}^\gamma$ is non-polynomial. If for every $a \in \mathbb{R}^m$ and $b \in \mathbb{R}$ it holds that
	\begin{equation}
		\label{EqPropDiscr1}
		\int_{\mathbb{R}^m} \rho\left( a^\top u - b \right) \eta(du) = 0,
	\end{equation}
	then it follows that $\eta = 0 \in \mathcal{M}_\gamma(\mathbb{R}^m)$.
\end{proposition}
\begin{proof}
	We follow the proof of \cite[Proposition~4.4~(A3)]{cuchiero23} and assume that $\rho \in \overline{C^0_b(\mathbb{R})}^\gamma$ is non-polynomial. Then, by using \cite[Examples~7.16]{rudin91}, there exists a non-zero point $t_0 \in \mathbb{R} \setminus \lbrace 0 \rbrace$ which belongs to the support of $\widehat{T_\rho} \in \mathcal{S}'(\mathbb{R};\mathbb{C})$. Moreover, let $\eta \in \mathcal{M}_\gamma(\mathbb{R}^m)$ satisfy \eqref{EqPropDiscr1} and assume by contradiction that $\eta \in \mathcal{M}_\gamma(\mathbb{R}^m)$ is non-zero.
	
	Now, for every $a \in \mathbb{R}^m$, we define the push-forward measure $\eta_a := \eta \circ \left( a^\top \cdot \right)^{-1}: \mathcal{B}(\mathbb{R}) \rightarrow \mathbb{R}$ by $\eta_a(B) := \eta\big( \big\lbrace u \in \mathbb{R}^m: a^\top u \in B \big\rbrace \big)$, for $B \in \mathcal{B}(\mathbb{R})$. Moreover, for every fixed $\lambda \in \mathbb{R} \setminus \lbrace 0 \rbrace$, we define the function $\mathbb{R} \ni s \mapsto \rho_\lambda(s) := \rho(\lambda s) \in \mathbb{R}$. Then, by applying \cite[Theorem~3.6.1]{bogachev07} (to the positive and negative part of $\eta \in \mathcal{M}_\gamma(\mathbb{R}^m)$) and by using the assumption \eqref{EqPropDiscr1} (with $\lambda a \in \mathbb{R}^m$ and $\lambda b \in \mathbb{R}$ instead of $a \in \mathbb{R}^m$ and $b \in \mathbb{R}$, respectively), it follows for every $a \in \mathbb{R}^m$ and $b \in \mathbb{R}$ that
	\begin{equation}
		\label{EqPropDiscrProof1}
		\int_{\mathbb{R}} \rho_\lambda(s-b) \eta_a(ds) = \int_{\mathbb{R}^m} \rho\left( \lambda a^\top u - \lambda b \right) \eta(du) = 0.
	\end{equation}
	Since $\eta \in \mathcal{M}_\gamma(\mathbb{R}^m)$ is non-zero, there exists some $a \in \mathbb{R}^m$ such that $\eta_a: \mathcal{B}(\mathbb{R}) \rightarrow \mathbb{R}$ is non-zero. Hence, there exists some $h \in \mathcal{S}(\mathbb{R};\mathbb{C})$ such that $\big( z \mapsto f(z) := (h * \eta_a)(-z) := \int_{\mathbb{R}} h(-z-s) \eta_a(ds) \big) \in L^1(\mathbb{R},\mathcal{L}(\mathbb{R}),du;\mathbb{C})$ is also non-zero. Then, by using that the Fourier transform is injective, $\widehat{f}: \mathbb{R} \rightarrow \mathbb{C}$ is non-zero, too, i.e.~there exists some $t_1 \in \mathbb{R} \setminus \lbrace 0 \rbrace$ such that $\widehat{f}(t_1) \neq 0$. Hence, by using \cite[Table~7.2.2]{folland92}, the function $\big( z \mapsto f_0(z) := f(z) e^{-i t_1 z} \big) \in L^1(\mathbb{R},\mathcal{L}(\mathbb{R}),du;\mathbb{C})$ satisfies $\widehat{f_0}(0) = \widehat{f}(t_1) \neq 0$. Moreover, we choose $\lambda := \frac{t_1}{t_0} \in \mathbb{R} \setminus \lbrace 0 \rbrace$ and define the function $\mathbb{R} \ni z \mapsto \rho_0(z) := \rho_\lambda(z) e^{-i t_1 z} \in \mathbb{C}$.
	
	Next, we use \cite[Theorem~3.6.1]{bogachev07} (applied to $\vert \eta \vert: \mathcal{B}(\mathbb{R}^m) \rightarrow [0,\infty)$), the inequality $1+\big\vert \lambda a^\top u - b \big\vert \leq 1+ \vert\lambda\vert \Vert a \Vert \Vert u \Vert + \vert \lambda \vert \vert b \vert \leq \max(1,\vert \lambda \vert) (1+\Vert a \Vert)(1+\vert b \vert)(1+\Vert u \Vert)$ for any $a,u \in \mathbb{R}^m$ and $b,y \in \mathbb{R}$, the inequality $(1+\vert b \vert)^\gamma \leq 2^\gamma \big( 1+\vert b \vert^2 \big)^{\gamma/2} \leq 2^\gamma \big( 1+\vert b \vert^2 \big)^{\lceil\gamma/2\rceil}$ for any $b \in \mathbb{R}$, and that for every $y \in \mathbb{R}$ the reflected translation $\mathbb{R} \ni b \mapsto \widetilde{h}_y(b) := h(-y-b) \in \mathbb{R}$ of the Schwartz function $h \in \mathcal{S}(\mathbb{R};\mathbb{C})$ is again a Schwartz function (see \cite[p.~331]{folland92}) to conclude for every $y \in \mathbb{R}$ that
	\begin{equation}
		\label{EqPropDiscrProof2}
		\begin{aligned}
			& \int_{\mathbb{R}} \int_{\mathbb{R}} \vert h(-y-b) \vert \vert \rho_\lambda(s - b) \vert \vert\eta_a\vert(ds) db = \int_{\mathbb{R}^m} \vert h(-y-b) \vert \int_{\mathbb{R}} \left\vert \rho\left( \lambda a^\top u - \lambda b \right) \right\vert \vert\eta\vert(du) db \\
			& \quad\quad \leq \int_{\mathbb{R}} \vert h(-y-b) \vert \left( \sup_{u \in \mathbb{R}^m} \frac{\left\vert \rho\left( \lambda a^\top u - \lambda b \right) \right\vert}{\left( 1+\left\vert \lambda a^\top u - \lambda b \right\vert \right)^\gamma} \right) \int_{\mathbb{R}} \left( 1+\left\vert \lambda a^\top u - \lambda b \right\vert \right)^\gamma \vert\eta\vert(du) db \\
			& \quad\quad \leq \max(1,\vert\lambda\vert)^\gamma (1+\Vert a \Vert)^\gamma \left( \sup_{s \in \mathbb{R}} \frac{\vert \rho(s) \vert}{(1+\vert s \vert)^\gamma} \right) \left( \int_{\mathbb{R}} \vert h(-y-b) \vert (1+\vert b \vert)^\gamma db \right) \int_{\mathbb{R}} (1+\Vert u \Vert)^\gamma \vert\eta\vert(du) \\
			& \quad\quad \leq \max(1,\vert\lambda\vert)^\gamma (1+\Vert a \Vert)^\gamma \Vert \rho \Vert_{C^0_{pol,\gamma}(\mathbb{R})} \left( \sup_{y \in \mathbb{R}} \left\vert \widetilde{h}_y(b) \right\vert \left( 1+\vert b \vert^2 \right)^{\lceil\gamma/2\rceil+1} \right) \\
			& \quad\quad\quad\quad \cdot \left( \int_{\mathbb{R}} \frac{1}{1+b^2} db \right) \int_{\mathbb{R}} (1+\Vert u \Vert)^\gamma \vert\eta\vert(du) < \infty. \\
		\end{aligned}
	\end{equation}
	Then, by using the substitution $z \mapsto s-b$ and the identity \eqref{EqPropDiscrProof1}, it follows for every $y \in \mathbb{R}$ that
	\begin{equation}
		\label{EqPropDiscrProof3}
		\begin{aligned}
			& (f_0 * \rho_0)(y) = \int_{\mathbb{R}} f(y-z) e^{it_1(y-z)} \rho_\lambda(z) e^{-it_1z} dz = e^{it_1 y} \int_{\mathbb{R}} (h * \eta_a)(z-y) \rho_\lambda(z) dz \\
			& \quad\quad = e^{it_1 y} \int_{\mathbb{R}} \int_{\mathbb{R}} h(z-y-s) \rho_\lambda(z) \eta_a(ds) dz = e^{it_1 y} \int_{\mathbb{R}} h(-y-b) \int_{\mathbb{R}} \rho_\lambda(s-b) \eta_a(ds) db = 0,
		\end{aligned}
	\end{equation}
	where \eqref{EqPropDiscrProof2} ensures that the convolution $f_0 * \rho_0: \mathbb{R} \rightarrow \mathbb{R}$ is well-defined.
	
	Moreover, let $\phi \in \mathcal{S}(\mathbb{R};\mathbb{C})$ such that $\widehat{\phi}(\xi) = 1$, for all $\xi \in [-1,1]$, and $\widehat{\phi}(\xi) = 0$, for all $\xi \in \mathbb{R} \setminus [-2,2]$. In addition, for every $n \in \mathbb{N}$, we define $\big( s \mapsto \phi_n(s) := \frac{1}{n} \phi\big( \frac{1}{n} \big) \big) \in \mathcal{S}(\mathbb{R};\mathbb{C})$. Then, by following the proof of \cite[Theorem~A]{korevaar65}, there exists some large enough $n \in \mathbb{N}$ and $w \in L^1(\mathbb{R},\mathcal{L}(\mathbb{R}),du)$ such that $w * f_0 = \phi_{2n} \in \mathcal{S}(\mathbb{R};\mathbb{C})$. Hence, by using \eqref{EqPropDiscrProof2}, we conclude for every $g \in \mathcal{S}(\mathbb{R};\mathbb{C})$ that
	\begin{equation}
		\label{EqPropDiscrProof4}
		\left( T_{\rho_0} * \phi_{2n} \right)(g) := T_{\rho_0}\left( \phi_{2n}(-\,\cdot) * g \right) = (g * \phi_{2n} * \rho_0)(0) = (g * w * f_0 * \rho_0)(0) = 0,
	\end{equation}
	where $\phi_{2n}(-\,\cdot)$ denotes the function $\mathbb{R} \ni s \mapsto \phi_{2n}(-s) \in \mathbb{R}$. Thus, by using \cite[Equation~9.32]{folland92} together with \eqref{EqPropDiscrProof4}, i.e.~that $\widehat{\phi_{2n}} \widehat{T_{\rho_0}} = \reallywidehat{T_{\rho_0} * \phi_{2n}} = 0 \in \mathcal{S}'(\mathbb{R};\mathbb{C})$, and that $\widehat{\phi_{2n}}(\xi) = \widehat{\phi}(2n \xi) = 1$ for any $\xi \in [-\frac{1}{2n},\frac{1}{2n}]$, it follows that $\widehat{T_{\rho_0}} \in \mathcal{S}'(\mathbb{R};\mathbb{C})$ vanishes on $(-\frac{1}{2n},\frac{1}{2n})$.
	
	Finally, for any fixed $g \in C^\infty((t_0-\frac{1}{2n \vert \lambda \vert},t_0+\frac{1}{2n \vert \lambda \vert});\mathbb{C})$, we define $\big( z \mapsto g_0(z) := g\left( \frac{z}{\lambda} + t_0 \right) \big) \in C^\infty_c((-\frac{1}{2n},\frac{1}{2n});\mathbb{C})$. Hence, by using the definition of $\widehat{T_\rho} \in \mathcal{S}'(\mathbb{R};\mathbb{C})$, the substitution $\zeta \mapsto \xi/\lambda$, \cite[Table~9.2.2]{folland92}, and that $\widehat{T_{\rho_0}} \in \mathcal{S}'(\mathbb{R};\mathbb{C})$ vanishes on $(-\frac{1}{2n},\frac{1}{2n})$, we conclude that
	\begin{equation}
		\label{EqPropDiscrProof5}
		\begin{aligned}
			\widehat{T_\rho}(g) & = T_\rho(\widehat{g}) = \int_{\mathbb{R}} \rho(\xi) \widehat{g}(\xi) d\xi = \lambda \int_{\mathbb{R}} \rho(\lambda \zeta) \widehat{g}(\lambda \zeta) d\zeta = \int_{\mathbb{R}} \rho_0(\zeta) e^{i t_1 \zeta} \widehat{g(\cdot \, / \lambda)}(\zeta) dz \\
			& = \int_{\mathbb{R}} \rho_0(\zeta) \widehat{g_0}(\zeta) d\zeta = T_{\rho_0}(\widehat{g_0}) = \widehat{T_{\rho_0}}(g_0) = 0,
		\end{aligned}
	\end{equation}
	where $\widehat{g(\cdot \, / \lambda)}$ denotes the Fourier transform of the function $(s \mapsto g(s/\lambda)) \in \mathcal{S}(\mathbb{R};\mathbb{C})$. Since the function $g \in C^\infty_c((t_0-\frac{1}{2n \vert \lambda \vert},t_0+\frac{1}{2n \vert \lambda \vert});\mathbb{C})$ was chosen arbitrary, \eqref{EqPropDiscrProof5} shows that $\widehat{T_\rho} \in \mathcal{S}'(\mathbb{R};\mathbb{C})$ vanishes on the set $\big(t_0-\frac{1}{2n \vert \lambda \vert},t_0+\frac{1}{2n \vert \lambda \vert}\big)$. This however contradicts the assumption that $t_0 \in \mathbb{R} \setminus \lbrace 0 \rbrace$ belongs to the support of $\widehat{T_\rho} \in \mathcal{S}'(\mathbb{R};\mathbb{C})$ and shows that $\eta = 0 \in \mathcal{M}_\gamma(\mathbb{R})$.
\end{proof}

Next, we show some properties of measures $\eta \in \mathcal{M}_\gamma(\mathbb{R}^m)$, $\gamma \in (0,\infty)$, whenever they are convoluted with a bump function. To this end, we introduce the smooth bump function $\phi: \mathbb{R}^m \rightarrow \mathbb{R}$ defined by
\begin{equation*}
	\phi(u) :=
	\begin{cases}
		C e^{-\frac{1}{1-\Vert u \Vert^2}}, & u \in B_1(0), \\
		0, & u \in \mathbb{R}^m \setminus B_1(0),
	\end{cases}
\end{equation*}
where $C > 0$ is a normalizing constant such that $\Vert \phi \Vert_{L^1(\mathbb{R}^m,\mathcal{L}(\mathbb{R}^m),du)} = 1$. From this, we define for every fixed $\delta > 0$ the mollifier $\mathbb{R}^m \ni u \mapsto \phi_\delta(u) := \frac{1}{\delta^m} \phi\left( \frac{u}{\delta} \right) \in \mathbb{R}$. Moreover, for any $\gamma \in (0,\infty)$ and $\eta \in \mathcal{M}_\gamma(\mathbb{R}^m)$, we define the function $\mathbb{R}^m \ni u \mapsto (\phi_\delta * \eta)(u) := \int_{\mathbb{R}^m} \phi_\delta(u-v) \eta(dv) \in \mathbb{R}$.

\begin{lemma}
	\label{LemmaConvo}
	For $\gamma \in (0,\infty)$, let $\eta \in \mathcal{M}_\gamma(\mathbb{R}^m)$ and $f \in \overline{C^0_b(\mathbb{R}^m)}^\gamma$. Then, the following holds true:
	\begin{enumerate}
		\item\label{LemmaConvo0} For every $\delta > 0$ the function $\phi_\delta * \eta: \mathbb{R}^m \rightarrow \mathbb{R}$ is smooth with $\partial_\alpha (\phi_\delta * \eta)(u) = (\partial_\alpha \phi_\delta * \eta)(u)$ for all $\alpha \in \mathbb{N}^m_0$ and $u \in \mathbb{R}^m$.
		\item\label{LemmaConvo1} For every $\delta > 0$ and $\alpha \in \mathbb{N}^m_0$ it holds that
		\begin{equation*}
			\quad\quad\quad \lim_{r \rightarrow \infty} \sup_{u \in \mathbb{R}^m \setminus \overline{B_r(0)}} \left\vert f(u) \partial_\alpha (\phi_\delta * \eta)(u) \right\vert = 0.
		\end{equation*}
		\item\label{LemmaConvo2} For every $\delta > 0$ and $\alpha \in \mathbb{N}^m_0$ it holds that $\partial_\alpha (\phi_\delta * \eta)(u) du\big\vert_{\mathcal{B}(\mathbb{R}^m)} \in \mathcal{M}_\gamma(\mathbb{R}^m)$.
		\item\label{LemmaConvo3} For every $\delta > 0$ and $\alpha \in \mathbb{N}^m_0$ the map
		\begin{equation*}
			(\overline{C^0_b(\mathbb{R}^m)}^\gamma, \Vert \cdot \Vert_{C^0_{pol,\gamma}(\mathbb{R}^m)}) \ni f \quad \mapsto \quad \int_{\mathbb{R}^m} f(u) \partial_\alpha (\phi_\delta * \eta)(u) du \in \mathbb{R}
		\end{equation*}
		is a continuous linear functional.
		\item\label{LemmaConvo4} For every $\delta > 0$ it holds that
		\begin{equation*}
			\int_{\mathbb{R}^m} f(u) (\phi_\delta * \eta)(u) du = \int_{\mathbb{R}^m} \int_{\mathbb{R}^m} f(u+y) \eta(du) \phi_\delta(y) dy.
		\end{equation*}
		\item\label{LemmaConvo5} It holds that
		\begin{equation*}
			\lim_{\delta \rightarrow 0} \int_{\mathbb{R}^m} f(u) (\phi_\delta * \eta)(u) du = \int_{\mathbb{R}^m} f(u) \eta(du).
		\end{equation*}
	\end{enumerate}
\end{lemma}
\begin{proof}
	Fix some $\gamma \in (0,\infty)$, $\eta \in \mathcal{M}_\gamma(\mathbb{R}^m)$, $f \in \overline{C^0_b(\mathbb{R}^m)}^\gamma$, $\delta > 0$, and $\alpha \in \mathbb{N}^m_0$. For \ref{LemmaConvo0}, we first show that $\partial_\alpha \phi_\delta * \eta: \mathbb{R}^m \rightarrow \mathbb{R}$ is continuous. Indeed, we observe that for every $u,u_0,v \in \mathbb{R}^m$, it holds that
	\begin{equation}
		\label{EqLemmaConvoProof0a}
		\max\left( \left\vert \partial_\alpha \phi_\delta(u-v) \right\vert, \left\vert \partial_\alpha \phi_\delta(u_0-v) \right\vert \right) \leq C_{11} := \sup_{u_1 \in \mathbb{R}^m} \left\vert \partial_\alpha \phi_\delta(u_1) \right\vert < \infty.
	\end{equation}
	Then, the dominated convergence theorem (with \eqref{EqLemmaConvoProof0a} and that $\eta \in \mathcal{M}_\gamma(\mathbb{R}^m)$ is finite) implies that
	\begin{equation*}
		\lim_{u \rightarrow u_0} (\partial_\alpha \phi_\delta * \eta)(u) = \lim_{u \rightarrow u_0} \int_{\mathbb{R}^m} \partial_\alpha \phi_\delta(u-v) \eta(dv) = \int_{\mathbb{R}^m} \partial_\alpha \phi_\delta(u_0-v) \eta(dv) = (\partial_\alpha \phi_\delta * \eta)(u_0),
	\end{equation*}
	which shows that $\partial_\alpha \phi_\delta * \eta: \mathbb{R}^m \rightarrow \mathbb{R}$ is continuous. Moreover, for every fixed $\beta \in \mathbb{N}^m_0$ and $l = 1,...,m$ (with $e_l \in \mathbb{R}^m$ denoting the $l$-th unit vector of $\mathbb{R}^m$), we use the mean-value theorem to conclude for every $u,v \in \mathbb{R}^m$ and $h \in \mathbb{R}$ that
	\begin{equation}
		\label{EqLemmaConvoProof0b}
		\begin{aligned}
			& \max\left( \left\vert \frac{\partial_\beta \phi_\delta(u+he_l-v) - \partial_\beta \phi_\delta(u-v)}{h} \right\vert, \left\vert \partial_{\beta+e_l} \phi_\delta(u-v) \right\vert \right) \\
			& \quad\quad \leq C_{12} := \sup_{u_1 \in \mathbb{R}^m} \left\vert \partial_{\beta+e_l} \phi_\delta(u_1) \right\vert < \infty.
		\end{aligned}
	\end{equation}
	Then, the dominated convergence theorem (with \eqref{EqLemmaConvoProof0b} and that $\eta \in \mathcal{M}_\gamma(\mathbb{R}^m)$ is finite) implies that
	\begin{equation*}
		\begin{aligned}
			\partial_{e_l} (\partial_\alpha \phi_\delta * \eta)(u) & = \lim_{h \rightarrow 0} \frac{(\partial_\beta \phi_\delta * \eta)(u+he_l) - (\partial_\beta \phi_\delta * \eta)(u)}{h} \\
			& = \lim_{h \rightarrow 0} \int_{\mathbb{R}^m} \frac{\partial_\beta \phi_\delta(u+he_l-v) - \partial_\beta \phi_\delta(u-v)}{h} \eta(dv) \\
			& = \int_{\mathbb{R}^m} \partial_{\beta+e_l} \phi_\delta(u-v) \eta(dv) = (\partial_{\beta+e_l} \phi_\delta * \eta)(u).
		\end{aligned}
	\end{equation*}
	Hence, by induction on $\beta \in \mathbb{N}_0^m$, it follows that $\partial_\alpha (\phi_\delta * \eta)(u) = (\partial_\alpha \phi_\delta * \eta)(u)$ for any $u \in \mathbb{R}^m$. This together with the previous step shows \ref{LemmaConvo0}.
	
	For \ref{LemmaConvo1}, we use \ref{LemmaConvo0}, that $\supp(\phi_\delta) = B_\delta(0)$ implies $\supp(\partial_\alpha \phi_\delta) \subseteq B_\delta(0)$, the inequality $1+x+y \leq (1+x)(1+y)$ for any $x,y \geq 0$, that the constant $C_{13} := \sup_{y \in \mathbb{R}^m} \left\vert \partial_\alpha \phi_\delta(y) \right\vert > 0$ is finite, and that $\eta \in \mathcal{M}_\gamma(\mathbb{R}^m)$ to conclude that
	\begin{equation*}
		\begin{aligned}
			C_{14} & := \sup_{u \in \mathbb{R}^m} \big( (1+\Vert u \Vert)^\gamma \left\vert (\phi_\delta * \eta)(u) \right\vert \big) \leq \sup_{u \in \mathbb{R}^m} \int_{\mathbb{R}^m} (1+\Vert u \Vert)^\gamma \left\vert \partial_\alpha \phi_\delta(u-v) \right\vert \vert \eta \vert(dv) du \\
			& \leq \sup_{u \in \mathbb{R}^m} \int_{\mathbb{R}^m} (1+\underbrace{\Vert u-v \Vert}_{\leq \delta}+\Vert v \Vert)^\gamma \left\vert \partial_\alpha \phi_\delta(u-v) \right\vert \vert \eta \vert(dv) \leq C_{13} (1+\delta)^\gamma \int_{\mathbb{R}^m} (1+\Vert v \Vert)^\gamma \vert \eta \vert(dv) < \infty.
		\end{aligned}
	\end{equation*}
	Hence, by using this and that $f \in \overline{C^0_b(\mathbb{R}^m)}^\gamma$ together with Lemma~\ref{LemmaCkw}, it follows that
	\begin{equation*}
		\begin{aligned}
			\lim_{r \rightarrow \infty} \sup_{u \in \mathbb{R}^m \setminus \overline{B_r(0)}} \left\vert f(u) \partial_\alpha (\phi_\delta * \eta)(u) \right\vert & = \lim_{r \rightarrow \infty} \sup_{u \in \mathbb{R}^m \setminus \overline{B_r(0)}} \left( \frac{\vert f(u) \vert}{(1+\Vert u \Vert)^\gamma} (1+\Vert u \Vert)^\gamma \left\vert \partial_\alpha (\phi_\delta * \eta)(u) \right\vert \right) \\
			& = C_{14} \lim_{r \rightarrow \infty} \sup_{u \in \mathbb{R}^m \setminus \overline{B_r(0)}} \frac{\vert f(u) \vert}{(1+\Vert u \Vert)^\gamma} = 0,
		\end{aligned}
	\end{equation*}
	which shows \ref{LemmaConvo1}.
	
	For \ref{LemmaConvo2}, we first prove that $\partial_\alpha (\phi_\delta * \eta)(u) du\big\vert_{\mathcal{B}(\mathbb{R}^m)}: \mathcal{B}(\mathbb{R}^m) \rightarrow \mathbb{R}$ is a signed Radon measure. For this purpose, we denote its positive and negative part by $\eta_{\delta,\pm} := \pm \left( \partial_\alpha (\phi_\delta * \eta)(u) \right)_{\pm} du\big\vert_{\mathcal{B}(\mathbb{R}^m)}: \mathcal{B}(\mathbb{R}^m) \rightarrow [0,\infty]$ satisfying $\eta_{\delta,+}-\eta_{\delta,-} = \partial_\alpha (\phi_\delta * \eta)(u) du\big\vert_{\mathcal{B}(\mathbb{R}^m)}$, where $s_+ := \max(s,0)$ and $s_- := -\min(s,0)$, for any $s \in \mathbb{R}$. Moreover, we define the finite constant $C_{15} := \sup_{u \in \mathbb{R}^m} \left\vert \partial_\alpha \phi_\delta(u) \right\vert > 0$. Then, for every $u \in \mathbb{R}^m$, we choose a compact subset $K \subset \mathbb{R}^m$ with $u \in K$ and use that $\eta \in \mathcal{M}_\gamma(\mathbb{R}^m)$ is finite to conclude that
	\begin{equation*}
		\begin{aligned}
			\eta_{\delta,\pm}(K) & = \pm \int_{K} \left( \partial_\alpha (\phi_\delta * \eta)(u) \right)_{\pm} du \leq \bigg( \underbrace{\int_{K} du}_{=: \vert K \vert} \bigg) \sup_{u \in K} \left\vert (\partial_\alpha \phi_\delta * \eta)(u) \right\vert \\
			& \leq \vert K \vert \sup_{u \in K} \int_{\mathbb{R}^m} \left\vert \partial_\alpha \phi_\delta(u-v) \right\vert \vert \eta \vert(dv) \leq C_{15} \vert K \vert \vert \eta \vert(\mathbb{R}^m) < \infty.
		\end{aligned}
	\end{equation*}	
	This shows that both measures $\eta_{\delta,\pm}: \mathcal{B}(\mathbb{R}^m) \rightarrow [0,\infty]$ are locally finite. In addition, it holds for every $B \in \mathcal{B}(\mathbb{R}^m)$ that
	\begin{equation*}
		\begin{aligned}
			\eta_{\delta,\pm}(B) & = \pm \int_B \left( \partial_\alpha (\phi_\delta * \eta)(u) \right)_{\pm} du \\
			& = \inf\left\lbrace \pm \int_U \left( \partial_\alpha (\phi_\delta * \eta)(u) \right)_{\pm} du: U \subseteq \mathbb{R}^m \text{ open with } B \subseteq U \right\rbrace \\
			& = \inf\left\lbrace \eta_{\delta,\pm}(U): U \subseteq \mathbb{R}^m \text{ open with } B \subseteq U \right\rbrace,
		\end{aligned}
	\end{equation*}
	which shows that both measures $\eta_{\delta,\pm}: \mathcal{B}(\mathbb{R}^m) \rightarrow [0,\infty]$ are outer regular. Moreover, it holds for every $B \in \mathcal{B}(\mathbb{R}^m)$ that
	\begin{equation*}
		\begin{aligned}
			\eta_{\delta,\pm}(B) & = \pm \int_B \left( \partial_\alpha (\phi_\delta * \eta)(u) \right)_{\pm} du \\
			& = \sup\left\lbrace \pm \int_K \left( \partial_\alpha (\phi_\delta * \eta)(u) \right)_{\pm} du: K \subset B \text{ relatively compact} \right\rbrace \\
			& = \sup\left\lbrace \eta_{\delta,\pm}(K): K \subset B \text{ relatively compact} \right\rbrace,
		\end{aligned}
	\end{equation*}
	which shows that both measures $\eta_{\delta,\pm}: \mathcal{B}(\mathbb{R}^m) \rightarrow [0,\infty]$ are inner regular. Hence, both measures $\eta_{\delta,\pm}: \mathcal{B}(\mathbb{R}^m) \rightarrow [0,\infty]$ are Radon measures and $\partial_\alpha (\phi_\delta * \eta)(u) du\big\vert_{\mathcal{B}(\mathbb{R}^m)} = \eta_{\delta,+}-\eta_{\delta,-}: \mathcal{B}(\mathbb{R}^m) \rightarrow [0,\infty]$ is thus a signed Radon measure. Furthermore, by using the triangle inequality, that $\supp(\phi_\delta) = B_\delta(0)$ implies $\supp(\partial_\alpha \phi_\delta) \subseteq B_\delta(0)$, the inequality $1+x+y \leq (1+x)(1+y)$ for any $x,y \geq 0$, the substitution $y \mapsto u-v$ together with $\Vert \partial_\alpha \phi_\delta \Vert_{L^1(\mathbb{R}^m,\mathcal{L}(\mathbb{R}^m),du)} < \infty$, and that $\eta \in \mathcal{M}_\gamma(\mathbb{R}^m)$, we have
	\begin{equation*}
		\begin{aligned}
			\int_{\mathbb{R}^m} (1+\Vert u \Vert)^\gamma \left\vert \partial_\alpha (\phi_\delta * \eta)(u) \right\vert du & \leq \int_{\mathbb{R}^m} \int_{\mathbb{R}^m} (1+\Vert u \Vert)^\gamma \vert \partial_\alpha \phi_\delta(u-v) \vert du \vert \eta \vert(dv) \\
			& \leq \int_{\mathbb{R}^m} \int_{\mathbb{R}^m} (1+\underbrace{\Vert u-v \Vert}_{\leq \delta}+\Vert v \Vert)^\gamma \left\vert \partial_\alpha \phi_\delta(u-v) \right\vert du \vert \eta \vert(dv) \\
			& \leq (1+\delta)^\gamma \left( \sup_{v \in \mathbb{R}^m} \int_{\mathbb{R}^m} \left\vert \partial_\alpha \phi_\delta(u-v) \right\vert du \right) \left( \int_{\mathbb{R}^m} (1+\Vert v \Vert)^\gamma \vert \eta \vert(dv) \right) \\
			& \leq (1+\delta)^\gamma \Vert \partial_\alpha \phi_\delta \Vert_{L^1(\mathbb{R}^m,\mathcal{L}(\mathbb{R}^m),du)} \left( \int_{\mathbb{R}^m} (1+\Vert v \Vert)^\gamma \vert \eta \vert(dv) \right) < \infty. \\
		\end{aligned}
	\end{equation*}
	This shows that $\partial_\alpha (\phi_\delta * \eta)(u) du\big\vert_{\mathcal{B}(\mathbb{R}^m)} \in \mathcal{M}_\gamma(\mathbb{R}^m)$ is a finite signed Radon measure.	
	
	For \ref{LemmaConvo3}, we use \ref{LemmaConvo2} to conclude that the constant $C_{16} := \int_{\mathbb{R}^m} (1+\Vert u \Vert)^\gamma \left\vert (\phi_\delta * \eta)(u) \right\vert du > 0$ is finite. Then, it follows for every $f \in \overline{C^0_b(\mathbb{R}^m)}^\gamma$ that
	\begin{equation*}
		\begin{aligned}
			\left\vert \int_{\mathbb{R}^m} f(u) \partial_\alpha (\phi_\delta * \eta)(u) du \right\vert & \leq \left( \sup_{u \in \mathbb{R}^m} \frac{\vert f(u) \vert}{(1+\Vert u \Vert)^\gamma} \right) \int_{\mathbb{R}^m} (1 + \Vert u \Vert)^\gamma \left\vert \partial_\alpha (\phi_\delta * \eta)(u) \right\vert du \\
			& = C_{16} \Vert f \Vert_{C^0_{pol,\gamma}(\mathbb{R}^m)},
		\end{aligned}
	\end{equation*}
	which shows that $\overline{C^0_b(\mathbb{R}^m)}^\gamma \ni f \mapsto \int_{\mathbb{R}^m} f(u) \partial_\alpha (\phi_\delta * \eta)(u) du \in \mathbb{R}$ is a continuous linear functional.
	
	For \ref{LemmaConvo4}, we use the substitution $u \mapsto v+y$ to conclude that
	\begin{equation*}
		\begin{aligned}
			\int_{\mathbb{R}^m} f(u) (\phi_\delta * \eta)(u) du & = \int_{\mathbb{R}^m} \int_{\mathbb{R}^m} f(u) \phi_\delta(u-v) \eta(dv) du \\
			& = \int_{\mathbb{R}^m} \int_{\mathbb{R}^m} f(v+y) \eta(dv) \phi_\delta(y) dy.
		\end{aligned}
	\end{equation*}
	
	For \ref{LemmaConvo5}, we define for every $\delta \in (0,1)$ the function $\mathbb{R}^m \ni u \mapsto (\phi_\delta * f)(u) := \int_{\mathbb{R}^m} \phi_\delta(u-v) f(v) dv \in \mathbb{R}$. Then, by using the triangle inequality, that $\supp(\phi_\delta) = B_\delta(0)$, the substitution $y \mapsto u-v$ together with $\int_{\mathbb{R}^m} \vert \phi_\delta(y) \vert dy = \Vert \phi_\delta \Vert_{L^1(\mathbb{R}^m,\mathcal{L}(\mathbb{R}^m),du)} = \Vert \phi \Vert_{L^1(\mathbb{R}^m,\mathcal{L}(\mathbb{R}^m),du)} = 1$, the inequality $1+x+y \leq (1+x)(1+y)$ for any $x,y \geq 0$, and that $f \in \overline{C^0_b(\mathbb{R}^m)}^\gamma$, it follows for every $u \in \mathbb{R}^m$ that
	\begin{equation}
		\label{EqLemmaConvoProof1}
		\begin{aligned}
			\left\vert (\phi_\delta * f)(u) \right\vert & \leq \int_{\mathbb{R}^m} \vert \phi_\delta(u-v) \vert \frac{\vert f(v) \vert}{(1+\Vert v \Vert)^\gamma} (1+\Vert v \Vert)^\gamma dv \\
			& \leq \int_{\mathbb{R}^m} \vert \phi_\delta(u-v) \vert \frac{\vert f(v) \vert}{(1+\Vert v \Vert)^\gamma} (1+\Vert u \Vert + \underbrace{\Vert u-v \Vert}_{\leq \delta})^\gamma dv \\
			& \leq \left( \int_{\mathbb{R}^m} \vert \phi_\delta(u-v) \vert dv \right) \left( \sup_{v \in \mathbb{R}^m} \frac{\vert f(v) \vert}{(1+\Vert v \Vert)^\gamma} \right) (1+\Vert u \Vert+\delta)^\gamma \\
			& \leq \left( \int_{\mathbb{R}^m} \vert \phi_\delta(y) \vert dy \right) \Vert f \Vert_{C^0_{pol,\gamma}(\mathbb{R}^m)} (1+\delta)^\gamma (1+\Vert u \Vert)^\gamma \\
			& \leq 2^\gamma \Vert f \Vert_{C^0_{pol,\gamma}(\mathbb{R}^m)} (1+\Vert u \Vert)^\gamma.
		\end{aligned}
	\end{equation}
	Moreover, by using that $f \in \overline{C^0_b(\mathbb{R}^m)}^\gamma$, we conclude for every $u \in \mathbb{R}^m$ that
	\begin{equation}
		\label{EqLemmaConvoProof2}
		\vert f(u) \vert \leq \left( \sup_{u \in \mathbb{R}^m} \frac{\vert f(u) \vert}{(1+\Vert u \Vert)^\gamma} \right) (1+\Vert u \Vert)^\gamma \leq \Vert f \Vert_{C^0_{pol,\gamma}(\mathbb{R}^m)} (1+\Vert u \Vert)^\gamma.
	\end{equation} 
	Hence, by using \ref{LemmaConvo4}, Fubini's theorem, the substitution $u \mapsto v+y$, and the dominated convergence theorem (with \eqref{EqLemmaConvoProof1}, \eqref{EqLemmaConvoProof2}, $(1+\Vert u \Vert)^\gamma \in L^1(\mathbb{R}^m,\mathcal{B}(\mathbb{R}^m),\vert \eta \vert)$ as $\eta \in \mathcal{M}_\gamma(\mathbb{R}^m)$, and \cite[Theorem~C.7]{evans10}, i.e.~that $\phi_\delta * f: \mathbb{R}^m \rightarrow \mathbb{R}$ converges a.e.~to $f: \mathbb{R}^m \rightarrow \mathbb{R}$, as $\delta \rightarrow 0$), it follows that
	\begin{equation*}
		\begin{aligned}
			\lim_{\delta \rightarrow 0} \int_{\mathbb{R}^m} f(u) (\phi_\delta * \eta)(u) du & = \lim_{\delta \rightarrow 0} \int_{\mathbb{R}^m} \int_{\mathbb{R}^m} f(v+y) \eta(dv) \phi_\delta(y) dy \\
			& = \lim_{\delta \rightarrow 0} \int_{\mathbb{R}^m} \left( \int_{\mathbb{R}^m} f(v+y) \phi_\delta(y) dy \right) \eta(dv) \\
			& = \lim_{\delta \rightarrow 0} \int_{\mathbb{R}^m} \left( \int_{\mathbb{R}^m} \phi(v-u) f(u) du \right) \eta(dv) \\
			& = \lim_{\delta \rightarrow 0} \int_{\mathbb{R}^m} (\phi_\delta * f)(v) \eta(dv) \\
			& = \int_{\mathbb{R}^m} f(v) \eta(dv),
		\end{aligned}
	\end{equation*}
	which completes the proof.
\end{proof}

Finally, we provide the proof of Theorem~\ref{ThmUAT}, i.e.~the universal approximation property of neural networks $\mathcal{NN}^\rho_{U,d}$ within any function space $(X,\Vert \cdot \Vert_X)$ satisfying Assumption~\ref{AssApprox} with $k \in \mathbb{N}_0$, $U \subseteq \mathbb{R}^m$ (open, if $k \geq 1$), and $\gamma \in (0,\infty)$, where $\rho \in \overline{C^k_b(\mathbb{R})}^\gamma$ is the activation function.

\begin{proof}[Proof of Theorem~\ref{ThmUAT}]
	First, we apply Lemma~\ref{LemmaProp} to conclude that $\mathcal{NN}^\rho_{\mathbb{R}^m,d} \subseteq \overline{C^k_b(\mathbb{R}^m;\mathbb{R}^d)}^\gamma$ and that $\mathcal{NN}^\rho_{U,d} \subseteq X$. Now, we assume by contradiction that $\mathcal{NN}^\rho_{U,d}$ is not dense in $X$. Then, by using that $(X,\Vert \cdot \Vert_X)$ satisfies Assumption~\ref{AssApprox}, i.e.~that the restriction map \eqref{EqAssApprox1} is a continuous dense embedding, it follows from Remark~\ref{RemApprox} that $\mathcal{NN}^\rho_{\mathbb{R}^m,d}$ cannot be dense in $\overline{C^k_b(\mathbb{R}^m;\mathbb{R}^d)}^\gamma$. Hence, by applying the Hahn-Banach theorem, there exists a non-zero continuous linear functional $l: \overline{C^k_b(\mathbb{R}^m;\mathbb{R}^d)}^\gamma \rightarrow \mathbb{R}$ such that for every $\varphi \in \mathcal{NN}^\rho_{\mathbb{R}^m,d}$ it holds that $l(\varphi) = 0$.
	
	Next, we use the Riesz representation result in Proposition~\ref{PropRiesz} to conclude that there exist some signed Radon measures $(\eta_{\alpha,i})_{\alpha \in \mathbb{N}^m_{0,k}, \, i = 1,...,d} \in \mathcal{M}_\gamma(\mathbb{R}^m)$ such that for every $f \in \overline{C^k_b(\mathbb{R}^m;\mathbb{R}^d)}^\gamma$ it holds that
	\begin{equation*}
		l(f) = \sum_{\alpha \in \mathbb{N}^m_{0,k}} \sum_{i=1}^d \int_{\mathbb{R}^m} \partial_\alpha f_i(u) \eta_{\alpha,i}(du).
	\end{equation*}
	Since $l(\varphi) = 0$ for any $\varphi \in \mathcal{NN}^\rho_{\mathbb{R}^m,d}$, it follows for every $a \in \mathbb{R}^m$, $b \in \mathbb{R}$, and $i = 1,...,d$ that
	\begin{equation}
		\label{EqThmUATProof1}
		l\left( e_i \rho\left( \lambda a^\top \cdot - b \right) \right) = \sum_{\alpha \in \mathbb{N}^m_{0,k}} \int_{\mathbb{R}^m} \rho^{(\vert \alpha \vert)}\left( a^\top u - b \right) a^\alpha \eta_{\alpha,i}(du) = 0,
	\end{equation}
	where $e_i \rho\big( \lambda a^\top \cdot - b \big)$ denotes the function $\mathbb{R}^m \ni u \mapsto e_i \rho\big( \lambda a^\top u - b \big) \in \mathbb{R}^d$ with $e_i \in \mathbb{R}^d$ being the $i$-th unit vector of $\mathbb{R}^d$, and where $a^\alpha := \prod_{l=1}^m a_l^{\alpha_l}$ for $a := (a_1,...,a_m) \in \mathbb{R}^m$ and $\alpha := (\alpha_1,...,\alpha_m) \in \mathbb{N}^m_{0,k}$.
	
	Now, we define for every fixed $\delta > 0$ the linear map $l_\delta: \overline{C^k_b(\mathbb{R}^m;\mathbb{R}^d)}^\gamma \rightarrow \mathbb{R}$ by
	\begin{equation*}
		\overline{C^k_b(\mathbb{R}^m;\mathbb{R}^d)}^\gamma \ni f \quad \mapsto \quad l_\delta(f) := \sum_{\alpha \in \mathbb{N}^m_{0,k}} \sum_{i=1}^d \int_{\mathbb{R}^m} \partial_\alpha f_i(u) (\phi_\delta * \eta)(u) du \in \mathbb{R}.
	\end{equation*}
	Then, Lemma~\ref{LemmaConvo}~\ref{LemmaConvo3} shows that $l_\delta: \overline{C^k_b(\mathbb{R}^m;\mathbb{R}^d)}^\gamma \rightarrow \mathbb{R}$ is a continuous linear functional as it is a finite sum of the continuous linear functionals $\overline{C^k_b(\mathbb{R}^m;\mathbb{R}^d)}^\gamma \ni f \mapsto \int_{\mathbb{R}^m} \partial_\alpha f_i(u) (\phi_\delta * \eta)(u) du \in \mathbb{R}$ taken over $\alpha \in \mathbb{N}^m_{0,k}$ and $i = 1,...,d$. Moreover, for every fixed $i = 1,...,d$, we define
	\begin{equation*}
		\mathbb{R}^m \ni u \mapsto h_{\delta,i}(u) := \sum_{\alpha \in \mathbb{N}^m_{0,k}} (-1)^{\vert \alpha \vert} \partial_\alpha\left( \phi_\delta * \eta_{\alpha,i} \right)(u) \in \mathbb{R},
	\end{equation*}
	which satisfies $h_{\delta,i}(u) du \in \mathcal{M}_\gamma(\mathbb{R}^m)$ as it is a finite linear combination of finite signed Radon measures $\partial_\alpha\left( \phi_\delta * \eta_{\alpha,i} \right)(u) du \in \mathcal{M}_\gamma(\mathbb{R}^m)$ taken over $\alpha \in \mathbb{N}^m_{0,k}$ (see Lemma~\ref{LemmaConvo}~\ref{LemmaConvo2}). Hence, integration by parts together with Lemma~\ref{LemmaConvo}~\ref{LemmaConvo1} shows that
	\begin{equation*}
		\begin{aligned}
			l_\delta(f) & = \sum_{\alpha \in \mathbb{N}^m_{0,k}} \sum_{i=1}^d \int_{\mathbb{R}^m} \partial_\alpha f_i(u) \left( \phi_\delta * \eta_{\alpha,i} \right)(u) du \\
			& = \sum_{\alpha \in \mathbb{N}^m_{0,k}} \sum_{i=1}^d (-1)^{\vert \alpha \vert} \int_{\mathbb{R}^m} f_i(u) \partial_\alpha \left( \phi_\delta * \eta_{\alpha,i} \right)(u) du \\
			& = \sum_{i=1}^d \int_{\mathbb{R}^m} f_i(u) h_{\delta,i}(u) du.
		\end{aligned}
	\end{equation*}
	Thus, by using this, Lemma~\ref{LemmaConvo}~\ref{LemmaConvo4}, and \eqref{EqThmUATProof1} (with $b - a^\top y \in \mathbb{R}$ instead of $b \in \mathbb{R}$), it follows for every $a \in \mathbb{R}^m$, $b \in \mathbb{R}$, and $i = 1,...,d$ that
	\begin{equation*}
		\begin{aligned}
			\int_{\mathbb{R}^m} \rho\left( a^\top u - b \right) h_{\delta,i}(u) du & = \sum_{\alpha \in \mathbb{N}^m_{0,k}} \int_{\mathbb{R}^m} \rho^{(\vert \alpha \vert)}\left( a^\top u - b \right) a^\alpha \left( \phi_\delta * \eta_{\alpha,i} \right)(u) du \\
			& = \sum_{\alpha \in \mathbb{N}^m_{0,k}} \int_{\mathbb{R}^m} \rho^{(\vert \alpha \vert)}\left( a^\top (u+y) - b \right) a^\alpha \eta_{\alpha,i}(du) \phi_\delta(y) du \\
			& = \int_{\mathbb{R}^m} \underbrace{l\left( e_i \rho\left( a^\top \cdot - \left( b - a^\top y \right) \right) \right)}_{=0} \phi_\delta(y) dy = 0.
		\end{aligned}
	\end{equation*}
	Now, for every $i = 1,...,d$, we apply Proposition~\ref{PropDiscr} with $h_{\delta,i}(u) du \in \mathcal{M}_\gamma(\mathbb{R}^m)$ to conclude that $h_{\delta,i}(u) du = 0 \in \mathcal{M}_\gamma(\mathbb{R}^m)$, and thus $h_{\delta,i}(u) = 0$ for a.e.~$u \in \mathbb{R}^m$. Hence, it follows for every $f \in \overline{C^k_b(\mathbb{R}^m;\mathbb{R}^d)}^\gamma$ that
	\begin{equation*}
		l_\delta(f) = \sum_{i=1}^d \int_{\mathbb{R}^m} f_i(u) h_{\delta,i}(u) du = 0,
	\end{equation*}
	which shows that $l_\delta: C^k_b(\mathbb{R}^m;\mathbb{R}^d) \rightarrow \mathbb{R}$ vanishes everywhere on $C^k_b(\mathbb{R}^m;\mathbb{R}^d)$.
	
	Finally, we use Lemma~\ref{LemmaConvo}~\ref{LemmaConvo5} to conclude for every $f \in \overline{C^k_b(\mathbb{R}^m;\mathbb{R}^d)}^\gamma$ that
	\begin{equation*}
		\begin{aligned}
			l(f) & = \sum_{\alpha \in \mathbb{N}^m_{0,k}} \sum_{i=1}^d \int_{\mathbb{R}^m} \partial_\alpha f_i(u) \eta_{\alpha,i}(du) = \lim_{\delta \rightarrow \infty} \sum_{\alpha \in \mathbb{N}^m_{0,k}} \sum_{i=1}^d \int_{\mathbb{R}^m} f_i(u) (\phi_\delta * \eta)(u) du = \lim_{\delta \rightarrow \infty} l_\delta(f) = 0,
		\end{aligned}
	\end{equation*}
	which shows that $l:\overline{C^k_b(\mathbb{R}^m;\mathbb{R}^d)}^\gamma \rightarrow \mathbb{R}$ vanishes everywhere. This however contradicts the assumption that $l: \overline{C^k_b(\mathbb{R}^m;\mathbb{R}^d)}^\gamma \rightarrow \mathbb{R}$ is non-zero. Hence, $\mathcal{NN}^\rho_{U,d}$ is dense in $X$.
\end{proof}

\subsubsection{Proof of Example~\ref{ExApprox}.}
\label{AppApproxNonPoly}

For the proof of Example~\ref{ExApprox}~\ref{ExApproxWkpw}, we first generalize the approximation result for compactly supported smooth functions in unweighted Sobolev spaces (see \cite[Theorem~3.18]{adams75}) to weighted Sobolev spaces $W^{k,p}(U,\mathcal{L}(U),w;\mathbb{R}^d)$ introduced in Footnote~\ref{FootnoteWkpw}.

\begin{proposition}[Approximation in Weighted Sobolev Spaces]
	\label{PropApproxWkpw}
	For $k \in \mathbb{N}$, $p \in [1,\infty)$, and $U \subseteq \mathbb{R}^m$ open and having the segment property, let $w: U \rightarrow [0,\infty)$ be a bounded weight such that for every bounded subset $B \subseteq U$ it holds that $\inf_{u \in B} w(u) > 0$. Then, $\left\lbrace f\vert_U: U \rightarrow \mathbb{R}^d: f \in C^\infty_c(\mathbb{R}^m;\mathbb{R}^d) \right\rbrace$ is dense in $W^{k,p}(U,\mathcal{L}(U),w;\mathbb{R}^d)$.
\end{proposition}
\begin{proof}
	First, we follow \cite[Theorem~3.18]{adams75} to show that every fixed function $f \in W^{k,p}(U,\mathcal{L}(U),w;\mathbb{R}^d)$ can be approximated by elements from the set $\left\lbrace f\vert_U: U \rightarrow \mathbb{R}^d: f \in C^\infty_c(\mathbb{R}^m;\mathbb{R}^d) \right\rbrace$ with respect to $\Vert \cdot \Vert_{W^{k,p}(U,\mathcal{L}(U),w;\mathbb{R}^d)}$. To this end, we choose some $h \in C^\infty_c(\mathbb{R}^m)$ which satisfies $h(u) = 1$ for all $u \in \overline{B_1(0)}$, $h(u) = 0$ for all $u \in \mathbb{R}^m \setminus B_2(0)$, and for which there exists a constant $C_h > 0$ such that for every $\alpha \in \mathbb{N}^m_{0,k}$ and $u \in \mathbb{R}^m$ it holds that $\vert \partial_\alpha h(u) \vert \leq C_h$. In addition, we define for every fixed $r > 1$ the functions $\mathbb{R}^m \ni u \mapsto h_r(u) := h(u/r) \in \mathbb{R}$ and $U \ni u \mapsto f_r(u) := f(u) h_r(u) \in \mathbb{R}^d$, which both have bounded support. Then, by using the Leibniz product rule together with the triangle inequality, that $\vert \partial_\alpha h_r(u) \vert = \left\vert \partial_\alpha h(u/r) \right\vert r^{-\vert \alpha \vert} \leq C_h$ for any $\alpha \in \mathbb{N}^m_{0,k}$ and $u \in \mathbb{R}^m$, and the inequality \eqref{EqLemmaCkwProof2}, it follows for every $\alpha \in \mathbb{N}^m_{0,k}$ and $u \in U$ that
	\begin{equation*}
		\begin{aligned}
			\Vert \partial_\alpha f_r(u) \Vert^p & \leq \left( \sum_{\beta_1,\beta_2 \in \mathbb{N}^m_0 \atop \beta_1+\beta_2=\alpha} \frac{\alpha!}{\beta_1! \beta_2!} \vert \partial_{\beta_1} h_r(u) \vert \Vert \partial_{\beta_2} f(u) \Vert \right)^p \\
			& \leq 2^{kp} C_h^p \max_{\beta_2 \in \mathbb{N}^m_{0,k}} \Vert \partial_{\beta_2} f(u) \Vert^p \\
			& \leq 2^{kp} C_h^p \sum_{\beta_2 \in \mathbb{N}^m_{0,k}} \Vert \partial_{\beta_2} f(u) \Vert^p.
		\end{aligned} 
	\end{equation*}
	Hence, by using this, it follows for every $V \in \mathcal{L}(U)$ that
	\begin{equation}
		\label{EqPropApproxWkpwProof1}
		\begin{aligned}
			\Vert f_r \Vert_{W^{k,p}(V,\mathcal{L}(V),w;\mathbb{R}^d)} & = \left( \sum_{\alpha \in \mathbb{N}^m_{0,k}} \int_V \Vert \partial_\alpha f_r(u) \Vert^p w(u) du \right)^\frac{1}{p} \\
			& \leq \left\vert \mathbb{N}^m_{0,k} \right\vert^\frac{1}{p} \left( \max_{\alpha \in \mathbb{N}^m_{0,k}} \int_V \Vert \partial_\alpha f_r(u) \Vert^p w(u) du \right)^\frac{1}{p} \\
			& \leq 2^k C_h \left\vert \mathbb{N}^m_{0,k} \right\vert^\frac{1}{p} \left( \sum_{\beta_2 \in \mathbb{N}^m_{0,k}} \int_V \Vert \partial_{\beta_2} f(u) \Vert^p w(u) du \right)^\frac{1}{p} \\
			& \leq 2^k C_h \left\vert \mathbb{N}^m_{0,k} \right\vert^\frac{1}{p} \Vert f \Vert_{W^{k,p}(V,\mathcal{L}(V),w;\mathbb{R}^d)} < \infty.
		\end{aligned}
	\end{equation}
	Thus, by taking $V := U$ in \eqref{EqPropApproxWkpwProof1}, we conclude that $f_r \in W^{k,p}(U,\mathcal{L}(U),w;\mathbb{R}^d)$. Similarly, by using the triangle inequality, that $\partial_\alpha f_r(u) = \partial_\alpha (f(u) h_r(u))  = \partial_\alpha f(u)$ for any $\alpha \in \mathbb{N}^m_{0,k}$ and $u \in U \cap \overline{B_r(0)}$ (as $h_r(u) = 1$ for any $u \in \overline{B_r(0)}$), and \eqref{EqPropApproxWkpwProof1} with $V := U \setminus \overline{B_r(0)}$, it follows that
	\begin{equation*}
		\begin{aligned}
			\Vert f - f_r \Vert_{W^{k,p}(U,\mathcal{L}(U),w;\mathbb{R}^d)} & \leq \underbrace{\Vert f - f_r \Vert_{W^{k,p}(U \cap \overline{B_r(0)},\mathcal{L}(U \cap \overline{B_r(0)}),w;\mathbb{R}^d)}}_{=0} + \Vert f - f_r \Vert_{W^{k,p}(U \setminus \overline{B_r(0)},\mathcal{L}(U \setminus \overline{B_r(0)}),w;\mathbb{R}^d)} \\
			& \leq \Vert f \Vert_{W^{k,p}(U \setminus \overline{B_r(0)},\mathcal{L}(U \setminus \overline{B_r(0)}),w;\mathbb{R}^d)} + \Vert f_r \Vert_{W^{k,p}(U \setminus \overline{B_r(0)},\mathcal{L}(U \setminus \overline{B_r(0)}),w;\mathbb{R}^d)} \\
			& \leq \left( 1 + 2^k C_h \left\vert \mathbb{N}^m_{0,k} \right\vert^\frac{1}{p} \right) \Vert f \Vert_{W^{k,p}(U \setminus \overline{B_r(0)},\mathcal{L}(U \setminus \overline{B_r(0)}),w;\mathbb{R}^d)}.
		\end{aligned}
	\end{equation*}
	Since the right-hand side tends to zero, as $r \rightarrow \infty$, this shows that $f \in W^{k,p}(U,\mathcal{L}(U),w;\mathbb{R}^d)$ can be approximated by elements of $\left\lbrace f \in W^{k,p}(U,\mathcal{L}(U),w;\mathbb{R}^d): \supp(f) \subseteq U \text{ is bounded} \right\rbrace$ with respect to $\Vert \cdot \Vert_{W^{k,p}(U,\mathcal{L}(U),w;\mathbb{R}^d)}$. Hence, we only need to show the approximation of the latter by elements from $\left\lbrace f\vert_U: U \rightarrow \mathbb{R}^d: f \in C^\infty_c(\mathbb{R}^m;\mathbb{R}^d) \right\rbrace$ with respect to $\Vert \cdot \Vert_{W^{k,p}(U,\mathcal{L}(U),w;\mathbb{R}^d)}$.
	
	Therefore, we now fix some $f \in W^{k,p}(U,\mathcal{L}(U),w;\mathbb{R}^d)$ with bounded support $\supp(f) \subseteq U$ and some $\varepsilon > 0$. Moreover, by recalling that $w: U \rightarrow [0,\infty)$ is bounded, we can define the finite constant $C_w := \sup_{u \in U} w(u) > 0$. Then, by using that $f(u) = 0$ for any $u \in U \setminus \supp(f)$, thus $\partial_\alpha f(u) = 0$ for any $\alpha \in \mathbb{N}^m_{0,k}$ and $u \in U \setminus \supp(f)$, and the assumption that $C_{f,w} := \inf_{u \in \supp(f)} w(u) > 0$, we have
	\begin{equation*}
		\begin{aligned}
			& \Vert f \Vert_{W^{k,p}(U,\mathcal{L}(U),du;\mathbb{R}^d)} = \left( \sum_{\alpha \in \mathbb{N}^m_{0,k}} \int_U \Vert \partial_\alpha f(u) \Vert^p du \right)^\frac{1}{p} = \left( \sum_{\alpha \in \mathbb{N}^m_{0,k}} \int_{\supp(f)} \Vert \partial_\alpha f(u) \Vert^p du \right)^\frac{1}{p} \\
			& \quad\quad \leq C_{f,w}^{-1} \left( \sum_{\alpha \in \mathbb{N}^m_{0,k}} \int_{\supp(f)} \Vert \partial_\alpha f(u) \Vert^p w(u) du \right)^\frac{1}{p} = C_{f,w}^{-1} \left( \sum_{\alpha \in \mathbb{N}^m_{0,k}} \int_U \Vert \partial_\alpha f(u) \Vert^p w(u) du \right)^\frac{1}{p} \\
			& \quad\quad = C_{f,w}^{-1} \Vert f \Vert_{W^{k,p}(U,\mathcal{L}(U),w;\mathbb{R}^d)} < \infty.
		\end{aligned}
	\end{equation*}
	This shows that $f \in W^{k,p}(U,\mathcal{L}(U),du;\mathbb{R}^d)$. Hence, by applying \cite[Theorem~3.18]{adams75} (with $U \subseteq \mathbb{R}^m$ having the segment property) componentwise, there exists some $g \in C^\infty_c(\mathbb{R}^m;\mathbb{R}^d)$ such that
	\begin{equation*}
		\Vert f - g \Vert_{W^{k,p}(U,\mathcal{L}(U),du,\mathbb{R}^d)} = \left( \sum_{\alpha \in \mathbb{N}^m_{0,k}} \int_U \Vert \partial_\alpha f(u) - \partial_\alpha g(u) \Vert^p du \right)^\frac{1}{p} < \frac{\varepsilon}{C_w}.
	\end{equation*}
	Thus, by using that $w: U \rightarrow [0,\infty)$ is bounded with $C_w := \sup_{u \in U} w(u) < \infty$, it follows that
	\begin{equation*}
		\begin{aligned}
			\Vert f - g \Vert_{W^{k,p}(U,\mathcal{L}(U),du,\mathbb{R}^d)} & = \left( \sum_{\alpha \in \mathbb{N}^m_{0,k}} \int_U \Vert \partial_\alpha f(u) - \partial_\alpha g(u) \Vert^p w(u) du \right)^\frac{1}{p} \\
			& \leq C_w \left( \sum_{\alpha \in \mathbb{N}^m_{0,k}} \int_U \Vert \partial_\alpha f(u) - \partial_\alpha g(u) \Vert^p du \right)^\frac{1}{p} \\
			& < C_w \frac{\varepsilon}{C_w} = \varepsilon.
		\end{aligned}
	\end{equation*}
	Since $f \in W^{k,p}(U,\mathcal{L}(U),w;\mathbb{R}^d)$ with bounded support $\supp(f) \subseteq U$ and $\varepsilon > 0$ were chosen arbitrarily, it follows together with the first step that $\left\lbrace f\vert_U: U \rightarrow \mathbb{R}^d: f \in C^\infty_c(\mathbb{R}^m;\mathbb{R}^d) \right\rbrace$ is dense in $W^{k,p}(U,\mathcal{L}(U),w;\mathbb{R}^d)$.
\end{proof}

\begin{proof}[Proof of Example~\ref{ExApprox}]
	For \ref{ExApproxCkb}, we use that $U \subset \mathbb{R}^m$ is bounded to define the finite constant $C_{21} := \sup_{u \in U} (1+\Vert u \Vert)^\gamma$. Then, it follows for every $f \in C^0_b(\mathbb{R}^m;\mathbb{R}^d)$ that
	\begin{equation*}
		\begin{aligned}
			\Vert f\vert_U \Vert_{C^k_b(U;\mathbb{R}^d)} & = \max_{\alpha \in \mathbb{N}^m_{0,k}} \sup_{u \in U} \Vert \partial_\alpha f(u) \Vert \\
			& \leq \left( \sup_{u \in U} (1+\Vert u \Vert)^\gamma \right) \max_{\alpha \in \mathbb{N}^m_{0,k}} \sup_{u \in U} \frac{\Vert \partial_\alpha f(u) \Vert}{(1+\Vert u \Vert)^\gamma} \\
			& \leq C_{21} \Vert f \Vert_{C^k_{pol,\gamma}(\mathbb{R}^m;\mathbb{R}^d)}.
		\end{aligned}
	\end{equation*}
	Moreover, by using that $\left\lbrace f\vert_U: f \in C^k_b(\mathbb{R}^m;\mathbb{R}^d) \right\rbrace = C^k_b(U;\mathbb{R}^d)$, the image $\left\lbrace f\vert_U: f \in C^k_b(\mathbb{R}^m;\mathbb{R}^d) \right\rbrace$ of the continuous embedding \eqref{EqAssApprox1} is dense in $C^k_b(U;\mathbb{R}^d)$.
	
	For \ref{ExApproxCkw}, the restriction map \eqref{EqAssApprox1} is by definition continuous. Moreover, by using that $\overline{C^k_b(U;\mathbb{R}^d)}^\gamma$ is defined as the closure of $C^k_b(U;\mathbb{R}^d)$ with respect to $\Vert \cdot \Vert_{C^k_{pol,\gamma}(U;\mathbb{R}^d)}$, the image $\left\lbrace g\vert_U: g \in C^k_b(\mathbb{R}^m;\mathbb{R}^d) \right\rbrace = C^k_b(U;\mathbb{R}^d)$ of the continuous embedding \eqref{EqAssApprox1} is dense in $\overline{C^k_b(U;\mathbb{R}^d)}^\gamma$.
	
	For \ref{ExApproxLp}, we first recall that $k = 0$. Then, we use that $f \in C^0_b(\mathbb{R}^m;\mathbb{R}^d)$ is continuous to conclude that its restriction $f\vert_U: U \rightarrow \mathbb{R}^d$ is $\mathcal{B}(U)/\mathcal{B}(\mathbb{R}^d)$-measurable. Moreover, we define the finite constant $C_{22} := \int_U (1+\Vert u \Vert)^{\gamma p} \mu(du) > 0$, which implies that $\mu: \mathcal{B}(U) \rightarrow [0,\infty)$ is finite as $\mu(U) = \int_U \mu(du) \leq C_{22} < \infty$. Then, it follows for every $f \in C^0_b(\mathbb{R}^m;\mathbb{R}^d)$ that
	\begin{equation*}
		\begin{aligned}
			\Vert f\vert_U \Vert_{L^p(U,\mathcal{B}(U),\mu;\mathbb{R}^d)} & = \left( \int_U \Vert f(u) \Vert^p \mu(du) \right)^\frac{1}{p} \\
			& \leq \left( \int_U (1+\Vert u \Vert)^{\gamma p} \mu(du) \right)^\frac{1}{p} \sup_{u \in U} \frac{\Vert f(u) \Vert}{(1+\Vert u \Vert)^\gamma} \\
			& \leq C_{22}^\frac{1}{p} \Vert f \Vert_{C^0_{pol,\gamma}(\mathbb{R}^m;\mathbb{R}^d)},
		\end{aligned}
	\end{equation*}
	which shows that the restriction map \eqref{EqAssApprox1} is continuous. In order to show that its image is dense, we fix some $f \in L^p(U,\mathcal{B}(U),\mu;\mathbb{R}^d)$ and $\varepsilon > 0$. Then, we extend $f: U \rightarrow \mathbb{R}^d$ to the function
	\begin{equation*}
		\mathbb{R}^m \ni u \quad \mapsto \quad \overline{f}(u) :=
		\begin{cases}
			f(u), & u \in U, \\
			0, & u \in \mathbb{R}^m \setminus U.
		\end{cases}
	\end{equation*}
	Moreover, we extend $\mu: \mathcal{B}(U) \rightarrow [0,\infty)$ to the Borel measure $\mathcal{B}(\mathbb{R}^m) \ni E \mapsto \overline{\mu}(E) := \mu(U \cap E) \in [0,\infty)$, which implies that $\overline{f} \in L^p(\mathbb{R}^m,\mathcal{B}(\mathbb{R}^m),\overline{\mu};\mathbb{R}^d)$. Hence, by applying \cite[Corollary~2.2.2]{bogachev07} componentwise (with $\overline{\mu}(B) = \mu(U \cap B) \leq \mu(U) \leq C_{22} < \infty$ for any bounded $B \in \mathcal{B}(\mathbb{R}^m)$), there exists some $g \in C^\infty_c(\mathbb{R}^m;\mathbb{R}^d) \subseteq C^0_b(\mathbb{R}^m;\mathbb{R}^d)$ with $\Vert \overline{f} - g \Vert_{L^p(\mathbb{R}^m,\mathcal{B}(\mathbb{R}^m),\overline{\mu};\mathbb{R}^d)} < \varepsilon$, which implies
	\begin{equation*}
		\Vert f - g\vert_U \Vert_{L^p(U,\mathcal{B}(U),\mu;\mathbb{R}^d)} = \Vert \overline{f} - g \Vert_{L^p(\mathbb{R}^m,\mathcal{B}(\mathbb{R}^m),\overline{\mu};\mathbb{R}^d)} < \varepsilon.
	\end{equation*}
	Since $f \in C^0(U;\mathbb{R}^d)$ and $\varepsilon > 0$ were chosen arbitrarily, it follows that the image $\left\lbrace f\vert_U: f \in C^0_b(\mathbb{R}^m;\mathbb{R}^d) \right\rbrace$ of the continuous embedding \eqref{EqAssApprox1} is dense in $L^p(U,\mathcal{B}(U),\mu;\mathbb{R}^d)$.
	
	For \ref{ExApproxWkp}, we first use that $f \in C^k_b(\mathbb{R}^m;\mathbb{R}^d)$ is $k$-times differentiable to conclude for every $\alpha \in \mathbb{N}^m_{0,k}$ that $\partial_\alpha f\vert_U: U \rightarrow \mathbb{R}^d$ is $\mathcal{L}(U)/\mathcal{B}(\mathbb{R}^d)$-measurable. Moreover, we use that $U \subset \mathbb{R}^m$ is bounded to define the finite constant $C_{23} := \int_U (1+\Vert u \Vert)^{\gamma p} du > 0$. Then, it follows for every $f \in C^k_b(\mathbb{R}^m;\mathbb{R}^d)$ that
	\begin{equation*}
		\begin{aligned}
			\Vert f \Vert_{W^{k,p}(U,\mathcal{L}(U),du;\mathbb{R}^d)} & = \left( \sum_{\alpha \in \mathbb{N}^m_{0,k}} \int_U \Vert \partial_\alpha f(u) \Vert^p du \right)^\frac{1}{p} \\
			& \leq \left( \left\vert \mathbb{N}^m_{0,k} \right\vert \int_U (1+\Vert u \Vert)^{\gamma p} du \right)^\frac{1}{p} \max_{\alpha \in \mathbb{N}^m_{0,k}} \sup_{u \in U} \frac{\Vert \partial_\alpha f(u) \Vert}{(1+\Vert u \Vert)^{\gamma p}} \\
			& \leq \left( C_{23} \left\vert \mathbb{N}^m_{0,k} \right\vert \right)^\frac{1}{p} \Vert f \Vert_{C^k_{pol,\gamma}(\mathbb{R}^m;\mathbb{R}^d)},
		\end{aligned}
	\end{equation*}
	which shows that the restriction map \eqref{EqAssApprox1} is continuous. In addition, by applying \cite[Theorem~3.18]{adams75} componentwise, $\left\lbrace g\vert_U: g \in C^\infty_c(\mathbb{R}^m;\mathbb{R}^d) \right\rbrace$ is dense in $W^{k,p}(U,\mathcal{L}(U),du;\mathbb{R}^d)$. Hence, by using that $C^\infty_c(\mathbb{R}^m;\mathbb{R}^d) \subseteq C^k_b(\mathbb{R}^m;\mathbb{R}^d)$, it follows that the image $\left\lbrace g\vert_U: g \in C^k_b(\mathbb{R}^m;\mathbb{R}^d)\right\rbrace$ of the continuous embedding \eqref{EqAssApprox1} is dense in $W^{k,p}(U,\mathcal{L}(U),du;\mathbb{R}^d)$.
	
	For \ref{ExApproxWkpw}, we use that $f \in C^k_b(\mathbb{R}^m;\mathbb{R}^d)$ is $k$-times differentiable to conclude for every $\alpha \in \mathbb{N}^m_{0,k}$ that $\partial_\alpha f\vert_U: U \rightarrow \mathbb{R}^d$ is $\mathcal{L}(U)/\mathcal{B}(\mathbb{R}^d)$-measurable. Moreover, by using the finite constant $C_{24} := \int_U (1+\Vert u \Vert)^{\gamma p} w(u) du > 0$, it follows for every $f \in C^k_b(\mathbb{R}^m;\mathbb{R}^d)$ that
	\begin{equation*}
		\begin{aligned}
			\Vert f \Vert_{W^{k,p}(U,\mathcal{L}(U),w;\mathbb{R}^d)} & = \left( \sum_{\alpha \in \mathbb{N}^m_{0,k}} \int_U \Vert \partial_\alpha f(u) \Vert^p w(u) du \right)^\frac{1}{p} \\
			& \leq \left( \left\vert \mathbb{N}^m_{0,k} \right\vert \int_U (1+\Vert u \Vert)^{\gamma p} w(u) du \right)^\frac{1}{p} \max_{\alpha \in \mathbb{N}^m_{0,k}} \sup_{u \in U} \frac{\Vert \partial_\alpha f(u) \Vert}{(1+\Vert u \Vert)^\gamma} \\
			& \leq \left( C_{24} \left\vert \mathbb{N}^m_{0,k} \right\vert \right)^\frac{1}{p} \Vert f \Vert_{C^k_{pol,\gamma}(\mathbb{R}^m;\mathbb{R}^d)}.
		\end{aligned}
	\end{equation*}
	which shows that the restriction map \eqref{EqAssApprox1} is continuous. In addition, we apply Proposition~\ref{PropApproxWkpw} to conclude that $\left\lbrace g\vert_U: g \in C^\infty_c(\mathbb{R}^m;\mathbb{R}^d) \right\rbrace$ is dense in $W^{k,p}(U,\mathcal{L}(U),w;\mathbb{R}^d)$. Hence, by using that $C^\infty_c(\mathbb{R}^m;\mathbb{R}^d) \subseteq C^k_b(\mathbb{R}^m;\mathbb{R}^d)$, it follows that the image $\left\lbrace g\vert_U: g \in C^k_b(\mathbb{R}^m;\mathbb{R}^d)\right\rbrace$ of the continuous embedding \eqref{EqAssApprox1} is dense in $W^{k,p}(U,\mathcal{L}(U),w;\mathbb{R}^d)$.
\end{proof}

\subsection{Proof of results in Section~\ref{SecApproxRates}.}
\label{SecProof4}

\subsubsection{Integral representation}
\label{AppIntRepr}

In this section, we show the integral representation in Proposition~\ref{PropIntRepr}. To this end, we first prove that the ridgelet transform of a fixed function is continuous.

\begin{lemma}
	\label{LemmaRidgeletCont}
	Let $\psi \in \mathcal{S}_0(\mathbb{R};\mathbb{C})$ and $g \in L^1(\mathbb{R}^m,\mathcal{L}(\mathbb{R}^m),du;\mathbb{R}^d)$. Then, the function $\mathbb{R}^m \times \mathbb{R} \ni (a,b) \mapsto (\mathfrak{R}_\psi g)(a,b) := \int_{\mathbb{R}^m} \psi\big( a^\top u - b \big) g(u) \Vert a \Vert du \in \mathbb{C}^d$ is continuous.
\end{lemma}
\begin{proof}
	Fix a sequence $(a_M,b_M)_{n \in \mathbb{N}} \subseteq \mathbb{R}^m \times \mathbb{R}$ converging to some $(a,b) \in \mathbb{R}^m \times \mathbb{R}$. Then, by using the constants $C_\psi := \sup_{s \in \mathbb{R}} \vert \psi(s) \vert <\infty$ and $C_a := \sup_{M \in \mathbb{N}} \Vert a_M \Vert <\infty$, it holds for every $M \in \mathbb{N}$ that
	\begin{equation*}
		\left\Vert \psi\big( a_M^\top u - b_M \big) g(u) \right\Vert \leq C_a C_\psi \Vert g(u) \Vert,
	\end{equation*}
	where the right-hand side is Lebesgue-integrable. Moreover, by using that $\psi \in \mathcal{S}_0(\mathbb{R};\mathbb{C})$ is continuous, it follows that $\psi\big( a_M^\top u - b_M \big) g(u) \Vert a_M \Vert \rightarrow \psi\big( a^\top u - b \big) g(u) \Vert a \Vert$ for any $u \in U$, as $n \rightarrow \infty$. Hence, we can apply the vector-valued dominated convergence theorem (see \cite[Proposition~2.5]{hytoenen16}) to conclude that
	\begin{equation*}
		\begin{aligned}
			& \lim_{M \rightarrow \infty} (\mathfrak{R}_\psi g)(a_M,b_M) = \lim_{M \rightarrow \infty} \int_{\mathbb{R}^m} \psi\left( a_M^\top u - b_M \right) g(u) \Vert a_M \Vert du \\
			& \quad\quad = \int_{\mathbb{R}^m} \psi\left( a^\top u - b \right) g(u) \Vert a \Vert du = (\mathfrak{R}_\psi g)(a,b),
		\end{aligned}
	\end{equation*}
	which completes the proof.
\end{proof}

In order to prove Proposition~\ref{PropIntRepr}, we denote by $\mathbb{S}^{m-1} := \left\lbrace v \in \mathbb{R}^m: \Vert v \Vert = 1 \right\rbrace$ the unit sphere in $\mathbb{R}^m$ and define the space $\mathbb{Y}^{m+1} := \mathbb{S}^{m-1} \times (0,\infty) \times \mathbb{R}$. Then, for any $\psi \in \mathcal{S}_0(\mathbb{R};\mathbb{C})$, we follow \cite[Equation~32]{sonoda17} and define the ridgelet transform in polar coordinates of any $g \in L^1(\mathbb{R}^m,\mathcal{L}(\mathbb{R}^m),du;\mathbb{R}^d)$ as
\begin{equation}
	\label{EqDefRidgeletPolar}
	\mathbb{Y}^{m+1} \ni (v,s,t) \quad \mapsto \quad (\widetilde{\mathfrak{R}}_\psi g)(v,s,t) := \int_{\mathbb{R}^m} g(u) \overline{\psi\left( \frac{v^\top u-t}{s} \right)} \frac{1}{s} du \in \mathbb{C}^d.
\end{equation}
Moreover, we follow \cite[Definition~4.4]{sonoda17} and recall that the dual ridgelet transform $\mathfrak{R}_\rho^\dagger$ of any $Q: \mathbb{S}^{m-1} \times (0,\infty) \times \mathbb{R} \rightarrow \mathbb{C}$ satisfying $Q\big( v,s,v^\top u - s \, \cdot \big) := \big( z \mapsto Q\big( v,s,v^\top u - s z \big) \big) \in \mathcal{S}(\mathbb{R};\mathbb{C})$, for all $v \in \mathbb{S}^{m-1}$, $s \in (0,\infty)$, and $u \in \mathbb{R}^m$, is defined by
\begin{equation*}
	\mathbb{R}^m \ni u \quad \mapsto \quad (\mathfrak{R}_\rho^\dagger Q)(u) := \lim_{\delta_1 \rightarrow 0 \atop \delta_2 \rightarrow \infty} \int_{\mathbb{S}^{m-1}} \int_{\delta_1}^{\delta_2} T_\rho\left( Q\left( v,s,v^\top u - s \, \cdot \right) \right) \frac{1}{s^{m+1}} ds dv \in \mathbb{R}^d.
\end{equation*}

\begin{proof}[Proof of Proposition~\ref{PropIntRepr}]
	Fix a function $g = (g_1,...,g_d)^\top \in L^1(\mathbb{R}^m,\mathcal{L}(\mathbb{R}^m),du;\mathbb{R}^d)$ satisfying $\widehat{g} \in L^1(\mathbb{R}^m,\mathcal{L}(\mathbb{R}^m),du;\mathbb{C}^d)$. Then, the latter implies for every $i=1,...,d$ that
	\begin{equation}
		\label{EqPropIntReprProof1}
		\begin{aligned}
			\Vert \widehat{g}_i \Vert_{L^1(\mathbb{R}^m,\mathcal{L}(\mathbb{R}^m),du)} & = \int_{\mathbb{R}^m} \big\vert \widehat{g}_i(\xi) \big\vert d\xi \leq \int_{\mathbb{R}^m} \Vert \widehat{g}(\xi) \Vert d\xi = \Vert \widehat{g} \Vert_{L^1(\mathbb{R}^m,\mathcal{L}(\mathbb{R}^m),du;\mathbb{C}^d)} < \infty.
		\end{aligned}
	\end{equation}
	Hence, by using that $(\mathfrak{R}_\psi g)(a,b) = 0$ for any $(a,b) \in \lbrace 0 \rbrace \times \mathbb{R}$, that $(\mathfrak{R}_\psi g)(a,b) = (\widetilde{\mathfrak{R}}_\psi f)\big( \frac{a}{\Vert a \Vert},\frac{1}{\Vert a \Vert},\frac{b}{\Vert a \Vert} \big)$ for any $(a,b) \in (\mathbb{R}^m \setminus \lbrace 0 \rbrace) \times \mathbb{R}$, the substitution $(\mathbb{R}^m \setminus \lbrace 0 \rbrace) \times \mathbb{R} \ni (a,b) \mapsto (v,s,z) := \big( \frac{a}{\Vert a \Vert}, \frac{1}{\Vert a \Vert}, a^\top u - b \big) \in \mathbb{S}^{m-1} \times (0,\infty) \times \mathbb{R}$ with Jacobi determinante $db da = s^{-m} dz ds dv$, and \cite[Theorem~5.6]{sonoda17} applied to $f_i \in L^1(\mathbb{R}^m,\mathcal{L}(\mathbb{R}^m),du)$ with $\widehat{f_i} \in L^1(\mathbb{R}^m,\mathcal{L}(\mathbb{R}^m),du;\mathbb{C})$ (see \eqref{EqPropIntReprProof1}), it follows for a.e.~$u \in \mathbb{R}^m$ that
	\begin{equation*}
		\begin{aligned}
			& \int_{\mathbb{R}^m} \int_{\mathbb{R}} (\mathfrak{R}_\psi g)(a,b) \rho\left( a^\top u - b \right) db da = \int_{\mathbb{R}^m \setminus \lbrace 0 \rbrace} \int_{\mathbb{R}} (\widetilde{\mathfrak{R}}_\psi g)\left( \frac{a}{\Vert a \Vert},\frac{1}{\Vert a \Vert},\frac{b}{\Vert a \Vert} \right) \rho\left( a^\top u - b \right) db da \\
			& \quad\quad = \left( \int_{\mathbb{S}^{m-1}} \int_0^\infty \int_{\mathbb{R}} (\widetilde{\mathfrak{R}}_\psi g_i)\left( v,s,v^\top u - sz \right) \rho(z) \frac{1}{s^{m+1}} dz ds dv \right)_{i=1,...,d}^\top \\
			& \quad\quad = \Bigg( \lim_{\delta_1 \rightarrow 0 \atop \delta_2 \rightarrow \infty} \int_{\mathbb{S}^{m-1}} \int_{\delta_1}^{\delta_2} T_\rho\left( (\widetilde{\mathfrak{R}}_\psi g_i)\left( v,s,v^\top u - s \, \cdot \right) \right) \frac{1}{s^{m+1}} ds dv \Bigg)_{i=1,...,d}^\top \\
			& \quad\quad = \big( \big( \widetilde{\mathfrak{R}}_\rho^\dagger \mathfrak{R}_\psi g_i \big)(u) \big)_{i=1,...,d}^\top = \big( C^{(\psi,\rho)}_m g_i(u) \big)_{i=1,...,d}^\top = C^{(\psi,\rho)}_m g(u),
		\end{aligned}
	\end{equation*}
	which completes the proof.
\end{proof}

\subsubsection{Properties of weighted Sobolev space $W^{k,p}(U,\mathcal{L}(U),w;\mathbb{R}^d)$}
\label{SecWkpw}

In this section, we show that the weighted Sobolev space $(W^{k,p}(U,\mathcal{L}(U),w,\mathbb{R}^d),\Vert \cdot \Vert_{W^{k,p}(U,\mathcal{L}(U),w,\mathbb{R}^d)})$ introduced in Footnote~\ref{FootnoteWkpw} is separable and has Banach space type $t := \min(2,p)$, where we set $(W^{0,p}(U,\mathcal{L}(U),w;\mathbb{R}^d),\Vert \cdot \Vert_{W^{0,p}(U,\mathcal{L}(U),w;\mathbb{R}^d)}) := (L^p(U,\mathcal{L}(U),w(u) du;\mathbb{R}^d),\Vert \cdot \Vert_{L^p(U,\mathcal{L}(U),w(u) du;\mathbb{R}^d)})$.

\begin{lemma}
	\label{LemmaWkpwSep}
	Let $k \in \mathbb{N}_0$, $p \in [1,\infty)$, $U \subseteq \mathbb{R}^m$ (open, if $k \geq 1$), and $w: U \rightarrow [0,\infty)$ be a weight. Then, the Banach space $(W^{k,p}(U,\mathcal{L}(U),w,\mathbb{R}^d),\Vert \cdot \Vert_{W^{k,p}(U,\mathcal{L}(U),w,\mathbb{R}^d)})$ is separable.
\end{lemma}
\begin{proof}
	First, we show the conclusion for $k = 0$, i.e.~that the Banach space $(W^{0,p}(U,\mathcal{L}(U),w;\mathbb{R}^d),\Vert \cdot \Vert_{W^{0,p}(U,\mathcal{L}(U),w;\mathbb{R}^d)}) := (L^p(U,\mathcal{L}(U),w(u) du;\mathbb{R}^d),\Vert \cdot \Vert_{L^p(U,\mathcal{L}(U),w(u) du;\mathbb{R}^d)})$ is separable. For this purpose, we observe that $\mathcal{B}(U)$ is generated by sets of the form $U \cap \bigtimes_{l=1}^m [r_{l,1},r_{l,2})$, with $r_{l,1}, r_{l,2} \in \mathbb{Q}$, $l=1,...,m$. Moreover, by using that $\mathcal{L}(U)$ and $\mathcal{B}(U)$ coincide up to Lebesgue nullsets and that $w: U \rightarrow [0,\infty)$ is a weight, ensuring that the measure spaces $(U,\mathcal{L}(U),w(u) du)$ and $(U,\mathcal{L}(U),du)$ share the same null sets, we conclude that $(U,\mathcal{L}(U),w(u) du)$ is countably generated up to ($w(u) du$)-null sets. Hence, by using \cite[p.~92]{doob94} componentwise, it follows that $(W^{0,p}(U,\mathcal{L}(U),w;\mathbb{R}^d),\Vert \cdot \Vert_{W^{0,p}(U,\mathcal{L}(U),w;\mathbb{R}^d)}) := (L^p(U,\mathcal{L}(U),w(u) du;\mathbb{R}^d),\Vert \cdot \Vert_{L^p(U,\mathcal{L}(U),w(u) du;\mathbb{R}^d)})$ is separable.
	
	Now, for the general case of $k \geq 1$, we consider the Banach space $(W^{k,p}(U,\mathcal{L}(U),w,\mathbb{R}^d),\Vert \cdot \Vert_{W^{k,p}(U,\mathcal{L}(U),w,\mathbb{R}^d)})$. Then, we define the map
	\begin{equation*}
		W^{k,p}(U,\mathcal{L}(U),w,\mathbb{R}^d) \ni f \quad \mapsto \quad \Xi(f) := (\partial_\alpha f)_{\alpha \in \mathbb{N}^m_{0,k}} \in \bigtimes_{\alpha \in \mathbb{N}^m_{0,k}} L^p(U,\mathcal{L}(U),w(u) du,\mathbb{R}^d) =: Z,
	\end{equation*}
	where $Z$ is equipped with the norm $\Vert g \Vert_Z := \sum_{\alpha \in \mathbb{N}^m_{0,k}} \Vert g_\alpha \Vert_{L^p(U,\mathcal{L}(U),du,\mathbb{R}^d)}$, for $g := (g_\alpha)_{\alpha \in \mathbb{N}^m_{0,k}} \in Z$. Then, by using the previous step, we conclude that the Banach space $(Z,\Vert \cdot \Vert_Z)$ is separable as finite product of separable Banach spaces. Hence, by using that $W^{k,p}(U,\mathcal{L}(U),w,\mathbb{R}^d)$ is by definition isometrically isomorphic to the closed vector subspace $\img(\Xi) := \left\lbrace \Xi(f): f \in W^{k,p}(U,\mathcal{L}(U),w,\mathbb{R}^d) \right\rbrace \subseteq Z$, it follows that $(W^{k,p}(U,\mathcal{L}(U),w,\mathbb{R}^d),\Vert \cdot \Vert_{W^{k,p}(U,\mathcal{L}(U),w,\mathbb{R}^d)})$ is separable.
\end{proof}

Moreover, we recall the notion of Banach space types and refer to \cite[Section~6.2]{albiac06}, \cite[Chapter~9]{ledoux91}, and \cite[Section~4.3.b]{hytoenen16} for more details.

\begin{definition}[{\cite[Definition~4.3.12~(1)]{hytoenen16}}]
	A Banach space $(X,\Vert \cdot \Vert_X)$ is called of \emph{type $t \in [1,2]$} if there exists a constant $C_X > 0$ such that for every $N \in \mathbb{N}$, $(f_n)_{n=1,...,N} \subseteq X$, and Rademacher sequence\footnote{A \emph{Rademacher sequence} $(\epsilon_n)_{n=1,...,N}$ on a given probability space $(\widetilde{\Omega},\widetilde{\mathcal{F}},\widetilde{\mathbb{P}})$ is an i.i.d.~sequence of random variables $(\epsilon_n)_{n=1,...,N}$ such that $\widetilde{\mathbb{P}}[\epsilon_n = \pm 1] = 1/2$.} $(\epsilon_n)_{n=1,...,N}$ on a probability space $(\widetilde{\Omega},\widetilde{\mathcal{F}},\widetilde{\mathbb{P}})$, it holds that
	\begin{equation*}
		\widetilde{\mathbb{E}}\left[ \left\Vert \sum_{n=1}^N \epsilon_n f_n \right\Vert_X^t \right]^\frac{1}{t} \leq C_X \left( \sum_{n=1}^N \Vert f_n \Vert_X^t \right)^\frac{1}{t}.
	\end{equation*}
\end{definition}

Every Banach space $(X,\Vert \cdot \Vert_X)$ is of type $t = 1$ with constant $C_X = 1$, whereas only some Banach spaces have non-trivial type $t \in (1,2]$, e.g.,~every Hilbert space $(X,\Vert \cdot \Vert_X)$ is of type $t = 2$ with constant $C_X = 1$ (see \cite[Remark~6.2.11~(b)+(c)]{albiac06}). Moreover, $(L^p(U,\mathcal{B}(U),\mu;\mathbb{R}^d), \Vert \cdot \Vert_{L^p(U,\mathcal{B}(U),\mu;\mathbb{R}^d)})$ introduced in Footnote~\ref{FootnoteLp} is a Banach space of type $t = \min(2,p)$ with constant $C_{L^p(U,\Sigma,\mu;\mathbb{R}^d)} > 0$ depending only on $p \in [1,\infty)$ (see \cite[Theorem~6.2.14]{albiac06}). Now, we show that this still holds true for the weighted Sobolev space $(W^{k,p}(U,\mathcal{L}(U),w;\mathbb{R}^d),\Vert \cdot \Vert_{W^{k,p}(U,\mathcal{L}(U),w;\mathbb{R}^d)})$ introduced in \ref{FootnoteWkpw}, where $(W^{0,p}(U,\mathcal{L}(U),w;\mathbb{R}^d),\Vert \cdot \Vert_{W^{0,p}(U,\mathcal{L}(U),w;\mathbb{R}^d)}) := (L^p(U,\mathcal{L}(U),w(u) du;\mathbb{R}^d),\Vert \cdot \Vert_{L^p(U,\mathcal{L}(U),w(u) du;\mathbb{R}^d)})$.

\begin{lemma}
	\label{LemmaBanachSpaceType}
	Let $k \in \mathbb{N}_0$, $p \in [1,\infty)$, $U \subseteq \mathbb{R}^m$ (open, if $k \geq 1$), and $w: U \rightarrow [0,\infty)$ be a weight. Then, the Banach space $(W^{k,p}(U,\mathcal{L}(U),w;\mathbb{R}^d),\Vert \cdot \Vert_{W^{k,p}(U,\mathcal{L}(U),w;\mathbb{R}^d)})$ is of type $t = \min(2,p)$ with constant $C_{W^{k,p}(U,\mathcal{L}(U),w;\mathbb{R}^d)} > 0$ depending only on $p \in [1,\infty)$.
\end{lemma}
\begin{proof}
	First, we recall that $(W^{k,p}(U,\mathcal{L}(U),w;\mathbb{R}^d),\Vert \cdot \Vert_{W^{k,p}(U,\mathcal{L}(U),w;\mathbb{R}^d)})$ is a Banach space. Indeed, this follows from \cite[p.~96]{rudin87} (for $k = 0$) and \cite[Theorem~3.2]{adams75} (for $k \geq 1$).
	
	Now, we fix some $N \in \mathbb{N}$, $(f_n)_{n=1,...,N} \subseteq W^{k,p}(U,\mathcal{L}(U),w;\mathbb{R}^d)$, and an i.i.d.~sequence $(\epsilon_n)_{n=1,...,N}$ defined on a probability space $(\widetilde{\Omega},\widetilde{\mathcal{F}},\widetilde{\mathbb{P}})$ such that $\widetilde{\mathbb{P}}[\epsilon_n = \pm 1] = 1/2$. Then, by using Fubini's theorem and the classical Khintchine inequality in \cite[Lemma~4.1]{ledoux91} with constant $C_p > 0$ depending only on $p \in [1,\infty)$, it follows that
	\begin{equation}
		\label{EqLemmaBanachSpaceTypeProof2}
		\begin{aligned}
			\widetilde{\mathbb{E}}\left[ \left\Vert \sum_{n=1}^N \epsilon_n f_n \right\Vert_{W^{k,p}(U,\mathcal{L}(U),w;\mathbb{R}^d)}^p \right]^\frac{1}{p} & = \widetilde{\mathbb{E}}\left[ \sum_{\alpha \in \mathbb{N}^m_{0,k}} \int_U \left\Vert \sum_{n=1}^N \epsilon_n \partial_\alpha f_n(u) \right\Vert^p w(u) du \right]^\frac{1}{p} \\
			& = \left( \sum_{\alpha \in \mathbb{N}^m_{0,k}} \int_U \widetilde{\mathbb{E}}\left[ \left\Vert \sum_{n=1}^N \epsilon_n \partial_\alpha f_n(u) \right\Vert^p \right] w(u) du \right)^\frac{1}{p} \\
			& \leq C_p \left( \sum_{\alpha \in \mathbb{N}^m_{0,k}} \int_U \left( \sum_{n=1}^N \left\Vert \partial_\alpha f_n(u) \right\Vert^2 \right)^\frac{p}{2} w(u) du \right)^\frac{1}{p}.
		\end{aligned} 
	\end{equation}
	If $p \in [1,2]$, we use \eqref{EqLemmaBanachSpaceTypeProof2} and the inequality $\left( \sum_{n=1}^N x_n \right)^{p/2} \leq \sum_{n=1}^N x_n^{p/2}$ for any $x_1,...,x_N \geq 0$ to conclude that
	\begin{equation*}
		\begin{aligned}
			\widetilde{\mathbb{E}}\left[ \left\Vert \sum_{n=1}^N \epsilon_n f_n \right\Vert_{W^{k,p}(U,\mathcal{L}(U),w;\mathbb{R}^d)}^{\min(2,p)} \right]^\frac{1}{\min(2,p)} & = \widetilde{\mathbb{E}}\left[ \left\Vert \sum_{n=1}^N \epsilon_n f_n \right\Vert_{W^{k,p}(U,\mathcal{L}(U),w;\mathbb{R}^d)}^p \right]^\frac{1}{p} \\
			& \leq C_p \left( \sum_{\alpha \in \mathbb{N}^m_{0,k}} \int_U \left( \sum_{n=1}^N \left\Vert \partial_\alpha f_n(u) \right\Vert^2 \right)^\frac{p}{2} w(u) du \right)^\frac{1}{p} \\
			& \leq C_p \left( \sum_{n=1}^N \sum_{\alpha \in \mathbb{N}^m_{0,k}} \int_U \left\Vert \partial_\alpha f_n(u) \right\Vert^p w(u) du \right)^\frac{1}{p} \\
			& = C_p \left( \sum_{n=1}^N \Vert f_n \Vert_{W^{k,p}(U,\mathcal{L}(U),w;\mathbb{R}^d)}^{\min(2,p)} \right)^\frac{1}{\min(2,p)}.
		\end{aligned}
	\end{equation*}
	This shows for $p \in [1,2]$ that the Banach space $(W^{k,p}(U,\mathcal{L}(U),w;\mathbb{R}^d),\Vert \cdot \Vert_{W^{k,p}(U,\mathcal{L}(U),w;\mathbb{R}^d)})$ is of type $t = \min(2,p)$, where the constant $C_p > 0$ depends only on $p \in [1,\infty)$.
	
	Otherwise, if $p \in (2,\infty)$, we consider the measure spaces $(\lbrace 1,...,N \rbrace,\mathcal{P}(\lbrace 1,...,N \rbrace),\eta)$ and $(\mathbb{N}^m_{0,k} \times U,\mathcal{P}(\mathbb{N}^m_{0,k}) \otimes \mathcal{L}(U),\mu \otimes w)$, where $\mathcal{P}(\lbrace 1,...,N \rbrace)$ and $\mathcal{P}(\mathbb{N}^m_{0,k})$ denote the power sets of $\lbrace 1,...,N \rbrace$ and $\mathbb{N}^m_{0,k}$, respectively, and where $\mathcal{P}(\lbrace 1,...,N \rbrace) \ni A \mapsto \eta(A) := \sum_{n=1}^N \mathds{1}_A(n) \in [0,\infty)$ and $\mathcal{P}(\mathbb{N}^m_{0,k}) \otimes \mathcal{L}(U) \ni (A,B) \mapsto (\mu \otimes w)(A,B) := \big( \sum_{\alpha \in \mathbb{N}^m_{0,k}} \mathds{1}_A(\alpha) \big) \int_B w(u) du \in [0,\infty]$ are both measures. Then, by using the Minkowski inequality in \cite[Proposition~1.2.22]{hytoenen16} with $p \geq 2$, it follows for every $\mathbf{f} \in L^2(\lbrace 1,...,N \rbrace,\mathcal{P}(\lbrace 1,...,N \rbrace),\eta;L^p(\mathbb{N}^m_{0,k} \times U,\mathcal{P}(\mathbb{N}^m_{0,k}) \otimes \mathcal{L}(U),\mu \otimes w;\mathbb{R}^d))$ that
	\begin{equation}
		\label{EqLemmaBanachSpaceTypeProof3}
		\begin{aligned}
			& \Vert \mathbf{f} \Vert_{L^p(\mathbb{N}^m_{0,k} \times U,\mathcal{P}(\mathbb{N}^m_{0,k}) \otimes \mathcal{L}(U),\mu \otimes w;L^2(\lbrace 1,...,N \rbrace,\mathcal{P}(\lbrace 1,...,N \rbrace),\eta;\mathbb{R}^d))} \\
			& \quad\quad \leq \Vert \mathbf{f} \Vert_{L^2(\lbrace 1,...,N \rbrace,\mathcal{P}(\lbrace 1,...,N \rbrace),\eta;L^p(\mathbb{N}^m_{0,k} \times U,\mathcal{P}(\mathbb{N}^m_{0,k}) \otimes \mathcal{L}(U),\mu \otimes w;\mathbb{R}^d))}.
		\end{aligned}
	\end{equation}
	Now, we define the map $\lbrace 1,...,N \rbrace \times (\mathbb{N}^m_{0,k} \times U) \ni (n;\alpha,u) \mapsto \mathbf{f}(n;\alpha,u) := \partial_\alpha f_n(u) \in \mathbb{R}^d$ satisfying
	\begin{equation}
		\label{EqLemmaBanachSpaceTypeProof4}
		\begin{aligned}
			\Vert \mathbf{f} \Vert_{L^2(\lbrace 1,...,N \rbrace,\mathcal{P}(\lbrace 1,...,N \rbrace),\eta;L^p(\mathbb{N}^m_{0,k} \times U,\mathcal{P}(\mathbb{N}^m_{0,k}) \otimes \mathcal{L}(U),\mu \otimes w;\mathbb{R}^d))} & = \left( \sum_{n=1}^N \left( \sum_{\alpha \in \mathbb{N}^m_{0,k}} \int_U \Vert \partial_\alpha f_n(u) \Vert^p w(u) du \right)^\frac{2}{p} \right)^\frac{1}{2} \\
			& = \left( \sum_{n=1}^N \Vert f_n \Vert_{W^{k,p}(U,\mathcal{L}(U),w;\mathbb{R}^d)}^2 \right)^\frac{1}{2} < \infty,
		\end{aligned}
	\end{equation}
	which shows that $\mathbf{f} \in L^2(\lbrace 1,...,N \rbrace,\mathcal{P}(\lbrace 1,...,N \rbrace),\eta;L^p(\mathbb{N}^m_{0,k} \times U,\mathcal{P}(\mathbb{N}^m_{0,k}) \otimes \mathcal{L}(U),\mu \otimes w;\mathbb{R}^d))$. Hence, by using first Jensen's inequality and then by combining \eqref{EqLemmaBanachSpaceTypeProof2} and \eqref{EqLemmaBanachSpaceTypeProof3} with \eqref{EqLemmaBanachSpaceTypeProof4}, we conclude that
	\begin{equation*}
		\begin{aligned}
			\widetilde{\mathbb{E}}\left[ \left\Vert \sum_{n=1}^N \epsilon_n f_n \right\Vert_{W^{k,p}(U,\mathcal{L}(U),w;\mathbb{R}^d)}^{\min(2,p)} \right]^\frac{1}{\min(2,p)} & \leq \widetilde{\mathbb{E}}\left[ \left\Vert \sum_{n=1}^N \epsilon_n f_n \right\Vert_{W^{k,p}(U,\mathcal{L}(U),w;\mathbb{R}^d)}^p \right]^\frac{1}{p} \\
			& \leq C_p \left( \sum_{\alpha \in \mathbb{N}^m_{0,k}} \int_U \left( \sum_{n=1}^N \left\Vert \partial_\alpha f_n(u) \right\Vert^2 \right)^\frac{p}{2} w(u) du \right)^\frac{1}{p} \\
			& = C_p \Vert \mathbf{f} \Vert_{L^p(\mathbb{N}^m_{0,k} \times U,\mathcal{P}(\mathbb{N}^m_{0,k}) \otimes \mathcal{L}(U),\mu \otimes w;L^2(\lbrace 1,...,N \rbrace,\mathcal{P}(\lbrace 1,...,N \rbrace),\eta;\mathbb{R}^d))} \\
			& \leq C_p \Vert \mathbf{f} \Vert_{L^2(\lbrace 1,...,N \rbrace,\mathcal{P}(\lbrace 1,...,N \rbrace),\eta;L^p(\mathbb{N}^m_{0,k} \times U,\mathcal{P}(\mathbb{N}^m_{0,k}) \otimes \mathcal{L}(U),\mu \otimes w;\mathbb{R}^d))} \\
			& = C_p \left( \sum_{n=1}^N \Vert f \Vert_{W^{k,p}(U,\mathcal{L}(U),w;\mathbb{R}^d)}^{\min(2,p)} \right)^\frac{1}{\min(2,p)}.
		\end{aligned}
	\end{equation*}
	This shows for $p \in (2,\infty)$ that the Banach space $(W^{k,p}(U,\mathcal{L}(U),w;\mathbb{R}^d),\Vert \cdot \Vert_{W^{k,p}(U,\mathcal{L}(U),w;\mathbb{R}^d)})$ is of type $t = \min(2,p)$, where the constant $C_p > 0$ depends only on $p \in [1,\infty)$.
\end{proof}

\subsubsection{Randomized neurons and strong measurability}
\label{AppBochner}

In this section, we randomly initialize the weight vectors and biases inside the activation function to obtain the approximation rates in Theorem~\ref{ThmApproxRates}. To this end, we first show that the map from the parameters to a neuron is continuous.

\begin{lemma}
	\label{LemmaWkpwCont}
	For $k \in \mathbb{N}_0$, $p \in [1,\infty)$, $U \subseteq \mathbb{R}^m$ (open, if $k \geq 1$), $\gamma \in [0,\infty)$, and $\rho \in C^k_{pol,\gamma}(\mathbb{R})$, let $w: U \rightarrow [0,\infty)$ be a weight such that the constant $C^{(\gamma,p)}_{U,w} > 0$ defined in \eqref{EqThmApproxRates1} is finite. Then, the mapping
	\begin{equation*}
		\mathbb{R}^d \times \mathbb{R}^m \times \mathbb{R} \ni (y,a,b) \quad \mapsto \quad y \rho\left( a^\top \cdot - b \right) \in W^{k,p}(U,\mathcal{L}(U),w,\mathbb{R}^d)
	\end{equation*}
	is continuous, where $y \rho\left( a^\top \cdot - b \right)$ denotes the function $U \ni u \mapsto y \rho\left( a^\top u - b \right) \in \mathbb{R}^d$.
\end{lemma}
\begin{proof}
	Fix a sequence $(y_M,a_M,b_M)_{M \in \mathbb{N}} \subseteq \mathbb{R}^d \times \mathbb{R}^m \times \mathbb{R}$ converging to $(y,a,b) \in \mathbb{R}^d \times \mathbb{R}^m \times \mathbb{R}$. Then, by using that $y_M a_M^\alpha (1 + \Vert a_M \Vert + \vert b_M \vert)$ converges uniformly in $\alpha \in \mathbb{N}^m_{0,k}$ to $y a^\alpha (1 + \Vert a \Vert + \vert b \vert)$, where $a^\alpha := \prod_{l=1}^m a_l^{\alpha_l}$ for $a := (a_1,...,a_m)^\top \in \mathbb{R}^m$ and $\alpha := (\alpha_1,...,\alpha_m) \in \mathbb{N}^m_{0,k}$, the constant $C_{y,a,b} := \max_{\alpha \in \mathbb{N}^m_{0,k}} \left\Vert y a^\alpha \right\Vert (1 + \Vert a \Vert + \vert b \vert) + \sup_{M \in \mathbb{N}} \big( \max_{\alpha \in \mathbb{N}^m_{0,k}} \left\Vert y_M a_M^\alpha \right\Vert (1 + \Vert a_M \Vert + \vert b_M \vert) \big) \geq 0$ is finite. Hence, by using that $\rho \in C^k_{pol,\gamma}(\mathbb{R})$, i.e.~that $\left\vert \rho^{(j)}(s) \right\vert \leq \Vert \rho \Vert_{C^k_{pol,\gamma}(\mathbb{R})} (1+\vert s \vert)^\gamma$ for any $j = 0,...,k$ and $s \in \mathbb{R}$, the inequality $1+\left\vert a_M^\top u - b_M \right\vert \leq 1+\Vert a_M \Vert \Vert u \Vert + \vert b_M \vert \leq (1+\Vert a_M \Vert + \vert b_M \vert) (1 + \Vert u \Vert)$ for any $M \in \mathbb{N}$ and $u \in \mathbb{R}^m$, it follows for every $\alpha \in \mathbb{N}^m_{0,k}$, $u \in U$, and $M \in \mathbb{N}$ that
	\begin{equation}
		\label{EqLemmaWkpwContProof1}
		\begin{aligned}
			\left\Vert y_M \rho^{(\vert \alpha \vert)}\left( a_M^\top u - b_M \right) a_M^\alpha \right\Vert & \leq \left\Vert y_M a_M^\alpha \right\Vert \left\vert \rho^{(\vert \alpha \vert)}\left( a_M^\top u - b_M \right) \right\vert \\
			& \leq \left\Vert y_M a_M^\alpha \right\Vert \Vert \rho \Vert_{C^k_{pol,\gamma}(\mathbb{R})} \left( 1 + \left\vert a_M^\top u - b_M \right\vert \right)^\gamma \\
			& \leq \left\Vert y_M a_M^\alpha \right\Vert (1+\Vert a_M \Vert+\vert b_M \vert)^\gamma \Vert \rho \Vert_{C^k_{pol,\gamma}(\mathbb{R})} (1+\Vert u \Vert)^\gamma \\
			& \leq C_{y,a,b} \Vert \rho \Vert_{C^k_{pol,\gamma}(\mathbb{R})} (1+\Vert u \Vert)^\gamma.
		\end{aligned}
	\end{equation}
	Analogously, we conclude for every $\alpha \in \mathbb{N}^m_{0,k}$ and $u \in U$ that
	\begin{equation}
		\label{EqLemmaWkpwContProof2}
		\left\Vert y_M \rho^{(\vert \alpha \vert)}\left( a_M^\top u - b_M \right) a_M^\alpha \right\Vert \leq C_{y,a,b} \Vert \rho \Vert_{C^k_{pol,\gamma}(\mathbb{R})} (1+\Vert u \Vert)^\gamma.
	\end{equation}
	Hence, by using the triangle inequality together with the inequality $(x+y)^p \leq 2^{p-1} \left( x^p + y^p \right)$ for any $x,y \geq 0$ as well as the inequalities \eqref{EqLemmaWkpwContProof1} and \eqref{EqLemmaWkpwContProof2}, it follows for every $\alpha \in \mathbb{N}^m_{0,k}$, $u \in U$, and $M \in \mathbb{N}$ that
	\begin{equation}
		\label{EqLemmaWkpwContProof3}
		\begin{aligned}
			& \left\Vert y \rho^{(\vert \alpha \vert)}\left( a^\top u - b \right) a^\alpha - y_M \rho^{(\vert \alpha \vert)}\left( a_M^\top u - b_M \right) a_M^\alpha \right\Vert^p \\
			& \quad\quad \leq 2^{p-1} \left( \left\Vert y \rho^{(\vert \alpha \vert)}\left( a^\top u - b \right) a^\alpha \right\Vert^p + \left\Vert y_M \rho^{(\vert \alpha \vert)}\left( a_M^\top u - b_M \right) a_M^\alpha \right\Vert^p \right) \\
			& \quad\quad \leq 2^p C_{y,a,b}^p \Vert \rho \Vert_{C^k_{pol,\gamma}(\mathbb{R})}^p (1+\Vert u \Vert)^{\gamma p}.
		\end{aligned}
	\end{equation}
	Thus, by applying the $\mathbb{R}^d$-valued dominated convergence theorem (see \cite[Proposition~1.2.5]{hytoenen16} with \eqref{EqLemmaWkpwContProof3} and $\int_U \left( 1 + \Vert u \Vert \right)^{\gamma p} w(u) du = \big( C^{(\gamma,p)}_{U,w} \big)^p < \infty$ by assumption), we have
	\begin{equation*}
		\begin{aligned}
			& \lim_{M \rightarrow \infty} \left\Vert y \rho\left( a^\top \cdot - b \right) - y_M \rho\left( a_M^\top \cdot - b_M \right) \right\Vert_{W^{k,p}(U,\mathcal{L}(U),w;\mathbb{R}^d)} \\
			& \quad\quad = \left( \sum_{\alpha \in \mathbb{N}^m_{0,k}} \lim_{M \rightarrow \infty} \int_U \left\Vert y \rho^{(\vert \alpha \vert)}\left( a^\top u - b \right) a^\alpha - y_M \rho^{(\vert \alpha \vert)}\left( a_M^\top u - b_M \right) a_M^\alpha \right\Vert^p w(u) du \right)^\frac{1}{p} = 0,
		\end{aligned}
	\end{equation*}
	which completes the proof.
\end{proof}

Moreover, we fix throughout the rest of this paper a probability space $(\Omega,\mathcal{F},\mathbb{P})$ and assume that $(a_n)_{n \in \mathbb{N}} \sim t_m$ and $(b_n)_{n \in \mathbb{N}} \sim t_1$ are independent sequences of independent and identically distributed (i.i.d.) random variables following a (multivariate) Student's $t$-distribution\footnote{\label{FootnoteTDistr}For any $m \in \mathbb{N}$, a random variable $a \sim t_m$ following a Student's $t$-distribution has probability density function $\mathbb{R}^m \ni a \mapsto p_a(a) = \frac{\Gamma((m+1)/2)}{\pi^{(m+1)/2}} \left( 1 + \Vert a \Vert^2 \right)^{-(m+1)/2} \in (0,\infty)$.}. In this case, we write $(a_n,b_n)_{n \in \mathbb{N}} \sim t_m \otimes t_1$. Then, we show that a randomized neuron with $(a_n,b_n)_{n \in \mathbb{N}} \sim t_m \otimes t_1$ used for the weight vectors and biases inside the activation function, is a strongly measurable map in the sense of \cite[Definition~1.1.14]{hytoenen16}, where we define the $\sigma$-algebra $\mathcal{F}_{a,b} := \sigma(\lbrace a_n,b_n: n \in \mathbb{N} \rbrace)$.

\begin{lemma}
	\label{LemmaWkpwStrongMbl}
	For $k \in \mathbb{N}_0$, $p \in [1,\infty)$, $U \subseteq \mathbb{R}^m$ (open, if $k \geq 1$), $\gamma \in [0,\infty)$, and $\rho \in C^k_{pol,\gamma}(\mathbb{R})$, let $w: U \rightarrow [0,\infty)$ be a weight such that the constant $C^{(\gamma,p)}_{U,w} > 0$ defined in \eqref{EqThmApproxRates1} is finite. Moreover, for $n \in \mathbb{N}$ and an $\mathcal{F}_{a,b}/\mathcal{B}(\mathbb{R}^d)$-measurable random vector $y: \Omega \rightarrow \mathbb{R}^d$, we define the map
	\begin{equation}
		\label{EqLemmaWkpwStrongMbl1}
		\Omega \ni \omega \quad \mapsto \quad R_n(\omega) := y(\omega) \rho\left( a_n(\omega)^\top \cdot - b_n(\omega) \right) \in W^{k,p}(U,\mathcal{L}(U),w,\mathbb{R}^d).
	\end{equation}
	Then, $R_n: \Omega \rightarrow W^{k,p}(U,\mathcal{L}(U),w,\mathbb{R}^d)$ is a strongly $(\mathbb{P},\mathcal{F}_{a,b})$-measurable map with values in the separable Banach space $(W^{k,p}(U,\mathcal{L}(U),w;\mathbb{R}^d),\Vert \cdot \Vert_{W^{k,p}(U,\mathcal{L}(U),w;\mathbb{R}^d)})$.
\end{lemma}
\begin{proof}
	First, we show that the map $R_n: \Omega \rightarrow W^{k,p}(U,\mathcal{L}(U),w,\mathbb{R}^d)$ takes values in the Banach space $(W^{k,p}(U,\mathcal{L}(U),w;\mathbb{R}^d),\Vert \cdot \Vert_{W^{k,p}(U,\mathcal{L}(U),w;\mathbb{R}^d)})$, which is by Lemma~\ref{LemmaWkpwSep} separable. Indeed, since $\rho \in C^k_{pol,\gamma}(\mathbb{R})$ is $k$-times differentiable, it follows for every fixed $\omega \in \Omega$ and $\alpha \in \mathbb{N}^m_{0,k}$ that $U \ni u \mapsto \partial_\alpha R_n(\omega) = y(\omega) \rho^{(\vert \alpha \vert)}\left( a_n(\omega)^\top u - b_n(\omega) \right) a_n(\omega)^\alpha \in \mathbb{R}^d$ is $\mathcal{L}(U)/\mathcal{B}(\mathbb{R}^d)$-measurable. Moreover, by using that $\rho \in C^k_{pol,\gamma}(\mathbb{R})$, i.e.~that $\left\vert \rho^{(j)}(s) \right\vert \leq \Vert \rho \Vert_{C^k_{pol,\gamma}(\mathbb{R})} (1+\vert s \vert)^\gamma$ for any $j = 0,...,k$ and $s \in \mathbb{R}$, the inequality $1+\left\vert a_n(\omega)^\top u - b_n(\omega) \right\vert \leq 1+\Vert a_n(\omega) \Vert \Vert u \Vert + \vert b_n(\omega) \vert \leq (1+\Vert a_n(\omega) \Vert + \vert b_n(\omega) \vert) (1 + \Vert u \Vert)$ for any $u \in \mathbb{R}^m$, and that $C^{(\gamma,p)}_{U,w} := \left( \int_U (1+\Vert u \Vert)^{\gamma p} w(u) du \right)^{1/p} > 0$ is finite, we conclude that
	\begin{equation*}
		\begin{aligned}
			& \left\Vert R_n(\omega) \right\Vert_{W^{k,p}(U,\mathcal{L}(U),w;\mathbb{R}^d)}^p = \sum_{\alpha \in \mathbb{N}^m_{0,k}} \int_U \left\Vert y(\omega) \rho^{(\vert \alpha \vert)}\left( a_n(\omega)^\top u - b_n(\omega) \right) a_n(\omega)^\alpha \right\Vert^p w(u) du \\
			& \quad\quad \leq \left( \sum_{\alpha \in \mathbb{N}^m_{0,k}} \left\Vert y(\omega) a_n(\omega)^\alpha \right\Vert^p \right) \int_U \left( 1 + \left\vert a_n(\omega)^\top u - b_n(\omega) \right\vert \right)^{\gamma p} w(u) du \\
			& \quad\quad \leq \left( \sum_{\alpha \in \mathbb{N}^m_{0,k}} \left\Vert y(\omega) a_n(\omega)^\alpha \right\Vert^p \right) (1 + \Vert a_n(\omega) \Vert + \vert b_n(\omega) \vert)^{\gamma p} \int_U (1+\Vert u \Vert)^{\gamma p} w(u) du < \infty.
		\end{aligned}
	\end{equation*}
	This shows that $R_n(\omega) \in W^{k,p}(U,\mathcal{L}(U),w;\mathbb{R}^d)$ for all $\omega \in \Omega$.
	
	Finally, in order to show that the map \eqref{EqLemmaWkpwStrongMbl1} is strongly $(\mathbb{P},\mathcal{F}_{a,b})$-measurable, we use that $\Omega \ni \omega \mapsto (y(\omega),a_n(\omega),b_n(\omega)) \in \mathbb{R}^d \times \mathbb{R}^m \times \mathbb{R}$ is by definition $\mathcal{F}_{a,b}/\mathbb{B}(\mathbb{R}^d \times \mathbb{R}^m \times \mathbb{R})$-measurable and that $\mathbb{R}^d \times \mathbb{R}^m \times \mathbb{R} \ni (y,a,b) \mapsto y \rho\big( a^\top \cdot - b \big) \in W^{k,p}(U,\mathcal{L}(U),w;\mathbb{R}^d)$ is by Lemma~\ref{LemmaWkpwCont} continuous to conclude that the concatenation \eqref{EqLemmaWkpwStrongMbl1} is $\mathcal{F}_{a,b}/\mathcal{B}(W^{k,p}(U,\mathcal{L}(U),w;\mathbb{R}^d))$-measurable, where $\mathcal{B}(W^{k,p}(U,\mathcal{L}(U),w;\mathbb{R}^d))$ denotes the Borel $\sigma$-algebra of $(W^{k,p}(U,\mathcal{L}(U),w;\mathbb{R}^d),\Vert \cdot \Vert_{W^{k,p}(U,\mathcal{L}(U),w;\mathbb{R}^d)})$. Since $(W^{k,p}(U,\mathcal{L}(U),w;\mathbb{R}^d,\Vert \cdot \Vert_{W^{k,p}(U,\mathcal{L}(U),w;\mathbb{R}^d})$ is by Lemma~\ref{LemmaWkpwSep} separable, we can apply \cite[Theorem 1.1.6+1.1.20]{hytoenen16} to conclude that \eqref{EqLemmaWkpwStrongMbl1} is strongly $(\mathbb{P},\mathcal{F}_{a,b})$-measurable.
\end{proof}

\subsubsection{Proof of Theorem~\ref{ThmApproxRates}}
\label{AppApproxRates}

In this section, we prove the approximation rates in Theorem~\ref{ThmApproxRates}. Let us first sketch the main ideas of the proof. For some fixed $f \in W^{k,p}(U,\mathcal{L}(U),w;\mathbb{R}^d) \cap \mathbb{B}^{k,\gamma}_\psi(U;\mathbb{R}^d)$ and $N \in \mathbb{N}$, we use the randomized neuron $R_n: \Omega \rightarrow W^{k,p}(U,\mathcal{L}(U),w,\mathbb{R}^d)$ in \eqref{EqLemmaWkpwStrongMbl1} with a particular linear readout. Then, by using the integral representation in Proposition~\ref{PropIntRepr} implying that $f = \mathbb{E}[R_n]$, a symmetrization argument with Rademacher averages, and the Banach space type of $W^{k,p}(U,\mathcal{L}(U),w,\mathbb{R}^d)$, we obtain
\begin{equation*}
	\begin{aligned}
		\mathbb{E}\left[ \left\Vert f - \frac{1}{N} \sum_{n=1}^N R_n \right\Vert_{W^{k,p}(U,\mathcal{L}(U),w,\mathbb{R}^d)}^2 \right]^\frac{1}{2} & = \frac{1}{N} \mathbb{E}\left[ \left\Vert \sum_{n=1}^N \left( \mathbb{E}\left[ R_n \right] - R_n \right) \right\Vert_{W^{k,p}(U,\mathcal{L}(U),w,\mathbb{R}^d)}^2 \right]^\frac{1}{2} \\
		& \leq C \frac{\left\Vert R_n \right\Vert_{L^2(\Omega,\mathcal{F},\mathbb{P};W^{k,p}(U,\mathcal{L}(U),w;\mathbb{R}^d))}}{N^{1-\frac{1}{\min(2,p)}}},
	\end{aligned}
\end{equation*}
where $C > 0$ is a constant and where $\Vert R_n \Vert_{L^2(\Omega,\mathcal{F},\mathbb{P};W^{k,p}(U,\mathcal{L}(U),w;\mathbb{R}^d))}$ can be bounded by $\Vert f \Vert_{\mathbb{B}^{k,\gamma}_\psi(U;\mathbb{R}^d)}$. Hence, there exists some $\omega \in \Omega$ such that the neural network $\varphi_N := \frac{1}{N} \sum_{n=1}^N R_n(\omega) \in \mathcal{NN}^\rho_{U,d}$ satisfies
\begin{equation*}
	\begin{aligned}
		\Vert f - \varphi \Vert_{W^{k,p}(U,\mathcal{L}(U),w,\mathbb{R}^d)} & \leq \mathbb{E}\left[ \left\Vert f - \frac{1}{N} \sum_{n=1}^N R_n \right\Vert_{W^{k,p}(U,\mathcal{L}(U),w,\mathbb{R}^d)}^2 \right]^\frac{1}{2} \\
		& \leq C \frac{\left\Vert R_n \right\Vert_{L^2(\Omega,\mathcal{F},\mathbb{P};W^{k,p}(U,\mathcal{L}(U),w;\mathbb{R}^d))}}{N^{1-\frac{1}{\min(2,p)}}}.
	\end{aligned}
\end{equation*}

\begin{proof}[Proof of Theorem~\ref{ThmApproxRates}]
	Fix $f \in W^{k,p}(U,\mathcal{L}(U),w;\mathbb{R}^d) \cap \mathbb{B}^{k,\gamma}_\psi(U;\mathbb{R}^d)$ and $N \in \mathbb{N}$. Then, by definition of $\mathbb{B}^{k,\gamma}_\psi(U;\mathbb{R}^d)$, there exists some $g \in L^1(\mathbb{R}^m,\mathcal{L}(\mathbb{R}^m),du;\mathbb{R}^d)$ with $\widehat{g} \in L^1(\mathbb{R}^m,\mathcal{L}(\mathbb{R}^m),du;\mathbb{C}^d)$ such that
	\begin{equation}
		\label{EqThmApproxRatesProof0}
		\left( \int_{\mathbb{R}^m} \int_{\mathbb{R}} \left( 1 + \Vert a \Vert^2 \right)^{\gamma+k+\frac{m+1}{2}} \left( 1 + \vert b \vert^2 \right)^{\gamma+1} \Vert (\mathfrak{R}_\psi g)(a,b) \Vert^2 db da \right)^\frac{1}{2} \leq 2 \Vert f \Vert_{\mathbb{B}^{k,\gamma}_\psi(U;\mathbb{R}^d)}.
	\end{equation}
	From this, we define for every $n = 1,...,N$ the map
	\begin{equation}
		\label{EqThmApproxRatesProof1}
		\Omega \ni \omega \quad \mapsto \quad R_n(\omega) := y_n(\omega) \rho\left( a_n(\omega)^\top \cdot - b_n(\omega) \right) \in W^{k,p}(U,\mathcal{L}(U),w,\mathbb{R}^d)
	\end{equation}
	with
	\begin{equation}
		\label{EqThmApproxRatesProof2}
		\Omega \ni \omega \quad \mapsto \quad y_n(\omega) := \re\left( \frac{(\mathfrak{R}_\psi g)(a_n(\omega),b_n(\omega))}{C^{(\psi,\rho)}_m p_a(a_n(\omega)) p_b(b_n(\omega))} \right) \in \mathbb{R}^d,
	\end{equation}
	where $p_a: \mathbb{R}^m \rightarrow (0,\infty)$ and $p_b: \mathbb{R} \rightarrow (0,\infty)$ denote the probability density function of the (multivariate) Student's $t$-distributions$^{\text{\ref{FootnoteTDistr}}}$. Then, by using that $\mathfrak{R}_\psi: \mathbb{R}^m \times \mathbb{R} \rightarrow \mathbb{C}^d$ is continuous (see Lemma~\ref{LemmaRidgeletCont}), we observe that $y_n: \Omega \rightarrow \mathbb{R}^d$ is $\mathcal{F}_{a,b}/\mathcal{B}(\mathbb{R}^d)$-measurable. Hence, we can apply Lemma~\ref{LemmaWkpwStrongMbl} to conclude that $R_n: \Omega \rightarrow W^{k,p}(U,\mathcal{L}(U),w;\mathbb{R}^d)$ is a strongly $(\mathbb{P},\mathcal{F}_{a,b})$-measurable map with values in the separable Banach space $(W^{k,p}(U,\mathcal{L}(U),w;\mathbb{R}^d),\Vert \cdot \Vert_{W^{k,p}(U,\mathcal{L}(U),w;\mathbb{R}^d)})$.
	
	Now, we show $R_n \in L^2(\Omega,\mathcal{F}_{a,b},\mathbb{P};W^{k,p}(U,\mathcal{L}(U),w;\mathbb{R}^d))$ and $\mathbb{E}[R_n] = f \in W^{k,p}(U,\mathcal{L}(U),w;\mathbb{R}^d)$. To this end, we use that $\rho \in C^k_{pol,\gamma}(\mathbb{R})$, i.e.~that $\left\vert \rho^{(j)}(s) \right\vert \leq \Vert \rho \Vert_{C^k_{pol,\gamma}(\mathbb{R})} (1+\vert s \vert)^\gamma$ for any $j = 0,...,k$ and $s \in \mathbb{R}$, the inequality $1+\left\vert a^\top u - b \right\vert \leq 1+\Vert a \Vert \Vert u \Vert + \vert b \vert \leq (1+\Vert a \Vert) (1+\vert b \vert) (1+\Vert u \Vert)$ for any $a,u \in \mathbb{R}^m$ and $b \in \mathbb{R}$, twice the inequality $(x+y)^2 \leq 2\big( x^2 + y^2 \big)$ for any $x,y \in [0,\infty)$, and the finite constant $C^{(\gamma,p)}_{U,w} > 0$ to conclude for every $a \in \mathbb{R}^m$, $b \in \mathbb{R}$, and $j=0,...,k$ that
	\begin{equation}
		\label{EqThmApproxRatesProof3}
		\begin{aligned}
			\left( \int_U \left\vert \rho^{(j)}\left( a^\top u - b \right) \right\vert^p w(u) du \right)^\frac{1}{p} & \leq \Vert \rho \Vert_{C^k_{pol,\gamma}(\mathbb{R})} \left( \int_U \left( 1 + \left\vert a^\top u - b \right\vert \right)^{\gamma p} w(u) du \right)^\frac{1}{p} \\
			& \leq \Vert \rho \Vert_{C^k_{pol,\gamma}(\mathbb{R})} (1+\Vert a \Vert)^\gamma (1+\vert b \vert)^\gamma \left( \int_U \left( 1 + \Vert u \Vert \right)^{\gamma p} w(u) du \right)^\frac{1}{p} \\
			& \leq 4 \Vert \rho \Vert_{C^k_{pol,\gamma}(\mathbb{R})} \left( 1 + \Vert a \Vert^2 \right)^\frac{\gamma}{2} \left( 1 + \vert b \vert^2 \right)^\frac{\gamma}{2} C^{(\gamma,p)}_{U,w}.
		\end{aligned}
	\end{equation}
	Hence, by using the inequality $\big\vert a^\alpha \big\vert := \prod_{l=1}^m \vert a_l \vert^{\alpha_l} \leq \big( 1+\Vert a \Vert^2 \big)^{\vert \alpha \vert/2} \leq \big( 1+\Vert a \Vert^2 \big)^{k/2}$ for any $\alpha \in \mathbb{N}^m_{0,k}$ and $a \in \mathbb{R}^m$, the inequality \eqref{EqThmApproxRatesProof3}, that $\big\vert \mathbb{N}^m_{0,k} \big\vert = \sum_{j=0}^k m^j \leq 2m^k$, and the inequality \eqref{EqThmApproxRatesProof1}, we obtain that
	\begin{equation}
		\label{EqThmApproxRatesProof4}
		\begin{aligned}
			& \Vert R_n \Vert_{L^2(\Omega,\mathcal{F},\mathbb{P};W^{k,p}(U,\mathcal{L}(U),w;\mathbb{R}^d))} = \mathbb{E}\left[ \left\Vert y_n \rho\left( a_n^\top \cdot \, - b_n \right) \right\Vert_{W^{k,p}(U,\mathcal{L}(U),w;\mathbb{R}^d)}^2 \right]^\frac{1}{2} \\
			& \quad\quad = \mathbb{E}\left[ \left( \sum_{\alpha \in \mathbb{N}^m_{0,k}} \int_U \left\Vert \partial_\alpha \left( y_n \rho\left( a_n^\top u - b_n \right) \right) \right\Vert^p du \right)^\frac{2}{p} \right]^\frac{1}{2} \\
			& \quad\quad = \mathbb{E}\left[ \left( \sum_{\alpha \in \mathbb{N}^m_{0,k}} \left\Vert \re\left( \frac{a_n^\alpha (\mathfrak{R}_\psi g)(a_n,b_n)}{C^{(\psi,\rho)}_m p_a(a_n) p_b(b_n)} \right) \right\Vert^p \int_U \left\vert \rho^{(\vert\alpha\vert)}\left( a_n^\top u - b_n \right) \right\vert^p du \right)^\frac{2}{p} \right]^\frac{1}{2} \\
			& \quad\quad \leq 4 \Vert \rho \Vert_{C^k_{pol,\gamma}(\mathbb{R})} \frac{C^{(\gamma,p)}_{U,w} \left\vert \mathbb{N}^m_{0,k} \right\vert^\frac{1}{p}}{\left\vert C^{(\psi,\rho)}_m \right\vert} \mathbb{E}\left[ \frac{\left( 1 + \Vert a_n \Vert^2 \right)^{\gamma+k} \left( 1 + \vert b_n \vert^2 \right)^\gamma}{p_a(a_n)^2 p_b(b_n)^2} \Vert (\mathfrak{R}_\psi g)(a_n,b_n) \Vert^2 \right]^\frac{1}{2} \\
			& \quad\quad \leq 2^{3+\frac{1}{p}} \Vert \rho \Vert_{C^k_{pol,\gamma}(\mathbb{R})} \frac{C^{(\gamma,p)}_{U,w} m^\frac{k}{p}}{\left\vert C^{(\psi,\rho)}_m \right\vert} \left( \int_{\mathbb{R}^m} \int_{\mathbb{R}} \frac{\left( 1 + \Vert a \Vert^2 \right)^{\gamma+k} \left( 1 + \vert b \vert^2 \right)^\gamma}{p_a(a)^2 p_b(b)^2} \Vert (\mathfrak{R}_\psi g)(a,b) \Vert^2 p_a(a) p_b(b) db da \right)^\frac{1}{2} \\
			& \quad\quad \leq 2^{3+\frac{1}{p}} \Vert \rho \Vert_{C^k_{pol,\gamma}(\mathbb{R})} \frac{C^{(\gamma,p)}_{U,w} m^\frac{k}{p}}{\left\vert C^{(\psi,\rho)}_m \right\vert} \\
			& \quad\quad\quad\quad \cdot \left( \frac{\pi^\frac{m+1}{2}}{\Gamma\left( \frac{m+1}{2} \right)} \pi \int_{\mathbb{R}^m} \int_{\mathbb{R}} \left( 1 + \Vert a \Vert^2 \right)^{\gamma+k+\frac{m+1}{2}} \left( 1 + \vert b \vert^2 \right)^{\gamma+1} \Vert (\mathfrak{R}_\psi g)(a,b) \Vert^2 db da \right)^\frac{1}{2} \\
			& \quad\quad \leq 2^{4+\frac{1}{p}} \pi \Vert \rho \Vert_{C^k_{pol,\gamma}(\mathbb{R})} \frac{C^{(\gamma,p)}_{U,w} m^\frac{k}{p} \pi^\frac{m+1}{4}}{\left\vert C^{(\psi,\rho)}_m \right\vert \Gamma\left( \frac{m+1}{2} \right)^\frac{1}{2}} \Vert f \Vert_{\mathbb{B}^{k,\gamma}_\psi(U;\mathbb{R}^d)} < \infty,
		\end{aligned}
	\end{equation}
	which shows that $R_n \in L^2(\Omega,\mathcal{F}_{a,b},\mathbb{P};W^{k,p}(U,\mathcal{L}(U),w;\mathbb{R}^d))$. Moreover, by using the probability density functions $p_a: \mathbb{R}^m \rightarrow (0,\infty)$ and $p_b: \mathbb{R} \rightarrow (0,\infty)$, Proposition~\ref{PropIntRepr}, and that $f = g$ a.e.~on $U$, it follows for a.e.~$u \in U$ that
	\begin{equation*}
		\begin{aligned}
			\mathbb{E}[R_n(u)] & = \mathbb{E}\left[ \re\left( \frac{(\mathfrak{R}_\psi g)(a_n,b_n)}{C^{(\psi,\rho)}_m p_a(a_n) p_b(b_n)} \right) \rho\left( a_n^\top u - b_n \right) \right] \\
			& = \int_{\mathbb{R}^m} \int_{\mathbb{R}} \re\left( \frac{(\mathfrak{R}_\psi g)(a,b)}{C^{(\psi,\rho)}_m p_a(a) p_b(b)} \right) \rho\left( a^\top u - b \right) p_a(a) p_b(b) db da \\
			& = \re\left( \frac{1}{C^{(\psi,\rho)}_m} \int_{\mathbb{R}^m} \int_{\mathbb{R}} (\mathfrak{R}_\psi g)(a,b) \rho\left( a^\top u - b \right) db da \right) \\
			& = \re\left( \frac{1}{C^{(\psi,\rho)}_m} C^{(\psi,\rho)}_m g(u) \right) = g(u) = f(u).
		\end{aligned}
	\end{equation*}
	Moreover, if $k \geq 1$, we use integration by parts to conclude for every $\alpha \in \mathbb{N}^m_{0,k}$ and $h \in C^\infty_c(U)$ that
	\begin{equation*}
		\begin{aligned}
			\int_U \partial_\alpha \mathbb{E}[R_n(u)] h(u) du & = (-1)^{\vert \alpha \vert} \int_U \mathbb{E}[R_n(u)](u) \partial_\alpha h(u) du = (-1)^{\vert \alpha \vert} \int_U f(u) \partial_\alpha h(u) du \\
			& = \int_U \partial_\alpha f(u) h(u) du.
		\end{aligned}
	\end{equation*}
	This shows for every $\alpha \in \mathbb{N}^m_{0,k}$ and a.e.~$u \in U$ that $\partial_\alpha \mathbb{E}[R_n](u) = \partial_\alpha \mathbb{E}[R_n(u)] = \partial_\alpha f(u)$, which implies that $f = \mathbb{E}[R_n] \in W^{k,p}(U,\mathcal{L}(U),w;\mathbb{R}^d)$.
	
	Finally, we use that $f = \mathbb{E}\big[ R_n \big] \in W^{k,p}(U,\mathcal{L}(U),w;\mathbb{R}^d)$, the right-hand side of \cite[Lemma~6.3]{ledoux91} for the independent mean-zero random variables $\big( \mathbb{E}\big[ R_n \big] - R_n \big)_{n=1,...,N}$ (with i.i.d.~$(\epsilon_n)_{n=1,...,N}$ satisfying $\mathbb{P}[\epsilon_n = \pm 1] = 1/2$ being independent of $\big( \mathbb{E}\big[ R_n \big] - R_n \big)_{n=1,...,N}$), the Kahane-Khintchine inequality in \cite[Theorem~3.2.23]{hytoenen16} with constant $\kappa_{2,\min(2,p)} > 0$ depending only on $p \in [1,\infty)$, that $(W^{k,p}(U,\mathcal{L}(U),w;\mathbb{R}^d),\Vert \cdot \Vert_{W^{k,p}(U,\mathcal{L}(U),w;\mathbb{R}^d)})$ is by Lemma~\ref{LemmaBanachSpaceType} a Banach space of type $\min(2,p) \in (1,2]$ (with constant $\widetilde{C}_p := C_{W^{k,p}(U,\mathcal{L}(U),w;\mathbb{R}^d)} > 0$ depending only on $p \in (1,2]$), that $(R_n)_{n = 1,...,N} \sim R_1$ are identically distributed, and Jensen's inequality, we obtain that
	\begin{equation*}
		\begin{aligned}
			& \mathbb{E}\left[ \left\Vert f - \frac{1}{N} \sum_{n=1}^N R_n \right\Vert_{W^{k,p}(U,\mathcal{L}(U),w;\mathbb{R}^d)}^2 \right]^\frac{1}{2} = \frac{1}{N} \mathbb{E}\left[ \left\Vert \sum_{n=1}^N \left( \mathbb{E}\left[ R_n \right] - R_n \right) \right\Vert_{W^{k,p}(U,\mathcal{L}(U),w;\mathbb{R}^d)}^2 \right]^\frac{1}{2} \\
			& \quad\quad \leq \frac{2}{N} \mathbb{E}\left[ \left\Vert \sum_{n=1}^N \epsilon_n \left( \mathbb{E}\left[ R_n \right] - R_n \right) \right\Vert_{W^{k,p}(U,\mathcal{L}(U),w;\mathbb{R}^d)}^2 \right]^\frac{1}{2} \\
			& \quad\quad \leq \frac{2 \kappa_{2,\min(2,p)}}{N} \mathbb{E}\left[ \left\Vert \sum_{n=1}^N \epsilon_n \left( \mathbb{E}\left[ R_n \right] - R_n \right) \right\Vert_{W^{k,p}(U,\mathcal{L}(U),w;\mathbb{R}^d)}^{\min(2,p)} \right]^\frac{1}{\min(2,p)} \\
			& \quad\quad \leq \frac{2 \widetilde{C}_p \kappa_{2,\min(2,p)}}{N} \left( \sum_{n=1}^N \mathbb{E}\left[ \left\Vert \mathbb{E}\left[ R_n \right] - R_n \right\Vert_{W^{k,p}(U,\mathcal{L}(U),w;\mathbb{R}^d)}^{\min(2,p)} \right] \right)^\frac{1}{\min(2,p)} \\
			& \quad\quad = \frac{2 \widetilde{C}_p \kappa_{2,\min(2,p)}}{N^{1-\frac{1}{\min(2,p)}}} \mathbb{E}\left[ \left\Vert \mathbb{E}\left[ R_1 \right] - R_1 \right\Vert_{W^{k,p}(U,\mathcal{L}(U),w;\mathbb{R}^d)}^{\min(2,p)} \right]^\frac{1}{\min(2,p)} \\
			& \quad\quad \leq \frac{2 \widetilde{C}_p \kappa_{2,\min(2,p)}}{N^{1-\frac{1}{\min(2,p)}}} \mathbb{E}\left[ \left\Vert \mathbb{E}\left[ R_1 \right] - R_1 \right\Vert_{W^{k,p}(U,\mathcal{L}(U),w;\mathbb{R}^d)}^2 \right]^\frac{1}{2}.
		\end{aligned}
	\end{equation*}
	Hence, by using this, Jensen's inequality, Minkowski's inequality together with \cite[Proposition~1.2.2]{hytoenen16}, the inequality \eqref{EqThmApproxRatesProof4}, and the constant $C_p := 4 \widetilde{C}_p \kappa_{2,\min(2,p)} \pi > 0$ (depending only on $p \in [1,\infty)$), it follows that
	\begin{equation*}
		\begin{aligned}
			& \mathbb{E}\left[ \left\Vert f - \frac{1}{N} \sum_{n=1}^N R_n \right\Vert_{W^{k,p}(U,\mathcal{L}(U),w;\mathbb{R}^d)}^2 \right]^\frac{1}{2} \leq \frac{2 \widetilde{C}_p \kappa_{2,\min(2,p)}}{N^{1-\frac{1}{\min(2,p)}}} \mathbb{E}\left[ \left\Vert \mathbb{E}\left[ R_1 \right] - R_1 \right\Vert_{W^{k,p}(U,\mathcal{L}(U),w;\mathbb{R}^d)}^2 \right]^\frac{1}{2} \\
			& \quad\quad \leq \frac{4 \widetilde{C}_p \kappa_{2,\min(2,p)}}{N^{1-\frac{1}{\min(2,p)}}} \left\Vert R_1 \right\Vert_{L^2(\Omega,\mathcal{F},\mathbb{P};W^{k,p}(U,\mathcal{L}(U),w;\mathbb{R}^d))} \\
			& \quad\quad \leq \frac{4 \widetilde{C}_p \kappa_{2,\min(2,p)}}{N^{1-\frac{1}{\min(2,p)}}} 2^{4+\frac{1}{p}} \pi \Vert \rho \Vert_{C^k_{pol,\gamma}(\mathbb{R})} \frac{C^{(\gamma,p)}_{U,w} m^\frac{k}{p} \pi^\frac{m+1}{4}}{\left\vert C^{(\psi,\rho)}_m \right\vert \Gamma\left( \frac{m+1}{2} \right)^\frac{1}{2}} \Vert f \Vert_{\mathbb{B}^{k,\gamma}_\psi(U;\mathbb{R}^d)} \\
			& \quad\quad \leq C_p \Vert \rho \Vert_{C^k_{pol,\gamma}(\mathbb{R})} \frac{C^{(\gamma,p)}_{U,w} m^\frac{k}{p} \pi^\frac{m+1}{4}}{\left\vert C^{(\psi,\rho)}_m \right\vert \Gamma\left( \frac{m+1}{2} \right)^\frac{1}{2}} \frac{\Vert f \Vert_{\mathbb{B}^{k,\gamma}_\psi(U;\mathbb{R}^d)}}{N^{1-\frac{1}{\min(2,p)}}}.
		\end{aligned}
	\end{equation*}	
	Thus, there exists some $\omega \in \Omega$ such that $\varphi_N := \frac{1}{N} \sum_{n=1}^N R_n(\omega) \in \mathcal{NN}^\rho_{U,d}$ satisfies
	\begin{equation*}
		\begin{aligned}
			\left\Vert f - \varphi_N \right\Vert_{W^{k,p}(U,\mathcal{L}(U),w;\mathbb{R}^d)} & = \left\Vert f - \Phi_N(\omega) \right\Vert_{W^{k,p}(U,\mathcal{L}(U),w;\mathbb{R}^d)} \\
			& \leq \mathbb{E}\left[ \left\Vert f - \Phi_N \right\Vert_{W^{k,p}(U,\mathcal{L}(U),w;\mathbb{R}^d)}^2 \right]^\frac{1}{2} \\
			& \leq C_p \Vert \rho \Vert_{C^k_{pol,\gamma}(\mathbb{R})} \frac{C^{(\gamma,p)}_{U,w} m^\frac{k}{p} \pi^\frac{m}{4}}{\left\vert C^{(\psi,\rho)}_m \right\vert \Gamma\left( \frac{m+1}{2} \right)^\frac{1}{2}} \frac{\Vert f \Vert_{\mathbb{B}^{k,\gamma}_\psi(U;\mathbb{R}^d)}}{N^{1-\frac{1}{\min(2,p)}}},
		\end{aligned}
	\end{equation*}
	which completes the proof.
\end{proof}

\subsubsection{Proof of Example~\ref{ExAdm}, Lemma~\ref{LemmaWeight}, and Proposition~\ref{PropConst}}
\label{AppConstants}

\begin{proof}[Proof of Example~\ref{ExAdm}]
	First, we observe that $\rho \in C^k_{pol,\gamma}(\mathbb{R})$ is in each case \ref{ExAdm1}-\ref{ExAdm4} of polynomial growth, which ensures that $\rho \in C^k_{pol,\gamma}(\mathbb{R})$ induces $\big( g \mapsto T_\rho(g) := \int_{\mathbb{R}} \rho(s) g(s) ds \big) \in \mathcal{S}'(\mathbb{R};\mathbb{C})$ (see \cite[p.~332]{folland92}).
	
	For \ref{ExAdm2}, we recall that $\tanh'(\xi) = \cosh(\xi)^{-2}$ holds true for all $\xi \in \mathbb{R}$. Moreover, the Fourier transform of the function $\big( s \mapsto h(s) := \frac{\pi s}{\sinh(\pi s/2)} \big) \in L^1(\mathbb{R},\mathcal{L}(\mathbb{R}),du)$ is for every $\xi \in \mathbb{R}$ given by
	\begin{equation}
		\label{EqExNonPolyProof0}
		\widehat{h}(\xi) = \frac{2\pi}{\cosh(\xi)^2} = 2\pi \tanh'(\xi).
	\end{equation}
	Then, by using $\big( g \mapsto \big( \id \cdot \widehat{T_{\tanh}} \big)(g) := \widehat{T_{\tanh}}(\id \cdot g) \big) \in \mathcal{S}'(\mathbb{R};\mathbb{C})$, \cite[Equation~9.31]{folland92} with $\mathbb{R} \ni s \mapsto \id(s) := s \in \mathbb{R}$, the definition of $\widehat{T_{\tanh}} \in \mathcal{S}'(\mathbb{R};\mathbb{C})$, the identity \eqref{EqExNonPolyProof0}, and the Plancherel theorem in \cite[p.~222]{folland92}, it follows for every $g \in C^\infty_c(\mathbb{R} \setminus \lbrace 0 \rbrace;\mathbb{C})$ that
	\begin{equation}
		\label{EqExNonPolyProof1}
		\begin{aligned}
			\widehat{T_{\tanh}}(\id \cdot g) & = \left( \id \cdot \widehat{T_{\tanh}} \right)(g) = \frac{1}{i} \widehat{T_{\tanh'}}(g) = (-\mathbf{i}) T_{\tanh'}\left( \widehat{g} \right) \\
			& = (-\mathbf{i}) \int_{\mathbb{R}} \tanh'(\xi) \widehat{g}(\xi) d\xi = \frac{-\mathbf{i}}{2\pi} \int_{\mathbb{R}} \overline{\widehat{h}(\xi)} \widehat{g}(\xi) d\xi \\
			& = (-\mathbf{i}) \int_{\mathbb{R}} \overline{h(\xi)} g(\xi) d\xi = \int_{\mathbb{R}} \frac{-\mathbf{i} \pi}{\sinh\left( \pi \xi/2 \right)} (\id \cdot g)(\xi) d\xi.
		\end{aligned}
	\end{equation}
	Hence, $\widehat{T_{\tanh}} \in \mathcal{S}'(\mathbb{R};\mathbb{C})$ coincides on $\mathbb{R} \setminus \lbrace 0 \rbrace$ with $\big( \xi \mapsto f_{\widehat{T_{\tanh}}}(\xi) := \frac{\mathbf{i} \pi}{\sinh(\pi \xi/2)} \big) \in L^1_{loc}(\mathbb{R} \setminus \lbrace 0 \rbrace;\mathbb{C})$.
	
	For \ref{ExAdm1}, we denote by $\big( s \mapsto \sigma(s) := \frac{1}{1+\exp(-s)} \big) \in C^k_{pol,\gamma}(\mathbb{R})$ the sigmoid function and observe that $\sigma(s) = \frac{1}{2} \big( \tanh\left( \frac{s}{2} \right) + 1 \big)$ for all $s \in \mathbb{R}$. Then, by using the linearity of the Fourier transform on $\mathcal{S}'(\mathbb{R};\mathbb{C})$, \cite[Equation~9.30]{folland92}, that $\widehat{T_1}(g) = 2 \pi \delta(g) := 2 \pi g(0)$ for any $g \in \mathcal{S}(\mathbb{R};\mathbb{C})$ (see \cite[Equation~9.35]{folland92}), the identity \eqref{EqExNonPolyProof1}, and the substitution $\xi \mapsto \widetilde{\xi}/2$, it follows for every $g \in C^\infty_c(\mathbb{R} \setminus \lbrace 0 \rbrace;\mathbb{C})$ that
	\begin{equation}
		\label{EqExNonPolyProof2}
		\begin{aligned}
			\widehat{T_\sigma}(g) & = \frac{1}{2} \reallywidehat{T_{\tanh(\frac{\cdot}{2})}}(g) + \frac{1}{2} \widehat{T_1}(g) = \frac{1}{2} \widehat{T_{\tanh}}\left(g\left(\frac{\cdot}{2}\right)\right) + \frac{2\pi}{2} g(0) \\
			& = \frac{1}{2} \int_{\mathbb{R}} \frac{-\mathbf{i} \pi}{\sinh\left( \pi \widetilde{\xi}/2 \right)} g\left( \widetilde{\xi}/2 \right) d\widetilde{\xi} = \int_{\mathbb{R}} \frac{-\mathbf{i} \pi}{\sinh(\pi \xi)} g(\xi) d\xi.
		\end{aligned}
	\end{equation}
	Hence, $\widehat{T_\sigma} \in \mathcal{S}'(\mathbb{R};\mathbb{C})$ coincides on $\mathbb{R} \setminus \lbrace 0 \rbrace$ with $\big( \xi \mapsto f_{\widehat{T_\sigma}}(\xi) := \frac{-\mathbf{i} \pi}{\sinh(\pi \xi)} \big) \in L^1_{loc}(\mathbb{R} \setminus \lbrace 0 \rbrace;\mathbb{C})$.
	
	For \ref{ExAdm3}, we denote by $\big( s \mapsto \sigma^{(-1)}(s) := \ln(1+\exp(s)) \big) \in C^k_{pol,\gamma}(\mathbb{R})$ the softplus function and observe that $\frac{d}{ds} \sigma^{(-1)}(s) = \sigma(s)$ for all $s \in \mathbb{R}$. Then, by using \cite[Equation~9.31]{folland92} with $\mathbb{R} \ni s \mapsto \id(s) := s \in \mathbb{R}$ and the identity \eqref{EqExNonPolyProof2}, it follows for every $g \in C^\infty_c(\mathbb{R} \setminus \lbrace 0 \rbrace;\mathbb{C})$ that
	\begin{equation*}
		\widehat{T_{\sigma^{(-1)}}}(\id \cdot g) = \left( \id \cdot \widehat{T_{\sigma^{(-1)}}} \right)(g) = \frac{1}{\mathbf{i}} \widehat{T_\sigma}(g) = \frac{1}{\mathbf{i}}\int_{\mathbb{R}} \frac{-\mathbf{i} \pi}{\sinh(\pi \xi)} g(\xi) d\xi = \int_{\mathbb{R}} \frac{- \pi}{\xi \sinh(\pi \xi)} (\id \cdot g)(\xi) d\xi.
	\end{equation*}
	Hence, $\widehat{T_{\sigma^{(-1)}}} \in \mathcal{S}'(\mathbb{R};\mathbb{C})$ coincides on $\mathbb{R} \setminus \lbrace 0 \rbrace$ with $\big( \xi \mapsto f_{\widehat{T_{\sigma^{(-1)}}}}(\xi) := \frac{-\pi}{\xi \sinh(\pi \xi)} \big) \in L^1_{loc}(\mathbb{R} \setminus \lbrace 0 \rbrace;\mathbb{C})$.
	
	For \ref{ExAdm4}, we denote by $\big( s \mapsto \text{ReLU}(s) := \max(s,0) \big) \in C^k_{pol,\gamma}(\mathbb{R})$ the ReLU function and observe that $\text{ReLU}(s) = \max(s,0) = \frac{s+\vert s \vert}{2}$ for all $s \in \mathbb{R}$. Moreover, the absolute value $\mathbb{R} \ni s \mapsto \vert s \vert \in \mathbb{R}$ is weakly differentiable with $\frac{d}{ds} \vert s \vert = \sgn(s)$ for all $s \in \mathbb{R}$, where $\sgn(s) := 1$ if $s > 0$, $\sgn(0) := 0$, and $\sgn(s) := -1$ if $s < 0$. Then, by using the linearity of the Fourier transform on $\mathcal{S}'(\mathbb{R};\mathbb{C})$, that $\widehat{T_{\id}}(g) = 2\pi \mathbf{i} \delta'(g) := 2\pi \mathbf{i} g'(0)$ for any $g \in \mathcal{S}(\mathbb{R};\mathbb{C})$ with $\mathbb{R} \ni s \mapsto \id(s) := s \in \mathbb{R}$ (see \cite[Equation~9.35]{folland92}), \cite[Equation~9.31]{folland92}, and \cite[Example~9.4.4]{folland92}, i.e.~that $\widehat{T_{\sgn}}(g) = -2\mathbf{i} \int_{\mathbb{R}} \frac{g(\xi)}{\xi} d\xi$ for any $g \in C^\infty_c(\mathbb{R} \setminus \lbrace 0 \rbrace;\mathbb{C})$, it follows for every $g \in C^\infty_c(\mathbb{R} \setminus \lbrace 0 \rbrace;\mathbb{C})$ that
	\begin{equation*}
		\begin{aligned}
			\widehat{T_{\text{ReLU}}}(\id \cdot g) & = \frac{1}{2} \widehat{T_{\id}}(\id \cdot g) + \frac{1}{2} \widehat{T_{\vert \cdot \vert}}(\id \cdot g) = \frac{2\pi \mathbf{i}}{2} (\id \cdot g)'(0) + \frac{1}{2} \left( \id \cdot \widehat{T_{\vert \cdot \vert}} \right)(g) \\
			& = \frac{1}{2\mathbf{i}} \widehat{T_{\sgn}}(g) = \frac{-2\mathbf{i}}{2\mathbf{i}} \int_{\mathbb{R}} \frac{g(\xi)}{\xi} d\xi = \int_{\mathbb{R}} \frac{-1}{\xi^2} (\id \cdot g)(\xi) d\xi.
		\end{aligned}
	\end{equation*}
	Hence, $\widehat{T_{\text{ReLU}}} \in \mathcal{S}'(\mathbb{R};\mathbb{C})$ coincides on $\mathbb{R} \setminus \lbrace 0 \rbrace$ with $\big( \xi \mapsto f_{\widehat{T_{\text{ReLU}}}}(\xi) := -\frac{1}{\xi^2} \big) \in L^1_{loc}(\mathbb{R} \setminus \lbrace 0 \rbrace;\mathbb{C})$.
	
	Now, for each case \ref{ExAdm1}-\ref{ExAdm4}, we fix some $m \in \mathbb{N}$ and $\psi \in \mathcal{S}_0(\mathbb{R};\mathbb{C})$ with non-negative $\widehat{\psi} \in C^\infty_c(\mathbb{R})$ such that $\supp(\widehat{\psi}) = [\zeta_1,\zeta_2]$ for some $0 < \zeta_1 < \zeta_2 < \infty$. Then, by using that $\widehat{\psi} \in C^\infty_c(\mathbb{R})$ is non-negative, it follows that
	\begin{equation*}
		C^{(\psi,\rho)}_m = (2\pi)^{m-1} \int_{\mathbb{R} \setminus \lbrace 0 \rbrace} \frac{\overline{\widehat{\psi}(\xi)} f_{\widehat{T_\rho}}(\xi)}{\vert \xi \vert^m} d\xi = (2\pi)^{m-1} \int_{\zeta_1}^{\zeta_2} \frac{\overline{\widehat{\psi}(\xi)} f_{\widehat{T_\rho}}(\xi)}{\vert \xi \vert^m} d\xi \neq 0.
	\end{equation*}
	This shows that $(\psi,\rho) \in \mathcal{S}_0(\mathbb{R};\mathbb{C}) \times C^k_{pol,\gamma}(\mathbb{R})$ is $m$-admissible. Moreover, in each case \ref{ExAdm1}-\ref{ExAdm4}, we define the constant $C_{\psi,\rho} := (2\pi)^{-1} \big\vert \int_{\zeta_1}^{\zeta_2} \overline{\widehat{\psi}(\xi)} f_{\widehat{T_\rho}}(\xi) d\xi \big\vert$ (independent of $m \in \mathbb{N}$) to conclude that
	\begin{equation*}
		\left\vert C^{(\psi,\rho)}_m \right\vert = (2\pi)^{m-1} \left\vert \int_{\zeta_1}^{\zeta_2} \frac{\overline{\widehat{\psi}(\xi)} f_{\widehat{T_\rho}}(\xi)}{\vert \xi \vert^m} d\xi \right\vert \geq \left\vert \int_{\zeta_1}^{\zeta_2} \frac{\overline{\widehat{\psi}(\xi)} f_{\widehat{T_\rho}}(\xi)}{2\pi} d\xi \right\vert \left( \frac{2\pi}{\zeta_2} \right)^m = C_{\psi,\rho} \left( \frac{2\pi}{\zeta_2} \right)^m,
	\end{equation*}
	which completes the proof.
\end{proof}

\begin{proof}[Proof of Lemma~\ref{LemmaWeight}]
	Let $U \ni u \mapsto w(u) := \prod_{l=1}^m w_0(u_l) \in [0,\infty)$ be a weight, where $w_0: \mathbb{R} \rightarrow [0,\infty)$ satisfies $\int_{\mathbb{R}} w_0(s) ds = 1$ and $C^{(\gamma,p)}_{\mathbb{R},w_0} := \big( \int_{\mathbb{R}} (1+\vert s \vert)^{\gamma p} w_0(s) ds \big)^{1/p} < \infty$. Then, by using that $1+\Vert u \Vert \leq 1+\sum_{l=1}^m \vert u_l \vert \leq \sum_{l=1}^m (1+\vert u_l \vert)$ for any $u := (u_1,...,u_m)^\top \in \mathbb{R}^m$, that $(x_1+...+x_m)^{\gamma p} \leq m^{\gamma p} \left( x_1^{\gamma p} + ... + x_m^{\gamma p} \right)$ for any $x_1,...,x_m \geq 0$, and Fubini's theorem, it follows that
	\begin{equation*}
		\begin{aligned}
			C^{(\gamma,p)}_{U,w} & = \left( \int_U (1+\Vert u \Vert)^{\gamma p} w(u) du \right)^\frac{1}{p} \\
			& \leq \left( \int_U \left( \sum_{l=1}^m (1+\vert u_l \vert) \right)^{\gamma p} w(u) du \right)^\frac{1}{p} \\
			& \leq m^\gamma \left( \sum_{l=1}^m \int_{\mathbb{R}^m} \left( 1 + \vert u_l \vert \right)^{\gamma p} \prod_{i=1}^m w_0(u_i) du \right)^\frac{1}{p} \\
			& \leq m^\gamma \Bigg( \sum_{l=1}^m \bigg( \underbrace{\int_{\mathbb{R}} \left( 1 + \vert u_l \vert \right)^{\gamma p} w_0(u_l) du_l}_{= \big( C^{(\gamma,p)}_{\mathbb{R},w_0} \big)^p} \bigg) \prod_{i=1 \atop i \neq l}^m \underbrace{\int_{\mathbb{R}^m} w_0(u_i) du_i}_{=1} \Bigg)^\frac{1}{p} \\
			& \leq C^{(\gamma,p)}_{\mathbb{R},w_0} m^{\gamma+\frac{1}{p}},
		\end{aligned}
	\end{equation*}
	which completes the proof.
\end{proof}

\begin{proof}[Proof of Proposition~\ref{PropConst}]
	Fix some $f \in L^1(\mathbb{R}^m,\mathcal{L}(\mathbb{R}^m),du;\mathbb{R}^d)$ with $(\lceil\gamma\rceil+2)$-times differentiable Fourier transform. Then, for any fixed $c \in \lbrace 0, \lceil\gamma\rceil+2 \rbrace$, we use that $(\mathfrak{R}_\psi f)(a,b) = (\widetilde{\mathfrak{R}}_\psi f)(v,s,t)$ for any $(a,b) \in (\mathbb{R}^m \setminus \lbrace 0 \rbrace) \times \mathbb{R}$ with $(v,s,t) := \big( \frac{a}{\Vert a \Vert}, \frac{1}{\Vert a \Vert}, \frac{b}{\Vert a \Vert} \big)$, where $\widetilde{\mathfrak{R}}_\psi f$ is introduced in \eqref{EqDefRidgeletPolar}, the identities \cite[Equation~(36)-(40)]{sonoda17}, $c$-times integration by parts, and the Leibniz product rule together with the chain rule, to conclude for every $(a,b) \in (\mathbb{R}^m \setminus \lbrace 0 \rbrace) \times \mathbb{R}$ that
	\begin{equation}
		\label{EqPropConstProof0}
		\begin{aligned}
			b^c (\mathfrak{R}_\psi f)(a,b) & = \frac{t^c}{s^c} (\widetilde{\mathfrak{R}}_\psi f)(v,s,t) \\
			& = \frac{1}{2\pi} \frac{t^c}{s^c} \int_{\mathbb{R}} \widehat{f}(\xi v) \overline{\widehat{\psi}(\xi s)} e^{\mathbf{i} \xi t} d\xi \\
			& = \frac{1}{2\pi} \frac{(-\mathbf{i})^c}{s^c} \int_{\mathbb{R}} \widehat{f}(\xi v) \overline{\widehat{\psi}(\xi s)} \frac{\partial^c}{\partial \xi^c} \left( e^{\mathbf{i} \xi t} \right) d\xi \\
			& = \frac{1}{2\pi} \frac{\mathbf{i}^c}{s^c} \int_{\mathbb{R}} \frac{\partial^c}{\partial \xi^c} \left( \widehat{f}(\xi v) \overline{\widehat{\psi}(\xi s)} \right) e^{\mathbf{i} \xi t} d\xi \\
			& = \frac{1}{2\pi} \frac{\mathbf{i}^c}{s^c} \sum_{\beta \in \mathbb{N}^m_{0,c}} \frac{c!}{\vert \beta \vert! (c-\vert \beta \vert)!} \int_{\mathbb{R}} v^\beta \partial_\beta \widehat{f}(\xi v) \overline{\widehat{\psi}^{(c-\vert\beta\vert)}(\xi s)} s^{c-\vert\beta\vert} e^{\mathbf{i} \xi t} d\xi \\
			& = \frac{1}{2\pi} \mathbf{i}^c \sum_{\beta \in \mathbb{N}^m_{0,c}} \frac{c!}{\vert \beta \vert! (c-\vert \beta \vert)!} \int_{\mathbb{R}} \left( \frac{v}{s} \right)^\beta \partial_\beta \widehat{f}(\xi v) \widehat{\psi}^{(c-\vert\beta\vert)}(\xi s) e^{\mathbf{i} \xi t} d\xi.
		\end{aligned}
	\end{equation}
	Therefore, by taking the norm in \eqref{EqPropConstProof0} and by using the substitution $\zeta \mapsto \xi s$ as well as the inequality $\left\vert (v/s)^\beta \right\vert := \big\vert \prod_{l=1}^m (v_l/s)^{\beta_l} \big\vert = \prod_{l=1}^\beta \vert v_l/s \vert^{\beta_l} \leq \big( 1+\Vert v/s \Vert^2 \big)^{\vert \beta \vert/2} \leq \big( 1+1/s^2 \big)^{c/2}$ for any $v \in \mathbb{S}^{m-1}$, $s \in (0,\infty)$, and $\beta \in \mathbb{N}^m_{0,c}$, we obtain for every $(a,b) \in (\mathbb{R}^m \setminus \lbrace 0 \rbrace) \times \mathbb{R}$ that
	\begin{equation}
		\label{EqPropConstProof1}
		\begin{aligned}
			\vert b \vert^c \left\Vert (\mathfrak{R}_\psi f)(a,b) \right\Vert & \leq \frac{1}{2\pi} \sum_{\beta \in \mathbb{N}^m_{0,c}} \frac{c!}{\vert \beta \vert! (c-\vert \beta \vert)!} \int_{\mathbb{R}} \left\vert \left( \frac{v}{s} \right)^\beta \right\vert \big\Vert \partial_\beta \widehat{f}(\xi v) \big\Vert \left\vert \widehat{\psi}^{(c-\vert\beta\vert)}(\xi s) \right\vert d\xi \\
			& = \frac{1}{2\pi} \sum_{\beta \in \mathbb{N}^m_{0,c}} \frac{c!}{\vert \beta \vert! (c-\vert \beta \vert)!} \int_{\mathbb{R}} \left\vert \left( \frac{v}{s} \right)^\beta \right\vert \left\Vert \partial_\beta \widehat{f}\left( \frac{\zeta v}{s} \right) \right\Vert \left\vert \widehat{\psi}^{(c-\vert\beta\vert)}(\zeta) \right\vert \frac{1}{s} d\zeta \\
			& \leq \frac{c!}{2\pi} \left( 1 + \frac{1}{s^2} \right)^\frac{c}{2} \frac{1}{s} \sum_{\beta \in \mathbb{N}^m_{0,c}} \int_{\mathbb{R}} \left\Vert \partial_\beta \widehat{f}\left( \frac{\zeta v}{s} \right) \right\Vert \left\vert \widehat{\psi}^{(c-\vert\beta\vert)}(\zeta) \right\vert d\zeta \\
			& \leq \frac{(\lceil\gamma\rceil+2)!}{2\pi} \left( 1 + \Vert a \Vert^2 \right)^\frac{\lceil\gamma\rceil+2}{2} \sum_{\beta \in \mathbb{N}^m_{0,\lceil\gamma\rceil+2}} \int_{\mathbb{R}} \big\Vert \partial_\beta \widehat{f}(\zeta a) \big\Vert \left\vert \widehat{\psi}^{(\lceil\gamma\rceil+2-\vert\beta\vert)}(\zeta) \right\vert d\zeta.
		\end{aligned}
	\end{equation}
	Hence, by using the inequality $(x+y)^s \leq 2^{s-1} \big( x^s + y^s \big)$ for any $x,y \in [0,\infty)$ and $s \in [1,\infty)$ and the inequality \eqref{EqPropConstProof1}, it follows for every $(a,b) \in (\mathbb{R}^m \setminus \lbrace 0 \rbrace) \times \mathbb{R}$ that
	\begin{equation}
		\label{EqPropConstProof2}
		\begin{aligned}
			& \left( 1 + \vert b \vert^2 \right)^\frac{\lceil\gamma\rceil+2}{2} \Vert (\mathfrak{R}_\psi f)(a,b) \Vert \leq 2^\frac{\lceil\gamma\rceil}{2} \left( \Vert (\mathfrak{R}_\psi f)(a,b) \Vert + \vert b \vert^{\lceil\gamma\rceil+2} \Vert (\mathfrak{R}_\psi f)(a,b) \Vert \right) \\
			& \quad\quad \leq 2^\frac{\lceil\gamma\rceil}{2} \frac{(\lceil\gamma\rceil+2)!}{\pi} \left( 1 + \Vert a \Vert^2 \right)^\frac{\lceil\gamma\rceil+2}{2} \sum_{\beta \in \mathbb{N}^m_{0,\lceil\gamma\rceil+2}} \int_{\mathbb{R}} \big\Vert \partial_\beta \widehat{f}(\zeta a) \big\Vert \left\vert \widehat{\psi}^{(\lceil\gamma\rceil+2-\vert\beta\vert)}(\zeta) \right\vert d\zeta.
		\end{aligned}
	\end{equation}
	Moreover, by using Fubini's theorem and that $(\mathfrak{R}_\psi f)(a,b) = 0$ for any $(a,b) \in \lbrace 0 \rbrace \times \mathbb{R}$, we have
	\begin{equation}
		\label{EqPropConstProof3}
		\begin{aligned}
			& \Vert f \Vert_{\mathbb{B}^{k,\gamma}_\psi(U;\mathbb{R}^d)} \leq \left( \int_{\mathbb{R}^m} \int_{\mathbb{R}} \left( 1 + \Vert a \Vert^2 \right)^{\gamma+k+\frac{m+1}{2}} \left( 1 + \vert b \vert^2 \right)^{\gamma+1} \Vert (\mathfrak{R}_\psi f)(a,b) \Vert^2 db da \right)^\frac{1}{2} \\
			& \quad \leq \left( \int_{\mathbb{R}} \int_{\mathbb{R}^m}  \left( 1 + \Vert a \Vert^2 \right)^{\gamma+k+\frac{m+1}{2}} \left( 1 + \vert b \vert^2 \right)^{\gamma+2} \Vert (\mathfrak{R}_\psi f)(a,b) \Vert^2 da \, \frac{1}{1+\vert b \vert^2} db \right)^\frac{1}{2} \\
			& \quad \leq \left( \sup_{b \in \mathbb{R}} \int_{\mathbb{R}^m \setminus \lbrace 0 \rbrace}  \left( \left( 1 + \Vert a \Vert^2 \right)^\frac{\lceil\gamma\rceil+k+\frac{m+1}{2}}{2} \left( 1 + \vert b \vert^2 \right)^\frac{\lceil\gamma\rceil+2}{2} \Vert (\mathfrak{R}_\psi f)(a,b) \Vert \right)^2 da \right)^\frac{1}{2} \Bigg( \underbrace{\int_{\mathbb{R}} \frac{1}{1+\vert b \vert^2} db}_{=\pi} \Bigg)^\frac{1}{2}.
		\end{aligned}
	\end{equation}
	Thus, by inserting the inequality \eqref{EqPropConstProof2} into the right-hand side of \eqref{EqPropConstProof3}, using Minkowski's integral inequality (with measure spaces $(\mathbb{R}^m \setminus \lbrace 0 \rbrace,\mathcal{L}(\mathbb{R}^m \setminus \lbrace 0 \rbrace),da)$ and $(\mathbb{N}^m_{0,k} \times \mathbb{R},\mathcal{P}(\mathbb{N}^m_{0,k}) \otimes \mathcal{B}(\mathbb{R}),\mu \otimes d\zeta)$, where $\mathcal{P}(\mathbb{N}^m_{0,k})$ denotes the power set of $\mathbb{N}^m_{0,k}$, and where $\mathcal{P}(\mathbb{N}^m_{0,k}) \ni E \mapsto \mu(E) := \sum_{\alpha \in \mathbb{N}^m_{0,k}} \mathds{1}_E(\alpha) \in [0,\infty)$ is the counting measure), the substitution $\xi \mapsto \zeta a$ with Jacobi determinant $d\xi = \vert \zeta \vert^m da$, that $\zeta_1 := \inf\big\lbrace \vert \zeta \vert : \zeta \in \supp(\widehat{\psi}) \big\rbrace > 0$, and the constant $C_1 := 2^{\lceil\gamma\rceil/2} \pi^{-1/2} (\lceil\gamma\rceil+2)! \max_{j=0,...,\lceil\gamma\rceil+2} \int_{\mathbb{R}} \big\vert \widehat{\psi}^{(j)}(\zeta) \big\vert d\zeta > 0$, we conclude that
	\begin{equation*}
		\begin{aligned}
			& \Vert f \Vert_{\mathbb{B}^{k,\gamma}_\psi(U;\mathbb{R}^d)} \\
			& \quad \leq 2^\frac{\lceil\gamma\rceil}{2} \frac{(\lceil\gamma\rceil+2)!}{\pi} \sqrt{\pi} \left( \int_{\mathbb{R}^m} \left( \left( 1 + \Vert a \Vert^2 \right)^\frac{2\lceil\gamma\rceil+k+\frac{m+5}{2}}{2} \sum_{\beta \in \mathbb{N}^m_{0,\lceil\gamma\rceil+2}} \int_{\mathbb{R}} \big\Vert \partial_\beta \widehat{f}(\zeta a) \big\Vert \left\vert \widehat{\psi}^{(\lceil\gamma\rceil+2-\vert\beta\vert)}(\zeta) \right\vert d\zeta \right)^2 da \right)^\frac{1}{2} \\
			& \quad \leq 2^\frac{\lceil\gamma\rceil}{2} \frac{(\lceil\gamma\rceil+2)!}{\sqrt{\pi}} \sum_{\beta \in \mathbb{N}^m_{0,\lceil\gamma\rceil+2}} \int_{\mathbb{R}} \left\vert \widehat{\psi}^{(\lceil\gamma\rceil+2-\vert\beta\vert)}(\zeta) \right\vert \left( \int_{\mathbb{R}^m} \big\Vert \partial_\beta \widehat{f}(\zeta a) \big\Vert^2 \left( 1 + \Vert a \Vert^2 \right)^{2\lceil\gamma\rceil+k+\frac{m+5}{2}} da \right)^\frac{1}{2} d\zeta \\
			& \quad \leq 2^\frac{\lceil\gamma\rceil}{2} \frac{(\lceil\gamma\rceil+2)!}{\sqrt{\pi}} \sum_{\beta \in \mathbb{N}^m_{0,\lceil\gamma\rceil+2}} \int_{\supp(\widehat{\psi})} \frac{\left\vert \widehat{\psi}^{(\lceil\gamma\rceil+2-\vert\beta\vert)}(\zeta) \right\vert}{\zeta^\frac{m}{2}} \left( \int_{\mathbb{R}^m} \big\Vert \partial_\beta \widehat{f}(\xi) \big\Vert^2 \left( 1 + \Vert \xi/\zeta \Vert^2 \right)^\frac{4\lceil\gamma\rceil+2k+m+5}{2} d\xi \right)^\frac{1}{2} d\zeta \\
			& \quad \leq \frac{C_1}{\zeta_1^\frac{m}{2}} \sum_{\beta \in \mathbb{N}^m_{0,\lceil\gamma\rceil+2}} \left( \int_{\mathbb{R}^m} \big\Vert \partial_\beta \widehat{f}(\xi) \big\Vert^2 \left( 1 + \Vert \xi/\zeta_1 \Vert^2 \right)^\frac{4\lceil\gamma\rceil+2k+m+5}{2} d\xi \right)^\frac{1}{2},
		\end{aligned}
	\end{equation*}
	which completes the proof.
\end{proof}

\begin{proof}[Proof of Proposition~\ref{PropCOD}]
	Fix some $m,d \in \mathbb{N}$ and $\varepsilon > 0$. Moreover, let $p \in (1,\infty)$ and $w: U \rightarrow [0,\infty)$ be a weight as in Lemma~\ref{LemmaWeight} (with constant $C^{(\gamma,p)}_{\mathbb{R},w_0} > 0$ being independent of $m,d \in \mathbb{N}$ and $\varepsilon > 0$), let $(\psi,\rho) \in \mathcal{S}_0(\mathbb{R};\mathbb{C}) \times C^k_{pol,\gamma}(\mathbb{R})$ be a pair as in Example~\ref{ExAdm} (with $0 < \zeta_1 < \zeta_2 < \infty$ and constant $C_{\psi,\rho} > 0$ being independent of $m,d \in \mathbb{N}$ and $\varepsilon > 0$), and fix some $f \in W^{k,p}(U,\mathcal{L}(U),w;\mathbb{R}^d)$ satisfying the conditions of Proposition~\ref{PropConst} such that the right-hand side of \eqref{EqPropConst1} satisfies $\mathcal{O}\left( m^s (2/\zeta_2)^m (m+1)^{m/2} \right)$ for some $s \in \mathbb{N}_0$. Then, there exists some constant $C > 0$ (being independent of $m,d \in \mathbb{N}$ and $\varepsilon > 0$) such that for every $m,d \in \mathbb{N}$ it holds that
	\begin{equation}
		\label{EqPropCODProof1}
		\frac{C_1}{\zeta_1^\frac{m}{2}} \sum_{\beta \in \mathbb{N}^m_{0,\lceil\gamma\rceil+2}} \left( \int_{\mathbb{R}^m} \big\Vert \partial_\beta \widehat{f}(\xi) \big\Vert^2 \left( 1 + \Vert \xi/\zeta_1 \Vert^2 \right)^{2\lceil\gamma\rceil+k+\frac{m+5}{2}} d\xi \right)^\frac{1}{2} \leq C m^s \left( \frac{2}{\zeta_2} \right)^m (m+1)^\frac{m}{2}.
	\end{equation}
	Hence, by using Proposition~\ref{PropConst} together with \eqref{EqPropCODProof1}, Lemma~\ref{LemmaWeight}, the inequality in Example~\ref{ExAdm}, that $\Gamma(x) \geq \sqrt{2\pi/x} (x/e)^x$ for any $x \in (0,\infty)$ (see \cite[Lemma~2.4]{gonon19}), and that $\frac{\pi^{m/4} (2/\zeta_2)^m}{(2\pi/\zeta_2)^m (1/(2e))^{m/2}} = \big( \frac{2e\sqrt{\pi}}{\pi^2} \big)^{m/2} \leq 1$ for any $m \in \mathbb{N}$, we conclude that there exist some constants $C_2, C_3 > 0$ (being independent of $m,d \in \mathbb{N}$ and $\varepsilon > 0$) such that
	\begin{equation}
		\label{EqPropCODProof2}
		\begin{aligned}
			& C_p \Vert \rho \Vert_{C^k_{pol,\gamma}(\mathbb{R})} \frac{C^{(\gamma,p)}_{U,w} m^\frac{k}{p} \pi^\frac{m+1}{4}}{\left\vert C^{(\psi,\rho)}_m \right\vert \Gamma\left( \frac{m+1}{2} \right)^\frac{1}{2}} \frac{\Vert f \Vert_{\mathbb{B}^{k,\gamma}_\psi(U;\mathbb{R}^d)}}{N^{1-\frac{1}{\min(2,p)}}} \\
			& \quad\quad \leq C_p \Vert \rho \Vert_{C^k_{pol,\gamma}(\mathbb{R})} \frac{C^{(\gamma,p)}_{\mathbb{R},w_0} m^{\gamma+\frac{k+1}{p}} \pi^\frac{m+1}{4}}{C_{\psi,\rho} \left( \frac{2\pi}{\zeta_2} \right)^m \left( \frac{4\pi}{m+1} \right)^\frac{1}{4} \left( \frac{m+1}{2e} \right)^\frac{m+1}{2}} C m^s \left( \frac{2}{\zeta_2} \right)^m (m+1)^\frac{m}{2} \\
			& \quad\quad \leq C_p \Vert \rho \Vert_{C^k_{pol,\gamma}(\mathbb{R})} \frac{C^{(\gamma,p)}_{\mathbb{R},w_0} m^{\gamma+\frac{k+1}{p}} \pi^\frac{1}{4} (2e)^\frac{1}{2} C_3}{C_{\psi,\rho} (4\pi)^\frac{1}{4}} C m^s \\
			& \quad\quad \leq \left( C_2 m^{C_3} \right)^{1-\frac{1}{\min(2,p)}}.
		\end{aligned}
	\end{equation}
	Hence, by using that $f \in \mathbb{B}^{k,\gamma}_\psi(U;\mathbb{R}^d)$ (see Proposition~\ref{PropConst}), we can apply Theorem~\ref{ThmApproxRates} with $N = \Big\lceil C_2 m^{C_3} \varepsilon^{-\frac{\min(2,p)}{\min(2,p)-1}} \Big\rceil$ and insert the inequality~\eqref{EqPropCODProof2} to obtain a neural network $\varphi \in \mathcal{NN}^\rho_{U,d}$ with $N$ neurons satisfying
	\begin{equation*}
		\begin{aligned}
			\Vert f - \varphi_N \Vert_{W^{k,p}(U,\mathcal{L}(U),w;\mathbb{R}^d)} & \leq C_p \Vert \rho \Vert_{C^k_{pol,\gamma}(\mathbb{R})} \frac{C^{(\gamma,p)}_{U,w} m^\frac{k}{p} \pi^\frac{m+1}{4}}{\left\vert C^{(\psi,\rho)}_m \right\vert \Gamma\left( \frac{m+1}{2} \right)^\frac{1}{2}} \frac{\Vert f \Vert_{\mathbb{B}^{k,\gamma}_\psi(U;\mathbb{R}^d)}}{N^{1-\frac{1}{\min(2,p)}}} \\
			& \leq \frac{\left( C_2 m^{C_3} \right)^{1-\frac{1}{\min(2,p)}}}{N^{1-\frac{1}{\min(2,p)}}} \leq \varepsilon,
		\end{aligned}
	\end{equation*}
	which completes the proof.
\end{proof}

\bibliographystyle{plain}
\bibliography{mybib}

\end{document}